\documentclass[journal]{IEEEtran}
\usepackage{xcolor,color,soul}
\soulregister\cite7 
\soulregister\citep7 
\soulregister\citet7 
\soulregister\ref7 
\soulregister\pageref7 
\soulregister\eqref7

\usepackage{cite}
\sethlcolor{white}

\ifCLASSINFOpdf
  \usepackage[pdftex]{graphicx}
  \graphicspath{{./pdf/}{./jpeg/}}
  \DeclareGraphicsExtensions{.pdf,.jpeg,.png}
\else
  \usepackage[dvips]{graphicx}
  \graphicspath{{./eps/}}
  \DeclareGraphicsExtensions{.eps}
\fi
\usepackage{amsmath,amssymb,amsthm}
\usepackage[T1]{fontenc}
\usepackage{newtxmath}
\usepackage{euscript}
\DeclareMathAlphabet{\mathpzc}{T1}{pzc}{m}{it}

\usepackage{algorithm} 
\usepackage{algpseudocode}

\usepackage{multirow}

\usepackage{lineno}

\begin{document}

%
\title{Model-Free 3D Shape Control of Deformable Objects Using Novel Features Based on Modal Analysis
}
%
%
%

\author{Bohan Yang, 
Bo Lu, 
Wei Chen, Fangxun Zhong, 
and Yun-Hui Liu,~\IEEEmembership{Fellow,~IEEE}

\thanks{This work is supported in part of the HK RGC under T42-409/18-R and 14202918, in part by the Shenzhen-HK Collaborative Development Zone, and in part by the VC Fund 4930745 of T Stone Robotics Institute. \textit{(Corresponding author: Yun-Hui Liu.)}}
\thanks{Yun-Hui Liu, Bohan Yang, Wei Chen, and Fangxun Zhong are with the T Stone Robotics Institute, the Department of Mechanical and Automation Engineering, The Chinese University of Hong Kong, HKSAR, China (e-mail: yhliu@cuhk.edu.hk, bhyang@mae.cuhk.edu.hk, weichen@link.cuhk.edu.hk, fxzhong@cuhk.edu.hk).

Bo Lu was with The Chinese University of Hong Kong, HKSAR, China. He is now with the Robotics and Microsystems Center, School of Mechanical and Electric Engineering, Soochow University, Suzhou 215021, China (e-mail:blu@suda.edu.cn).}

\thanks{The paper will appear in the IEEE Transactions on Robotics.
Copyright notice:   "© 20xx IEEE.  Personal use of this material is permitted.  Permission from IEEE must be obtained for all other uses, in any current or future media, including reprinting/republishing this material for advertising or promotional purposes, creating new collective works, for resale or redistribution to servers or lists, or reuse of any copyrighted component of this work in other works." }

}

\maketitle

\begin{abstract}
Shape control of deformable objects is a challenging and important robotic problem.
This paper proposes a model-free controller using novel 3D global deformation features based on modal analysis.
Unlike most existing controllers using geometric features, our controller employs physically based deformation features designed by decoupling global deformation into low-frequency modes.
Although modal analysis is widely adopted in computer vision and simulation,
its usage in robotic 
deformation control is still an open topic.
We develop a new model-free framework for the modal-based deformation control.
Physical interpretation of the modes enables
us to formulate an analytical
deformation Jacobian matrix mapping the robot manipulation onto changes of the modal features.
In the Jacobian matrix, unknown geometric and physical models of the object are treated as low-dimensional modal parameters which can be used to linearly parameterize the closed-loop system.
Thus, an adaptive controller with proven stability can be designed to deform the object while online estimating the modal parameters.
Simulations and experiments are conducted using linear, planar, and volumetric objects under different settings. 
The results not only confirm the superior performance of our controller but also demonstrate its advantages over the baseline method.
\end{abstract}


%

\section{Introduction}

\IEEEPARstart{I}{ntroducing} autonomy of 
deformable object manipulation (DoM) can promote significant developments of robots in sectors of medical robotics~\cite{shin2019autonomous},
public service~\cite{lee2015learning},  
industrial manufacturing~\cite{li2018vision}, etc.
However, unlike rigid object manipulation, DoM raises many difficult problems due to their complex deformation with varying physical properties and infinite dimensions.

During the past two decades, the field of
DoM
has gained distinctive progress.
Early works mainly focused on simple point-based positioning~\cite{wada2001robust}, shape control with only simulation analysis~\cite{das2011autonomous}, and specific task-oriented manipulation~\cite{tokumoto2002deformation}.
Recently, significant efforts have been devoted to more complex tasks
where researchers developed algorithms using advanced control strategies~\cite{navarro2016automatic}~\cite{hu20193},
deformation representation methods with better description abilities~\cite{navarro2018fourier}~\cite{qi2021contour},
and frameworks with multi-sensory feedback~\cite{ficuciello2018fem}~\cite{abayazid2015ultrasound}.
The latest survey papers~\cite{sanchez2018robotic}~\cite{arriola2020modeling}~\cite{nadon2018multi} reviewed the modeling, sensing, planning, and control strategies of robotic 
DoM.
Among the reviewed research topics, 
automatic (local) deformation control is one of the most fundamental issues.
Despite the efforts of the latest advances, 
real-time control of 3D shape deformation without accurate models is still an open problem.
Targeting this problem, we develop a model-free method for 3D shape control of deformable objects using stereo vision feedback.

\subsection{Related Work}
\subsubsection{Deformation Control}
Existing deformation control strategies can be classified into three categories: model-based methods~\cite{das2011autonomous}~\cite{ficuciello2018fem}, 
Jacobian-based methods~\cite{berenson2013manipulation}~\cite{navarro2013model},  
and learning-based methods~\cite{hu20193}.
Due to difficulties in obtaining accurate object models, increasing attentions have been put on the Jacobian- and learning-based model-free deformation controllers.
For model-free methods using deformation Jacobian approximations,
Berenson et al.~\cite{berenson2013manipulation} proposed the concept of diminishing rigidity to approximate deformation Jacobian matrices for rope- and cloth-like objects.
Shetab-Bushehri et al.~\cite{shetab2022rigid} adopted the As-Rigid-As-Possible model for the Jacobian approximation of planar objects.
However, these works need to know the geometry of the object-gripper configuration, which adds extra requirements for object sensing.
Online estimation of the deformation Jacobian matrix is another popular model-free method.
Navarro-Alarcon et al.~\cite{navarro2013model} developed a series of works for geometric features of 3D control points~\cite{navarro2016automatic},
2D curve descriptors using spline-based methods~\cite{qi2020adaptive},
as well as 2D contour descriptors using Fourier-based~\cite{navarro2018fourier} and moment-based~\cite{qi2021contour} methods.
Lagneau et al.~\cite{lagneau2020active} proposed an active deformation controller for marker-based and marker-less 3D point features using online Jacobian estimation with a sliding window.
Nevertheless, as these geometric deformation features become more complex or global, the deformation Jacobian matrix becomes more difficult to formulate and estimate.
These works need to re-calibrate local Jacobian matrices at different regions,
or offline testing deformation is required for pre-estimation of the Jacobian matrix.
Moreover, the influence of the online Jacobian estimation on their control performances is hard to be analyzed and proved mathematically.

Learning-based methods serve as powerful tools for model-free deformation control because of their potential to generate complex feature-manipulation relationships using structures of Neural Networks~\cite{cherubini2020model}~\cite{shin2019autonomous}~\cite{li2018vision}, Gaussian Process Regression~\cite{hu2018three}, and Deep Neural Networks~\cite{hu20193}.
Early works of data-driven methods mainly studied deformation models~\cite{cretu2011soft} or manipulation strategies~\cite{lee2015learning} without high accuracy requirements.
For better accuracy, researchers have combined learning-based techniques with traditional control methods.
Shin et al.~\cite{shin2019autonomous} proposed two learning-based model predictive control algorithms to manipulate points on tissues.
Hu et al.~\cite{hu2018three}~\cite{hu20193} designed online learning methods within visual servoing frameworks to control 3D deformation described by different geometric and point-cloud-based features.
However, more complex deformation representations lead to higher requirements for data collecting and training.
Although efforts have been put into different offline and online mechanisms, control performances of these learning-based methods highly rely on the learning architectures and parameters, which is hard to explain at both the mathematical and control levels.
Combining the advantages of neural network structures and stable adaptive control laws,
Yu et al.~\cite{yu2022global} and Li et al.~\cite{li2018vision} proposed deformation controllers for linear and planar objects.
They employed adaptive neural networks to approximate the deformation Jacobian matrix with guaranteed control stability.
Even though it is a promising technique, how to extend the method to more complex 3D
objects and shape control is still an issue to be further addressed.

\subsubsection{Deformation Representation}
Another crucial problem for deformation control is 3D shape representation.
Deformable shape modeling has been extensively studied in the fields of computer vision, computer graphics, and medical imaging.
Review papers~\cite{arriola2020modeling}~\cite{montagnat2001review} classified geometric methods into discrete representations (such as meshes~\cite{madi2019new}, 
piece-wise patches~\cite{fayad2010piecewise}, 
and point clouds~\cite{newcombe2010live}), 
implicit curves or surfaces~\cite{gascuel1993implicit},
explicit parameterized representations (such as the spline-~\cite{prasad2010finding} and Fourier-~\cite{kelemen1996segmentation} based decompositions), and the free-form deformation~\cite{sederberg1986free}.
Nevertheless, the relationship between geometric descriptors and robot manipulation is hard to formulate.
Researchers have explored the potential to embed physical models into the deformation representation under robotic scenarios.
Fugl et al.~\cite{fugl2012simultaneous} formulated deformable curves as physical-model-based functions of material parameters, mesh geometry, and gripper pose.
Other researchers~\cite{ficuciello2018fem}~\cite{zaidi2017model} represented deformation by relating sensor measurements with physical mesh models.
However, directly using mesh models in the robotic manipulation strategy design presents many technical challenges due to the high-dimensional nature of common modeling techniques and the difficulties to obtain accurate models.
For more compact shape representations, researchers also investigated a class of physically based methods using parameterized deformation.
These methods~\cite{barr1987global}~\cite{pentland1987perceptual} first treated solid shapes as deformed results from some reference shapes, and then parameterized their deformation with low dimensions.
To formulate the parametrized deformation, finding the reduced basis for subspace deformation~\cite{barbivc2005real}~\cite{an2008optimizing} is a hard key issue.
Metaxas et al.~\cite{metaxas1992dynamic} formulated the basis using deformable superquadrics but need to know the deformation form in advance.
Krysl et al.~\cite{krysl2001dimensional} adopted the reduced basis using principal component analysis (PCA) but required an example data set.
Modal analysis is also a popular technique to generate the reduced basis for linear~\cite{pentland1989good}~\cite{baraff1992dynamic} and nonlinear~\cite{barbivc2005real}~\cite{choi2005modal} deformation.
It is a standard tool for deformation reduction~\cite{pentland1989good}~\cite{choi2005modal} and 3D shape recovery~\cite{pentland1991closed}~\cite{agudo2014good}.
However, 
robotic deformation control using modal analysis is still an open problem.

Under the same representation-control framework in Fig.~\ref{fig:contribution_SoA}, we compare the reviewed state-of-the-art (SOTA) deformation controllers without accurate models.
Existing model-free controllers use geometric features that directly describe the object geometry.
Nevertheless, 
finding the analytical mapping relationship between the geometric features and robot manipulation is difficult because object deformation behaviors are highly coupled systems of high-dimensional object geometry and complex material properties.
With unknown object models, they had to design the black-box control~\cite{ljung2010perspectives} laws with totally unknown feature-manipulation mapping.
On the contrary, this paper studies physically based representations because the physical laws employed in the object description also provide a physical reference to formulate the grey-box control~\cite{ljung2010perspectives} laws with partially unknown feature-manipulation mapping.
In addition, unlike the model-based controller~\cite{ficuciello2018fem} using high-dimensional mesh models, our model-free controller adopts the low-dimensional modal analysis.

\begin{figure}[t]
    \centering
    \includegraphics[width=\linewidth]{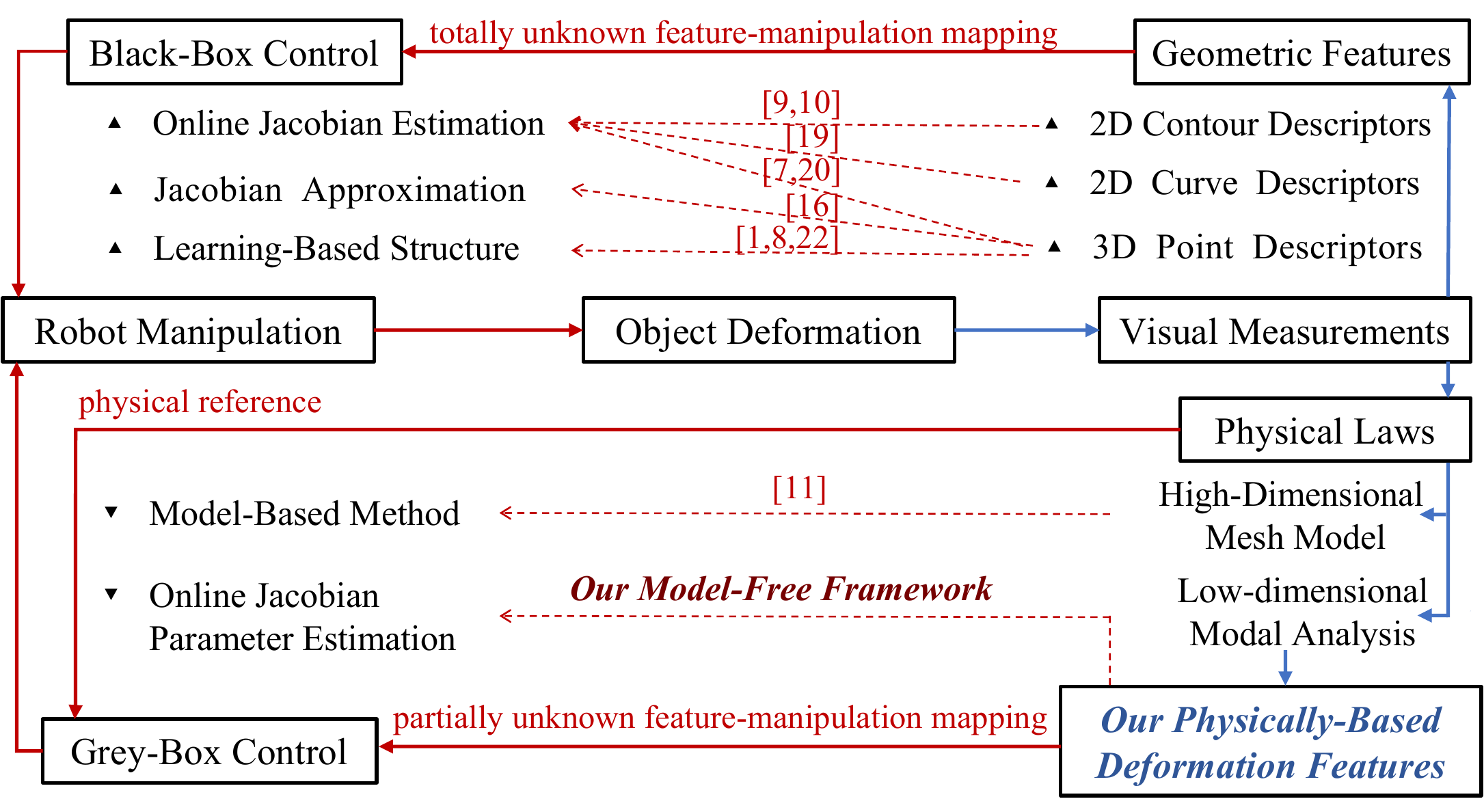}
    \vspace{-0.7cm}
    \caption{Comparisons with the SOTA deformation controllers without accurate models.
    The blue lines are for deformation representations while the red lines are for deformation control.
    }
    \label{fig:contribution_SoA}
    \vspace{-0.1cm}
\end{figure}

\subsection{Our Contributions}
This paper proposes a model-free 3D deformation controller using modal-based global deformation features.
The advantages of designing deformation features with modal analysis are:
common deformable shapes can be described 
with only a small number of low-frequency deformation modes~\cite{pentland1991recovery};
physical interpretation of the modes enables an analytical formulation of the deformation Jacobian matrix mapping
changes of the modal features with motions of the robot manipulation.
However, 
technical challenges arise because linear modal analysis~\cite{bathe2006finite} 
relies on object models and has limited abilities for large deformation.
To overcome these challenges,
instead of analyzing object dynamics by performing modal analysis on the object,
we only use low-frequency modes of a base mesh to span a feature space where the object's global deformation can be uniquely represented with low dimensions.
Then, 
based on the deformation features,
we design an adaptive
deformation controller
that tackles the object's unknown geometric and physical models by online estimating unknown modal parameters in the Jacobian matrix.
The stability of the designed controller is proved using Lyapunov-based methods.
Compared to other SOTA model-free methods, our method:\\
1) proposes novel physically-based, global deformation features using modal analysis. To the best of our knowledge, this is the first work adopting deformation modes to design deformation features for robotic 3D deformation control;\\
2) designs a new analytical deformation Jacobian matrix where unknown object models are treated as unknown modal parameters with only low dimensions.
These parameters can be used to linearly parameterize the closed-loop system, which enables us to develop stable control laws without requiring accurate model learning or parameter identifications;\\
3) can deal with different types (i.e. linear, planar, and volumetric) 
of objects and stereo features (i.e. points, curves, and contours).
Extensive simulations and experiments are conducted for validation.
The results confirm the superior performance of our method under different settings.

The rest of this paper is organized as follows:
Section II states the problem and overviews the algorithm flow; 
Section III defines our modal-based deformation feature space;
Section IV presents the method to compute the deformation features in the feature space from object measurements;
Section V derives the deformation Jacobian matrix between the deformation features and robot manipulation;
Section VI designs an adaptive controller based on the deformation Jacobian matrix; 
Simulations and experiments are presented and discussed in Section VII and VIII, respectively;
Section IX presents discussions and conclusions.

\begin{figure}[b]
    \vspace{-0.2cm}
    \centering
    \includegraphics[width=0.8\linewidth]{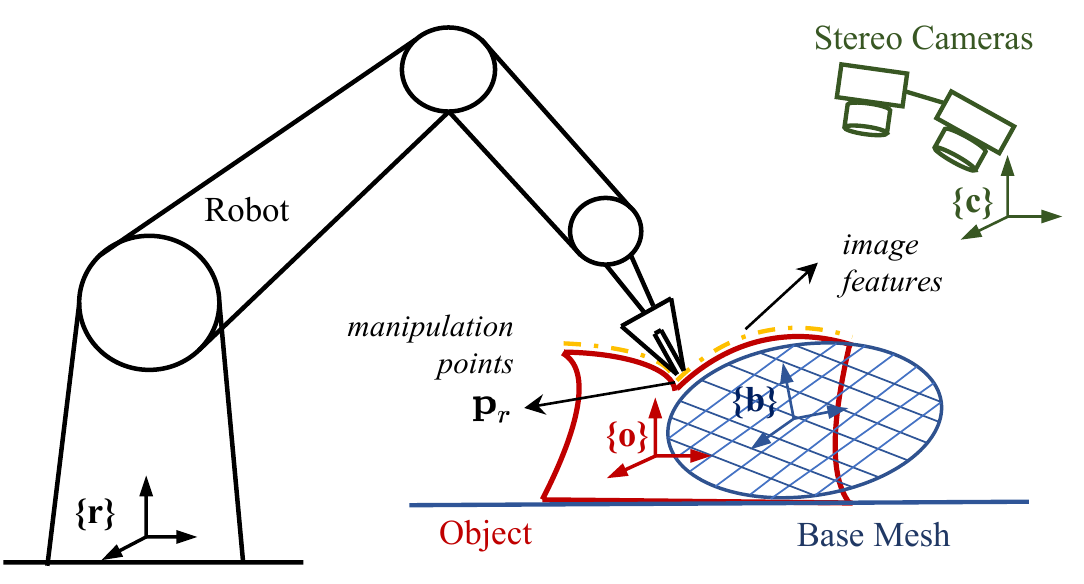}
    \caption{System configuration of the 3D deformation control problem. We consider four coordinate frames: the object frame $\left \{ o \right \}$, the robot frame $\left \{ r \right \}$, the camera frame $\left \{ c \right \}$, and the base mesh frame $\left \{ b \right \}$.}
    \label{fig:system_configuration}
\end{figure}

\section{Problem Statements}
\subsection{Notation}
In this paper, we use lowercase bold letters to denote column vectors, capital bold letters to denote matrices, and italic letters to denote scalar quantities.
Leading superscripts denote reference frames.
Variables without a leading superscript are defined under the default reference frame.
${}^a\mathbf{T}_b \in \mathbb{R}^{4 \times 4}$ denotes a homogeneous transformation matrix from frame ${b}$ to frame ${a}$.
The symbol $\hat{}$ means estimated values; the symbol ${}^{*}$ means desired values.

\subsection{Problem Definition}
We study the 3D shape control of a deformable object with unknown physical properties and undeformed geometry.
As shown in Fig.~\ref{fig:system_configuration}, the object is manipulated by a robot via manipulation points $\mathbf{p}_r$ (of size $k$).
The object is monitored with a pair of static calibrated stereo cameras.
To clarify our problem, the following assumptions are made:
\newtheorem{assumption}{Assumption}
\begin{assumption} 
The object is rigidly and firmly grasped by the robot.
\end{assumption}
\begin{assumption} 
Some image features (such as curves, contours, points, etc.) can be tracked on the surface of the object by both of the cameras in real time.
\end{assumption}
\begin{assumption} 
The robot manipulating motion is sufficiently slow such that we can only consider the quasi-static elastic deformation of the object~\cite{navarro2016automatic}.
\end{assumption}
\begin{assumption} 
The deformable points sampled from the object's image features and manipulated by the robot are elastically constrained in all directions around the equilibrium~\cite{navarro2016automatic}.
\end{assumption}

Consider the object with a point set $\mathbf{p}$ of size $L$ ($L \rightarrow \infty$), its 3D shape  ${}^o\mathbf{x}(\mathbf{p},t) \in \mathbb{R}^{3L}$ (i.e. the position vector of $\mathbf{p}$) is high-dimensional.
To represent the shape with low dimensions, we design deformation features $\mathbf{s} \in \mathbb{R}^m$ ($m \ll 3L$) using low-frequency deformation modes of the base mesh.
Then, based on the deformation features, we define the following 3D shape control problem:
\newtheorem*{problem}{Problem}
\begin{problem}
Given the desired shape ${}^o\mathbf{x}^*(\mathbf{p}) \in \mathbb{R}^{3l}$ and deformation features $\mathbf{s}^* \in \mathbb{R}^m$,
control the velocities $\mathbf{v}(\mathbf{p}_r,t) \in \mathbb{R}^{3k}$ of $\mathbf{p}_r$ using feedback of the deformation feature errors $\mathbf{e}_s(t)=\mathbf{s}(t)-\mathbf{s}^* \in \mathbb{R}^m$,
such that as $\mathbf{s}(t) \to \mathbf{s}^*$, ${}^o\mathbf{x}(\mathbf{p},t) \to {}^o\mathbf{x}^*(\mathbf{p})$.
\end{problem}

\begin{figure}[b]
    \centering
    \vspace{-0.3cm}
    \includegraphics[width=\linewidth]{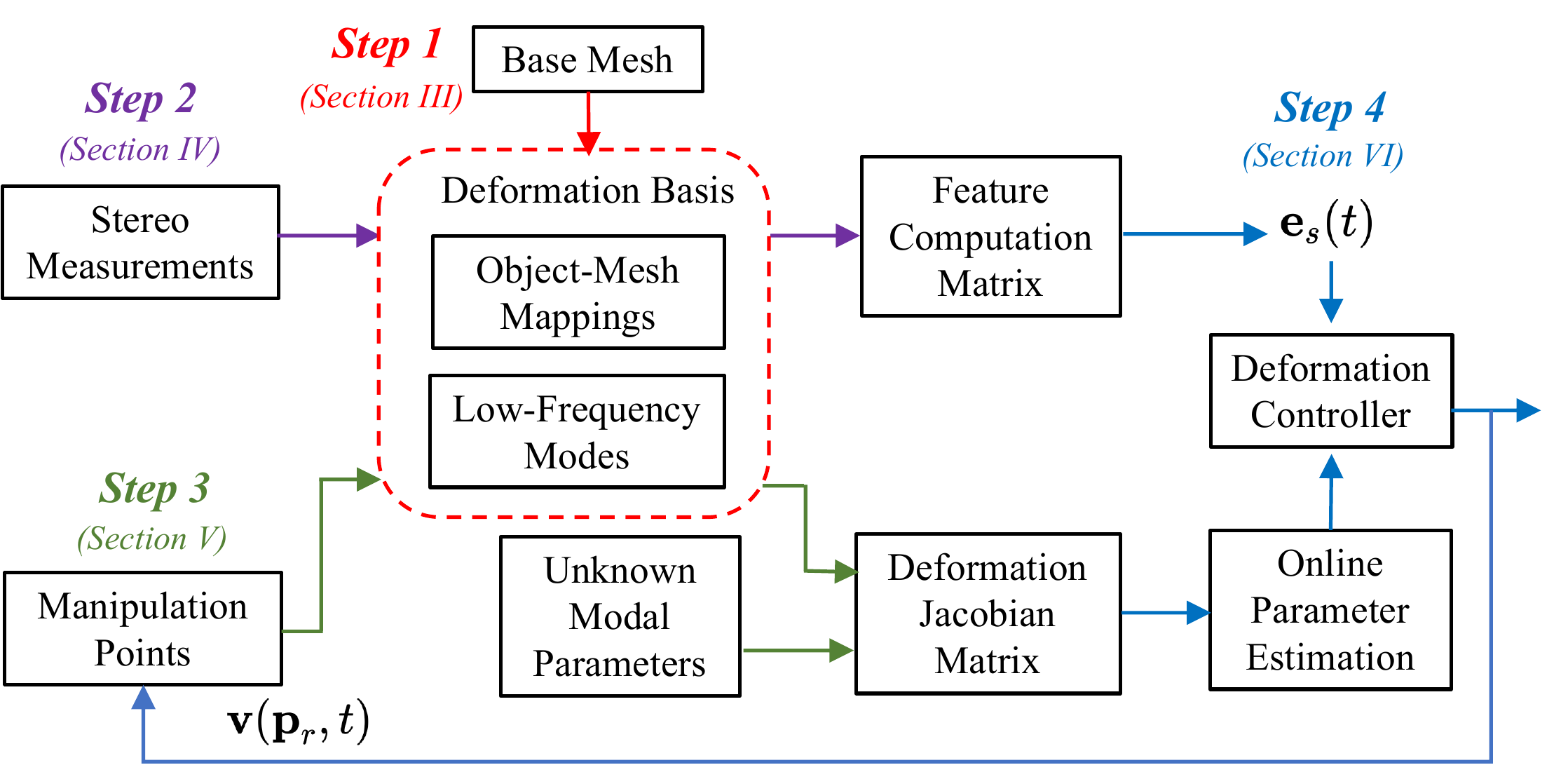}
    \caption{Schematic representation of the proposed algorithm flow.}
    \label{fig:algorithm_steps}
\end{figure}

\subsection{Algorithm Outline}
As shown in Fig.~\ref{fig:algorithm_steps}, the major steps of our algorithm are summarized as follows:\\
1) We introduce
the base mesh and compute its low-frequency modes.
By combining the modes according to the object-mesh mapping, we formulate the deformation basis that defines
a low-dimensional deformation feature space;\\
2) Given the deformation basis, we derive the deformation feature computation matrix, using which stereo measurements of the object can be projected to the feature space;\\
3) Given the deformation basis, we formulate the deformation Jacobian matrix with
unknown modal parameters;\\
4) Using the feature computation and the Jacobian matrices, 
we design a deformation controller that sends manipulation commands while online estimating the unknown parameters.

\section{Modal-Based Deformation Feature Space}
Modal analysis~\cite{bathe2006finite}~\cite{pentland1989good} is a well-known tool for the dimension reduction of global deformation.
However, it requires deformation models of the object which are hard to be obtained accurately on practical occasions.
In this section, a low-dimensional deformation feature space is generated using
modal analysis but with unknown object models.

\subsection{Modal Decomposition of 3D Global Deformation}
We 
begin with a brief introduction of how modal analysis is commonly used to
parameterize global deformation with low dimensions.
As shown in Fig.~\ref{fig:relative_deformation}(a), 3D global deformation can be defined by mapping the object (with domain $\Omega_o$) to itself with new, deformed coordinates~\cite{witkin1990fast}:
\begin{equation}
    {}^o\mathbf{x}(\mathbf{p},t) = {}^o\mathbf{x}(\mathbf{p}) + {}^o\mathbf{u}(\mathbf{p},t)
    \label{eq:model_global_deformation}
\end{equation}
where 
the position vector ${}^o\mathbf{x}(\mathbf{p},t)$ quantifies the 3D shape of the object during deformation;
${}^o\mathbf{x}(\mathbf{p}) \in \mathbb{R}^{3L}$ is the undeformed position vector of the object points $\mathbf{p}$ (i.e. the undeformed geometry of the object);
${}^o\mathbf{u}(\mathbf{p},t) \in \mathbb{R}^{3L}$ is the displacement vector of $\mathbf{p}$ (i.e. the displacement field of the object).
Then, according to the modal decomposition techniques~\cite{bathe2006finite}~\cite{pentland1989good},
the displacement field can be decoupled 
into a set of frequency-ordered, orthogonal modes:
\begin{equation}
    {}^o\mathbf{u}(\mathbf{p},t)
    = \sum_j^{3L} \boldsymbol{\phi}_o^j c_j(t)
\end{equation}
where
$\boldsymbol{\phi}_o^j \in \mathbb{R}^{3L}$ is the $j$-th deformation mode of the object whose coefficient is $c_j(t)$.
As discussed in~\cite{pentland1989perception}, high-frequency modes typically have little energy and less effect on the overall 3D shape.
Thus, by discarding high-frequency modes, the truncated mode coefficients can form a low-dimensional approximation for the 3D global deformation of the object.

Another advantage of using modal descriptions in robotic DoM is that the physical interpretation of the modes facilitates an analytical formulation of the deformation Jacobian matrix.
Specifically, the mode shape vector, $\boldsymbol{\phi}_o^j$, describes how positions of the object points (${}^o\mathbf{x}(\mathbf{p}_i,t)=(x_i(t),y_i(t),z_i(t))^T, i=\left \{ 1,..., L \right \}$) change in the mode as a function of the coefficient ${c}_j(t)$~\cite{pentland1991recovery}:
\begin{equation}
    \boldsymbol{\phi}_o^j \propto 
  (
  \begin{matrix}
 \frac{\partial x_1}{\partial c_j}, & 
 \frac{\partial y_1}{\partial c_j}, & 
 \frac{\partial z_1}{\partial c_j}, &
 ... & 
 \frac{\partial x_L}{\partial c_j},
 & 
 \frac{\partial y_L}{\partial c_j},
 &
 \frac{\partial z_L}{\partial c_j}
 \end{matrix}
  )^T.
\end{equation}
In this way, it naturally defines a Jacobian relationship between 
${}^o\mathbf{u}(\mathbf{p},t)$ and 
${c}_j$.
By extracting the rows of the manipulation points from the mode shape vectors, we can analytically construct the deformation Jacobian matrix mapping motions of the points onto changes of the modal coefficients.

Unfortunately, 
besides the reliance on 
object models, the modal decomposition techniques are also limited by the local nature of the linear modal analysis when analyzing object dynamics under
large deformation.
To avoid these problems, 
we do not compute the dynamic behaviors of the object or perform modal analysis on it.
Instead,
we propose a model-free framework to generate a deformation feature space using low-frequency modes of a base mesh.

\begin{figure}[b]
    \centering
    \vspace{-0.2cm}
    \includegraphics[width=\linewidth]{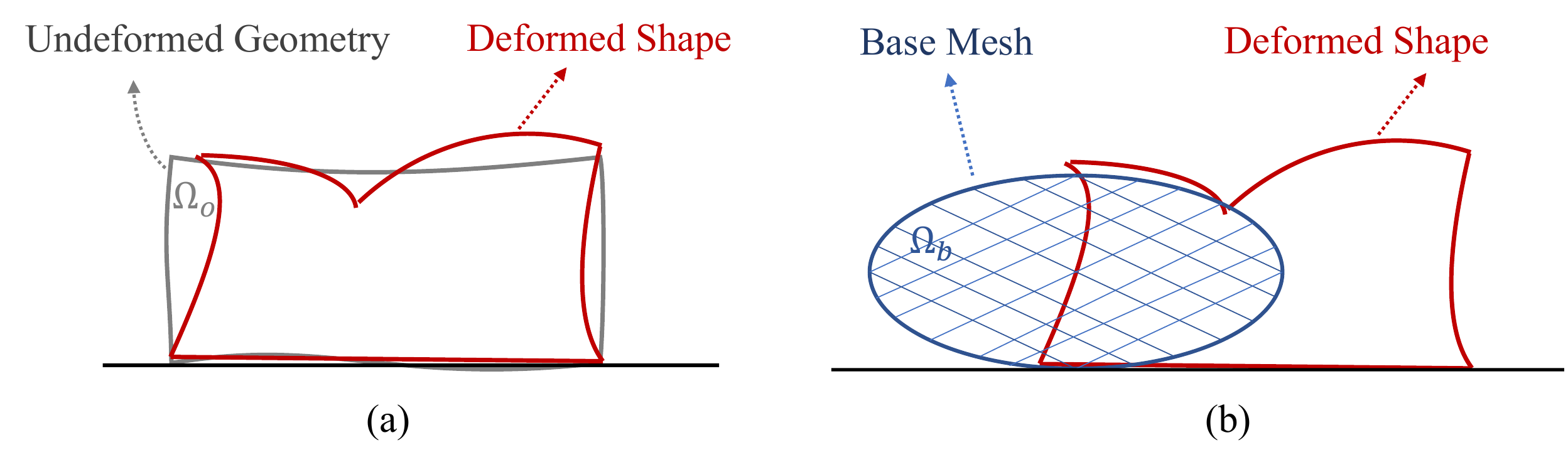}
    \vspace{-0.7cm}
    \caption{Formulation of the object's global deformation relative to: (a) the undeformed geometry of the object; (b) the base mesh.}
    \label{fig:relative_deformation}
\end{figure}

\subsection{Base Mesh and Object-Mesh Mapping}
The base mesh used in our model-free framework is selected to be an ellipsoidal mesh.
Consider an ellipsoid with the following surface vector~\cite{solina1990recovery}:
\begin{equation}
    {}^c\boldsymbol{\gamma}(\zeta ,\sigma) = {}^c\mathbf{T}_b
    \begin{bmatrix}
    a_x cos\zeta cos\sigma \\
    a_y cos\zeta sin\sigma \\
    a_z sin\zeta
\end{bmatrix}^T
\label{eq:ellipsoid}
\end{equation}
where
$\zeta \in [-\pi/2,\pi/2]$ is the latitude while $\sigma \in [-\pi,\pi )$ is the longitude;
$\left \{ a_x,a_y,a_z \right \}$ are the size parameters in x, y, and z dimension, respectively;
${}^c\mathbf{T}_b \in \mathbb{R}^{4 \times 4}$ is its pose with respect to the camera frame.
Discreting the ellipsoid, we obtain the base mesh with the domain $\Omega_b$ and the node set $\mathbf{n} \in \Omega_b$ of size $N$ ($N \ll L$).
Using ${}^c\mathbf{T}_b$, stereo measurements of the object during deformation can be transformed to the base mesh frame.
In the rest of our derivations, we set the base mesh frame $\left \{ b \right \}$ as our default reference frame.

Following the shape modeling techniques in~\cite{metaxas1992dynamic}~\cite{solina1990recovery},
we treat the object shape as the deformed result from the base mesh.
Then, as shown in Fig.~\ref{fig:relative_deformation}(b),
the object's global deformation can be re-formulate with:
\begin{equation}
    \mathbf{x}(\mathbf{p},t) =
    \mathbf{x}(\boldsymbol{\eta}(\mathbf{p})) +
    \mathbf{u}(\mathbf{p},\boldsymbol{\eta}(\mathbf{p}),t)
    \label{eq:global_deformation}
\end{equation}
where $\boldsymbol{\eta}(\mathbf{p}) \in \partial \Omega_b$ are the projections of object points $\mathbf{p}$ on the surface of the base mesh;
$\mathbf{x}(\boldsymbol{\eta}(\mathbf{p})) \in \mathbb{R}^{3L}$ is their position vector.
For a static base mesh with no deformation, the object-mesh displacements $\mathbf{u}(\mathbf{p},\boldsymbol{\eta}(\mathbf{p}),t) \in \mathbb{R}^{3L}$ can be used to quantify the object shape.
Note that pose and size differences between the object and the base mesh can be absorbed in $\mathbf{u}(\mathbf{p},\boldsymbol{\eta}(\mathbf{p}),t)$.
Thus, when generating the base mesh,
we do not need to calibrate the object's pose and size to set equation~\eqref{eq:ellipsoid}.
The mesh pose ${}^c\mathbf{T}_b$ and size parameters $a_x$ and $a_y$ can be roughly estimated using the stereo measurements of the object at the rest configuration (i.e. the configuration before the robot runs).
Only a rough estimation of the object's thickness is required to set $a_z$.

\newtheorem{remark}{Remark}
\begin{remark}
The mesh pose ${}^c\mathbf{T}_b$ and sizes $a_x$ and $a_y$ can be computed using 3D measurements of the object points which are sampled from the image features (on the object surface) and manipulated by the robot.
Besides the moment-based computation method in~\cite{pentland1991closed}, we also propose a simple solution in Algorithm 1 (shown in Section VIII).
\end{remark}

Then, 
we establish the object-mesh mapping to 
calculate
$\mathbf{u}(\mathbf{p},\boldsymbol{\eta}(\mathbf{p}),t)$
in equation~\eqref{eq:global_deformation} and to further allocate it to mesh nodes.
Inspired by the radially projecting virtual spring attachment method~\cite{terzopoulos1987symmetry}~\cite{pentland1991closed},
the mapping is computed as follows:
First, for each object point $\mathbf{p}^{w}$ ($w=\left \{ 1,...,L \right \} $) at the rest configuration,
we project it to the mesh surface along its connection line with the mesh center and set the intersection surface point
to be its mesh projection $\boldsymbol{\eta}(\mathbf{p}^w)$;
Second, we attach a virtual spring between
$\mathbf{p}^{w}$ and $\boldsymbol{\eta}(\mathbf{p}^w)$.
Then, their object-mesh displacement $\mathbf{u}(\mathbf{p}^w,\boldsymbol{\eta}(\mathbf{p}^w),t) \in \mathbb{R}^{3}$ can be allocated among the nodes of the intersection element $L(\mathbf{p}^w)$ via:
\begin{equation}
    \mathbf{u}(\mathbf{p}^w,\boldsymbol{\eta}(\mathbf{p}^w),t) = \!\! \sum_{\mathbf{n}^i \in L(\mathbf{p}^w)} \! 
    ^{}{N}_{w,i}(\boldsymbol{\eta}(\mathbf{p}^w),\mathbf{n}^i)\mathbf{u}(\mathbf{n}^i, t)
    \label{eq:element_shape_function}
\end{equation}
where $\mathbf{n}^i$ is the node
whose global index in the base mesh is $i$;
the scalar ${N}_{w,i}(\boldsymbol{\eta}(\mathbf{p}^w),\mathbf{n}^i)$ is the allocation weight of $\mathbf{n}^i$ with respect to $\boldsymbol{\eta}(\mathbf{p}^w)$, 
which is computed using the finite element (FE) shape function~\cite{kattan2010matlab} of $L(\mathbf{p}^w)$;
$\mathbf{u}(\mathbf{n}^i, t) \in \mathbb{R}^3$ is the allocated displacement vector of $\mathbf{n}^i$;
Finally, by assembling the element object-mesh mappings in equation \eqref{eq:element_shape_function} for all object points $\mathbf{p}$, we can construct the global object-mesh mapping:
\begin{equation}
    \mathbf{u}(\mathbf{p},\boldsymbol{\eta}(\mathbf{p}),t) = \mathbf{N}(\boldsymbol{\eta}(\mathbf{p}), \mathbf{n})\mathbf{u}(\mathbf{n}, t)
    \label{eq:global_shape_function}
\end{equation}
where $\mathbf{N}(\boldsymbol{\eta}(\mathbf{p}), \mathbf{n}) \in \mathbb{R}^{3L \times 3N}$ is the global object-mesh allocating matrix;
$\mathbf{u}(\mathbf{n}, t) \in \mathbb{R}^{3N}$ is the allocated nodal displacement field. 
However, this nodal displacement field still has high dimensions.
The next step is to derive a reduced deformation basis by performing modal truncation on $\mathbf{u}(\mathbf{n}, t)$.

\subsection{Modal Truncation on Base Mesh}
Given a set of arbitrarily assigned physical properties (including Young's modulus $E$, Poisson's ratio $v$, and the total mass $M$),
we compute the general stiffness matrix $\mathbf{K} \in \mathbb{R}^{3N \times 3N}$ and mass matrix $\mathbf{M} \in \mathbb{R}^{3N \times 3N}$ of the base mesh using the finite element method (FEM).
Following the approaches in~\cite{bathe2006finite}, we compute the undamped free deformation (or vibration) modes of the base mesh by solving the following generalized eigenproblem:
\begin{equation}
\mathbf{K}\boldsymbol{\phi}=\omega^2\mathbf{M}\boldsymbol{\phi}
\end{equation}
which has $3N$ eigensolutions:
$$ 
\begin{matrix}
 \begin{Bmatrix}
 (\omega_1^2,\boldsymbol{\phi}^1), &
 ...&
 (\omega_{3N}^2,\boldsymbol{\phi}^{3N})
 \end{Bmatrix},
 & 
 \omega_1^2 \le  ... \le \omega_{3N}^2
\end{matrix}
$$
where the eigenvector $\boldsymbol{\phi}^j \in \mathbb{R}^{3N}$ is the $j$-th mode shape vector while $\omega_j$ is the corresponding natural frequency.
Then, we form the mode shape matrix of the base mesh using these frequency-ordered eigenvectors:
$$ 
  \mathbf{\Phi} = 
  \begin{bmatrix}
  \boldsymbol{\phi}^1 & \boldsymbol{\phi}^2 & ... & \boldsymbol{\phi}^{3N}
  \end{bmatrix}.
$$
Note that $\mathbf{\Phi} \in \mathbb{R}^{3N \times 3N}$ is $\mathbf{M}$-orthonormal and $\mathbf{K}$-orthogonal.
We then discard high-frequency modes and normalize the retained $m$ ($m \ll 3N$) modes.
Afterward, we 
use the modal truncation to approximate the nodal displacements with:
\begin{equation}
    \mathbf{u}(\mathbf{n},t) 
    = \sum_{j=1}^m \boldsymbol{\phi}_n^j s_j(t)
    \label{eq:modal_truncation}
\end{equation}
where $\boldsymbol{\phi}_n^j \in \mathbb{R}^{3N}$ is the $j$-th normalized mode whose coefficient is $s_j$.
Note that these modes are computed using the deformation model of the base mesh rather than the object.
Even though different deformation models lead to different sets of mode shapes, each set satisfies similar orthogonal and frequency-ordered properties to span a low-dimensional deformation feature space.
Thus, the truncated modal coefficients $\mathbf{s}(t) =\left \{ s_1(t),....,s_m(t) \right \}^T$ can
form a unique and compact deformation representation.

\subsection{Model-free Deformation Basis}
Choosing $\mathbf{s}(t)$ to be our deformation features, based on the above discussions, we derive the following modal-based deformation basis $\mathbf{B}(\mathbf{p},\mathbf{n}) \in \mathbb{R}^{3L \times m}$:
\begin{equation}
\begin{aligned}
       \mathbf{u}(\mathbf{p},\mathbf{n},t) 
       & =  \sum_{j=1}^{m}
       \underbrace{\sum_{i=1}^{N}\mathbf{N}_i(\boldsymbol{\eta}(\mathbf{p}),\mathbf{n}^i)[\boldsymbol{\phi}_n^j]_i}_{\mathbf{b}^j(\mathbf{p},\mathbf{n})}
       s_j(t) 
       \\ &
       = \mathbf{N}(\boldsymbol{\eta}(\mathbf{p}),\mathbf{n}) \mathbf{\Phi}_n \mathbf{s}(t)
       = \mathbf{B}(\mathbf{p},\mathbf{n})\mathbf{s}(t)
\end{aligned}
\label{eq:shape_basis}
\end{equation}
where $\mathbf{N}_i(\boldsymbol{\eta}(\mathbf{p}),\mathbf{n}^i) \in \mathbb{R}^{3L \times 3}$ is the local object-mesh allocating matrix for $\mathbf{n}^i$;
the sub-vector $[\boldsymbol{\phi}_n^j]_i \in \mathbb{R}^{3}$  consists of the variables from $\boldsymbol{\phi}_n^j$ that correspond to $\mathbf{n}^i$;
the vector $\mathbf{b}^j(\mathbf{p},\mathbf{n}) \in \mathbb{R}^{3L}$ (the column vector of the basis $\mathbf{B}(\mathbf{p},\mathbf{n})$) can be regarded as the Ritz basis function in~\cite{krysl2001dimensional} (unlike~\cite{krysl2001dimensional}, we formulate the \textit{priori} vectors by combining the base mesh's low-frequency modes according to the object-mesh mapping);
since the modes are generated in a model-free manner, we call $\mathbf{\Phi}_n \in \mathbb{R}^{3N \times m}$ the truncated reference mode shape matrix and the deformation feature space they span the truncated reference modal space.
Note that our deformation features are designed by approximating the object's global deformation using a linear low-dimensional combination of the whole-body deformation modes.
Thus, unlike the existing geometric deformation features that describe the (deformed) positions of some measurable parts of the object, our deformation features represent the global deformation of the object.
By setting the feature dimension $m$ to be equal or greater than 
the measurement dimension, 
over-constrained solutions of $\mathbf{s}(t)$ can be still computed even with 
partial 
stereo measurements.
It should also be mentioned that our deformation basis can be formulated for linear, planar, and volumetric objects and can deal with stereo measurements of different image features.

\section{Deformation Feature Computation}
This section presents the method to compute our modal-based 
deformation features
in the truncated reference modal space 
from stereo measurements of the object.

\subsection{Surface Samplings}
As the proposed deformation basis directly transforms between object points and the deformation features $\mathbf{s}(t)$, we can compute $\mathbf{s}(t)$ by simply sampling points from the image features tracked on the object (with respect to Assumption 2).
Given the tracked image features $\mathbf{y}(t) \in \mathbb{R}^{\varrho(t) \times 4}$ (i.e. a pixel vector in the pair of 2D images obtained from the stereo cameras),
we sample $l$ 
points from them and compute the 3D positions of the points:
\begin{equation}
    \mathbf{x}(\mathbf{p}_s,t)= f(\mathbf{y}(t))
\end{equation}
where the sampled points
$\mathbf{p}_s$ are called the surface samplings, and $\mathbf{x}(\mathbf{p}_s,t) \in \mathbb{R}^{3l}$ is their position vector;
$f: \mathbb{R}^{\varrho(t) \times 4} \to \mathbb{R}^{3l}$ is a nonlinear function of the point sampling and reconstruction process.
Fig.~\ref{fig:surface_features} shows some examples of the point sampling from different image features.

\begin{figure}[t]
    \vspace{0.1cm}
    \centering
    \includegraphics[width=0.85\linewidth]{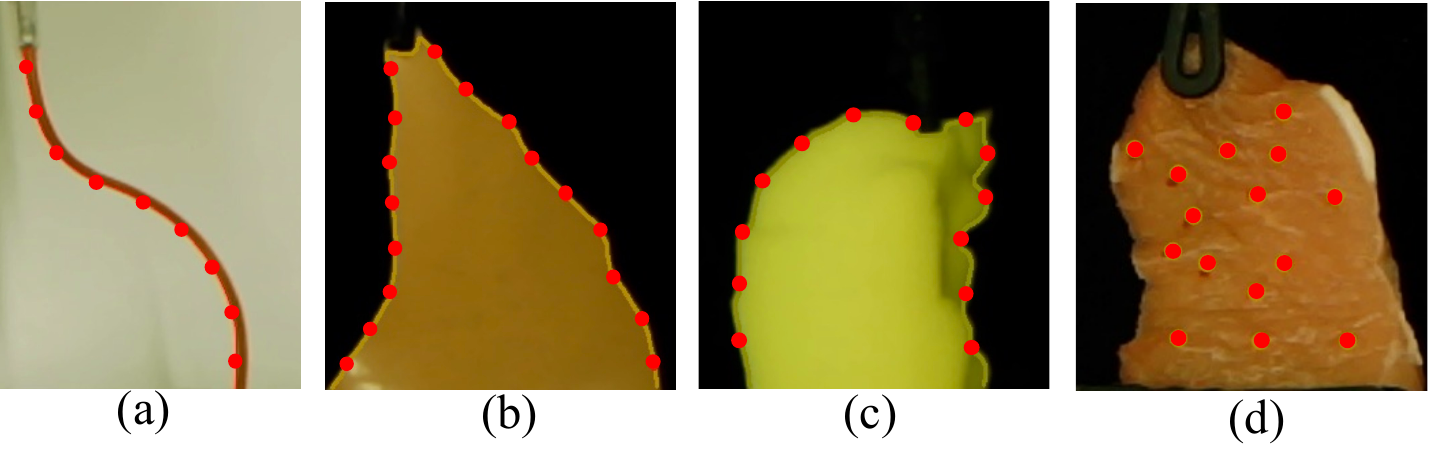}
    \vspace{-0.5cm}
    \caption{Examples of obtaining surface samplings (the red dots) from different image features: 
    (a) sampling points from the curve feature of a linear object by arc length;
    (b/c) sampling points from the contour feature of a planar/volumetric object by arc length;
    (d) tracking point features on the surface of an animal (pork) tissue.}
    \label{fig:surface_features}
    \vspace{-0.2cm}
\end{figure}

\subsection{Object-Mesh Mapping of Surface Samplings}
To establish the relationship between the surface samplings $\mathbf{p}_s$ and our deformation features,
we first formulate the object-mesh mapping of $\mathbf{p}_s$ such that their corresponding part in the deformation basis can be computed.
Given $\mathbf{p}_s$, following the method in Section III-B, 
we can find their base mesh projections $\boldsymbol{\eta}(\mathbf{p}_s)$ and compute their object-mesh displacements $\mathbf{u}(\mathbf{p}_s,\boldsymbol{\eta}(\mathbf{p}_s),t) \in \mathbb{R}^{3l}$ with:
\begin{equation}
    \mathbf{u}(\mathbf{p}_s,\boldsymbol{\eta}(\mathbf{p}_s),t) = \mathbf{x}(\mathbf{p}_s,t) - \mathbf{x}(\boldsymbol{\eta}(\mathbf{p}_s))
    \label{eq:point_displacement}
\end{equation}
where $\mathbf{x}(\boldsymbol{\eta}(\mathbf{p}_s)) \in \mathbb{R}^{3l}$ is the position vector of $\boldsymbol{\eta}(\mathbf{p}_s)$.
For each of $\boldsymbol{\eta}(\mathbf{p}_s)$, we compute its element object-mesh mapping in equation~\eqref{eq:element_shape_function}.
Then, all the element mappings of $\boldsymbol{\eta}(\mathbf{p}_s)$ are assembled into the following matrix form:
\begin{equation}
    \mathbf{u}(\mathbf{p}_s,\boldsymbol{\eta}(\mathbf{p}_s),t) = \mathbf{N}_s(\boldsymbol{\eta}(\mathbf{p}_s),\mathbf{n}_s)\mathbf{u}(\mathbf{n}_s,t)
    \label{eq:raw_distribution}
\end{equation}
where 
$\mathbf{n}_s$ (of size $n$) are the surface samplings' allocation nodes with non-zero allocation weights;
$\mathbf{N}_s(\boldsymbol{\eta}(\mathbf{p}_s),\mathbf{n}_s) \in \mathbb{R}^{3l \times 3n}$ is the (local) object-mesh allocating matrix for surface samplings, which can be regarded as a sub-matrix consisting of the rows of $\boldsymbol{\eta}(\mathbf{p}_s)$ and the columns of $\mathbf{n}_s$ from the global object-mesh allocating matrix $\mathbf{N}(\boldsymbol{\eta}(\mathbf{p}), \mathbf{n})$;
$\mathbf{u}(\mathbf{n}_s,t) \in \mathbb{R}^{3n}$ is the nodal displacement vector of $\mathbf{n}_s$.
Afterwards, we compute $[\mathbf{\Phi}_n]_s \in \mathbb{R}^{3n \times m}$ by extracting the rows of $\mathbf{n}_s$ from $\mathbf{\Phi}_n$, using which the nodal displacements $\mathbf{u}(\mathbf{n}_s,t)$ can be represented with our deformation features using equation~\eqref{eq:modal_truncation}.

As our deformation features describe the global deformation rather than the local point positions,
we do not require the surface samplings to be the same object points.
During the deformation, the surface samplings can slip on the object surface, whose influence can be absorbed in their object-mesh displacements.
This benefits the implementation of our method for practical applications.
By comparison, some existing works require artificial markers~\cite{yu2022global}~\cite{navarro2016automatic} or
textured objects~\cite{shetab2022rigid}
to track points.

\begin{figure}[b]
    \centering
    \vspace{-0.2cm}
    \includegraphics[width=\linewidth]{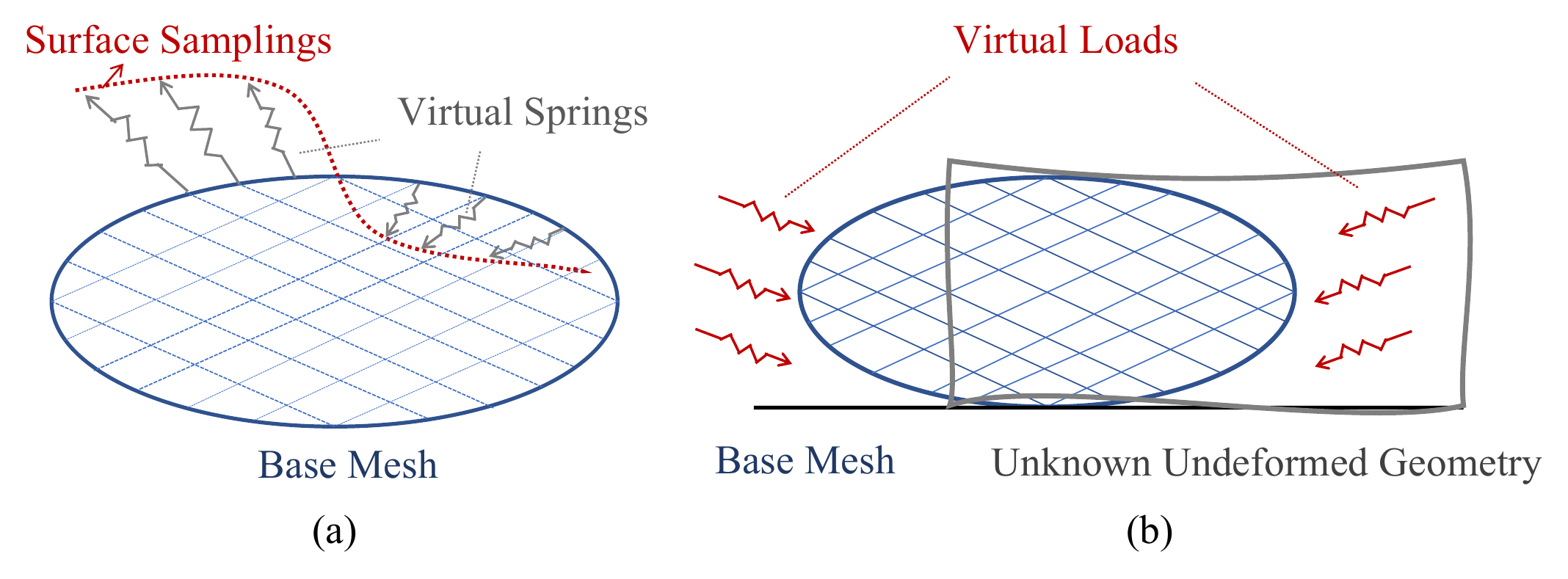}
    \vspace{-0.7cm}
    \caption{Illustrations of:
    (a) the virtual springs between surface samplings and their bash mesh projections;
    (b) the virtual loads to deform the base mesh to fit the undeformed object geometry.
    }
    \label{fig:mesh_virtual}
\end{figure}

\subsection{Deformation Feature Computation Matrix}
Computing our deformation features from 3D position measurements of surface samplings requires the inverse process of equation \eqref{eq:raw_distribution}.
However, as the surface samplings $\mathbf{p}_s$ are usually sparser than their allocated nodes $\mathbf{n}_s$ ($l \le n$), the direct inverse of equation~\eqref{eq:raw_distribution} is under-determined.
To avoid solving the under-determined problem,
according to the virtual spring attachment method\cite{terzopoulos1987symmetry}~\cite{pentland1991closed},
we specify displacement measurements as virtual forces that deform the object to fit the measured points through virtual springs (Fig.~\ref{fig:mesh_virtual}(a)).
In this way, given the principle of virtual work,
the point deformation $\mathbf{u}(\mathbf{p}_s,\boldsymbol{\eta}(\mathbf{p}_s),t)$ can be allocated among the mesh nodes using the transpose (rather than the inverse) of $\mathbf{N}_s(\boldsymbol{\eta}(\mathbf{p}_s),\mathbf{n}_s)$:
\begin{equation}
\begin{aligned}
    \mathbf{r}(\mathbf{n}_s,t)  
    & = \mathbf{N}_s^T(\boldsymbol{\eta}(\mathbf{p}_s),\mathbf{n}_s)\mathbf{u}(\mathbf{p}_s,\boldsymbol{\eta}(\mathbf{p}_s),t)
\end{aligned}
\label{eq:virtual_f}
\end{equation}
where $\mathbf{r}(\mathbf{n}_s,t) \in \mathbb{R}^{3n}$ is the allocated virtual force vector of nodes $\mathbf{n}_s$.
It can be further transformed to the modal forces $\widetilde{\mathbf{r}}(t) \in \mathbb{R}^{m}$ with:
\begin{equation}
    \widetilde{\mathbf{r}}(t) = [\mathbf{\Phi}_n]_s^T\mathbf{r}(\mathbf{n}_s,t).
\end{equation}
Nevertheless, in equation~\eqref{eq:virtual_f}, displacement measurements are allocated as virtual forces.
We need to add a rectified term between the force variables and the displacement variables in the truncated reference modal space.
Following the relative derivations in~\cite{pentland1991closed}, we perform the force-displacement rectification with:
\begin{equation}
    \mathbf{s}(t) = (\widetilde{\mathbf{K}}+\mathbf{I}_6)^{-1}\widetilde{\mathbf{r}}(t)
    \label{eq:modal_rectify}
\end{equation}
where $\mathbf{I}_{6} \in \mathbb{R}^{m \times m}$ is a diagonal matrix whose first six diagonal elements are ones, while the other elements are zeros;
the diagonal matrix $\widetilde{\mathbf{K}} \in \mathbb{R}^{m \times m}$ is the normalized modal stiffness matrix (unlike~\cite{pentland1991closed}, we compute a normalized version with $\widetilde{\mathbf{K}}=\mathbf{\Phi}_n^T\mathbf{K}\mathbf{\Phi}_n$).
The matrix $\mathbf{I}_6$ is added to $\widetilde{\mathbf{K}}$ because under the base mesh frame, eigenvalues of the first six modes (the rigid motion modes) are zeros.

Based on the above discussions, we can compute our 3D deformation features via:
\begin{equation}
\mathbf{s}(t) = \mathbf{D}_\Phi(\mathbf{n}_s)\mathbf{D}_N(\boldsymbol{\eta}(\mathbf{p}_s),\mathbf{n}_s) \mathbf{u}(\mathbf{p}_s,\boldsymbol{\eta}(\mathbf{p}_s),t)
\label{eq:feature_computation}
\end{equation}
where $\mathbf{D}_N(\boldsymbol{\eta}(\mathbf{p}_s),\mathbf{n}_s) \in \mathbb{R}^{3n \times 3l}$ is the mesh projection matrix that projects point measurements of the object to the nodal space of the base mesh with:
\begin{equation}
    \mathbf{D}_N(\boldsymbol{\eta}(\mathbf{p}_s),\mathbf{n}_s) = \mathbf{N}_s^T(\boldsymbol{\eta}(\mathbf{p}_s),\mathbf{n}_s)
\end{equation}
and $\mathbf{D}_\Phi(\mathbf{n}_s) \in \mathbb{R}^{m \times 3n}$ is the modal projection matrix that transforms nodal variables of the base mesh to the deformation features in the truncated reference modal space with:
\begin{equation}
    \mathbf{D}_\Phi(\mathbf{n}_s) = (\widetilde{\mathbf{K}}+\mathbf{I}_6)^{-1}[\mathbf{\Phi}_n]_s^T.
\end{equation}

Although the use of virtual forces is a well-established technique in some shape modeling works~\cite{terzopoulos1987symmetry}~\cite{pentland1991closed}, the definitions and computations of the forces are not physically accurate.
Those works define the virtual forces to deform the object to fit sensor measurements, and then solve the equilibrium displacements (that constitute the recovered shape) numerically or through several iterations.
By comparison, we only use the virtual forces to project stereo measurements to the deformation feature space such that the object deformation is uniquely represented (rather than accurately recovered).
To ensure the unique representation, we formulate the deformation feature computation matrix whose rank satisfies
\begin{equation}
    rank(\mathbf{D}_\Phi(\mathbf{n}_s)\mathbf{D}_N(\boldsymbol{\eta}(\mathbf{p}_s),\mathbf{n}_s))=m,
    \label{eq:full_rank}
\end{equation}
and also require the dimension of surface samplings satisfies $l \ge \frac{m}{3}$.
In addition, we need to further deal with the modeling uncertainties introduced by the virtual forces when designing control laws.
\begin{remark}
 The deformation feature computation matrices are formulated at the rest configuration. They do not need to be reformulated unless some one-to-one correspondences of the surface samplings are lost due to occlusions, in which case we can re-assemble the matrices $\mathbf{N}_s(\boldsymbol{\eta}(\mathbf{p}_s),\mathbf{n}_s)$ and $[\mathbf{\Phi}_n]_s$ according to the changed configurations of $\mathbf{p}_s$ and $\mathbf{n}_s$.
\end{remark}

\section{Deformation Jacobian Matrix with Unknown Modal Parameters}
To formulate robotic manipulation strategies, we need to 
analyze how robot manipulation affects our deformation features.
In Euclidean space, the dynamics of a deformable object 
under external manipulation 
are described by partial differential equations~\cite{junkins1993introduction} with infinite dimensions.
To investigate the dynamics in our deformation feature space,
we consider the following low-dimensional approximation of the object dynamics using the model reduction techniques in~\cite{barbivc2012fem}~\cite{von2013efficient}:
\begin{equation}
    \widetilde{\mathbf{M}}_o\ddot{\mathbf{s}}(t) = 
    \widetilde{\mathbf{F}}(\mathbf{s}(t), \dot{\mathbf{s}}(t)), \mathbf{x}(\mathbf{p}_r,t))
    \label{eq:modal_eom}
\end{equation}
where $\widetilde{\mathbf{M}}_o \in \mathbb{R}^{m \times m}$ is the mass matrix of the object in the truncated reference modal space;
$\dot{\mathbf{s}} \in \mathbb{R}^m$ and $\ddot{\mathbf{s}} \in \mathbb{R}^m$ are the first and second order derivative of our deformation features $\mathbf{s}(t)$;
$\mathbf{x}(\mathbf{p}_r,t) \in \mathbb{R}^{3k}$ ($3k \le m$) is the position vector of the manipulation points $\mathbf{p}_r$;
the nonlinear term $\widetilde{\mathbf{F}}(\cdot) \in \mathbb{R}^m$ is the superposition of elastic forces, damping forces, and external forces
of the object in the truncated reference modal space.
However,
equation~\eqref{eq:modal_eom}
is hard to solve with unknown deformation models of the object.
Focusing on kinematic control under the quasi-static assumption (Assumption 3), 
this paper does not consider the object dynamics in our controller design.
Instead, we formulate a deformation Jacobian matrix between $\dot{\mathbf{s}}(t)$ and $\dot{\mathbf{x}}(\mathbf{p}_r,t) \in \mathbb{R}^{3k}$, the first order derivative of $\mathbf{x}(\mathbf{p}_r,t)$, and 
treat the unknown object models as low-dimensional unknown parameters in the Jacobian matrix.

\subsection{Quasi-Static Deformation in Truncated Reference Modal Space}
Under the quasi-static assumption, the dominant effect of the object deformation is its potential energy.
Thus, we investigate the potential energy of the object to formulate the relationship between $\dot{\mathbf{s}}(t)$ and $\dot{\mathbf{x}}(\mathbf{p}_r,t)$.
Note that the quasi-static assumption is widely used in the existing works of deformation control (such as~\cite{navarro2016automatic}~\cite{hu2018three}~\cite{yu2022global}~\cite{hirai2000indirect}).
Unlike these works studying the potential energy in Euclidean space, we formulate the potential energy in the truncated reference modal space using the subspace domain integral~\cite{an2008optimizing}:
\begin{equation}
    W(\mathbf{s}(t)) = \int_{\Omega_b} \psi(\mathbf{p};\mathbf{s}(t))d{\Omega}_{\mathbf{p}}
\end{equation}
where $W(\mathbf{s}(t)): \mathbb{R}^m \to \mathbb{R}$ is the subspace potential energy function;
$\psi(\mathbf{p};\mathbf{s}(t))$ is the non-negative energy density at the object point $\mathbf{p}$ in the point domain $\Omega_{\mathbf{p}}$.
Then, the equation of equilibrium in the modal space follows:
\begin{equation}
    \frac{\partial W}{\partial \mathbf{s}}(\mathbf{s}(t)) + \widetilde{\mathbf{f}}(\mathbf{x}(\mathbf{p}_r,t)) = \mathbf{0}.
\label{eq:equi_eq}
\end{equation}
Solving this equation requires modeling knowledge of the object.
To avoid that, we locally focus on a linear approximation of the relationship between $\delta\mathbf{s} \in \mathbb{R}^m$ (a small change of $\mathbf{s}(t)$) and $\delta\mathbf{x}(\mathbf{p}_r) \in \mathbb{R}^{3k}$ (a small change of $\mathbf{x}(\mathbf{p}_r,t)$).
To construct the local approximation,
inspired by~\cite{navarro2016automatic}~\cite{zhong2019dual}, we linearize equation \eqref{eq:equi_eq} around the equilibrium into:
\begin{equation}
    \left ( 
         \frac{\partial^2 W}{\partial \mathbf{s}\partial \mathbf{s}} +
         \frac{\partial \widetilde{\mathbf{f}}}{\partial \mathbf{s}} 
    \right ) \delta\mathbf{s}  +
    \left ( 
         \frac{\partial^2 W}{\partial \mathbf{s}\partial \mathbf{x}(\mathbf{p}_r,t)} +
         \frac{\partial \widetilde{\mathbf{f}}}{\partial \mathbf{x}(\mathbf{p}_r,t)} 
    \right ) \delta\mathbf{x}(\mathbf{p}_r)  = \mathbf{0}
\label{eq:linear_equil_eq}
\end{equation}
where $\partial W / \partial \mathbf{s}$ and $\widetilde{\mathbf{f}}(\mathbf{x}(\mathbf{p}_r,t))$ are regarded as internal and external forces
in the modal space.
We denote their counterparts in Euclidean space by $\mathbf{f}_{int}(t)$ and $\mathbf{f}(t)$, respectively.
Then, 
equation~\eqref{eq:linear_equil_eq} can be rewritten as:
\begin{equation}
   \begin{aligned}
    &
    \underbrace{
     \left ( 
     \frac{\partial (\partial W / \partial \mathbf{s})}{\partial \mathbf{s}} +
         \frac{\partial \widetilde{\mathbf{f}}}{\partial \mathbf{s}} 
    \right )
    }_{\widetilde{\mathbf{A}}} \delta\mathbf{s} + 
    \\ &
    \underbrace{
    \left (
    \frac{\partial (\partial W / \partial \mathbf{s})}{\partial \mathbf{f}_{int}}  \frac{\partial \mathbf{f}_{int}}{\partial \mathbf{x}(\mathbf{p}_r,t)} + 
   \frac{\partial \widetilde{\mathbf{f}}}{\partial \mathbf{f}} \frac{\partial \mathbf{f}}{\partial \mathbf{x}(\mathbf{p}_r,t)}
   \right ) 
    }_{\widetilde{\mathbf{B}}} \delta\mathbf{x}(\mathbf{p}_r) = \mathbf{0}
 \end{aligned}
 \label{eq:linear_equili_rewrite}
\end{equation}
where 
$\widetilde{\mathbf{A}} \in \mathbb{R}^{m \times m}$ and $\widetilde{\mathbf{B}} \in \mathbb{R}^{m \times 3k}$ are unknown constant relationship matrices computed with unknown stiffness matrices.
Given Assumption 4 and equation~\eqref{eq:full_rank},
the relationship matrices $\widetilde{\mathbf{A}}$ and $\widetilde{\mathbf{B}}$ have full column ranks.
Then, for slow robot manipulation (where $\dot{\mathbf{s}}(t) \approx \delta\mathbf{s} / \delta t$ and $\dot{\mathbf{x}}(\mathbf{p}_r,t) \approx \delta \mathbf{x}(\mathbf{p}_r) / \delta t$ hold for $\delta t$, a small time interval),
we can establish a Jacobian relationship between $\dot{\mathbf{s}}(t)$ and $\dot{\mathbf{x}}(\mathbf{p}_r,t)$.

However, formulating or estimating a deformation Jacobian matrix is a common technical challenge in the related works.
To derive our deformation Jacobian matrix analytically,
as shown in Fig.~\ref{fig:modal_spring}, we introduce an intermediate model-free variable $\widetilde{\mathbf{u}} \in \mathbb{R}^{m}$, the modal displacements of the base mesh (rather than the object) under
robot manipulation.
To compute $\widetilde{\mathbf{u}}$, we project the robot manipulation on the base mesh to the modal space.
Based on Assumption 1, the manipulation points $\mathbf{p}_r$ are fixed object points
whose positions can be obtained using robot measurements.
Given $\mathbf{p}_r$, following the derivations in Section III and IV,
$\widetilde{\mathbf{u}}$ can be computed with:
\begin{equation}
   {\widetilde{\mathbf{u}}}(t) = (\widetilde{\mathbf{K}}+\mathbf{I}_6)^{-1}[\mathbf{\Phi}_n]_r^T\mathbf{N}_r^T(\boldsymbol{\eta}(\mathbf{p}_r),\mathbf{n}_r) (\mathbf{x}(\mathbf{p}_r,t) - \mathbf{x}(\boldsymbol{\eta}(\mathbf{p}_r)))
   \label{eq:jacobian_mesh}
\end{equation}
where
$\boldsymbol{\eta}(\mathbf{p}_r)$ are the base mesh projections of $\mathbf{p}_r$, whose position vector is
$\mathbf{x}(\boldsymbol{\eta}(\mathbf{p}_r))  \in \mathbb{R}^{3k}$;
$\mathbf{n}_r$ is the manipulation points' allocation nodes of size $h$;
$\mathbf{N}_r(\boldsymbol{\eta}(\mathbf{p}_r),\mathbf{n}_r) \in \mathbb{R}^{3k \times 3h}$ is the (local) object-mesh allocating matrix for the manipulation points;
$[\mathbf{\Phi}_n]_r \in \mathbb{R}^{3h \times m}$ consists of the rows of $\mathbf{n}_r$ extracted from $\mathbf{\Phi}_n$.
Nevertheless, ${\widetilde{\mathbf{u}}}(t)$ only reflects deformation properties of the base mesh while our deformation features $\mathbf{s}(t)$ (computed from the object's stereo measurements) reflect the unknown deformation properties of the object.
We must take the modeling uncertainties between the base mesh and the object into consideration.

\begin{figure}[t]
    \centering
    \includegraphics[width=\linewidth]{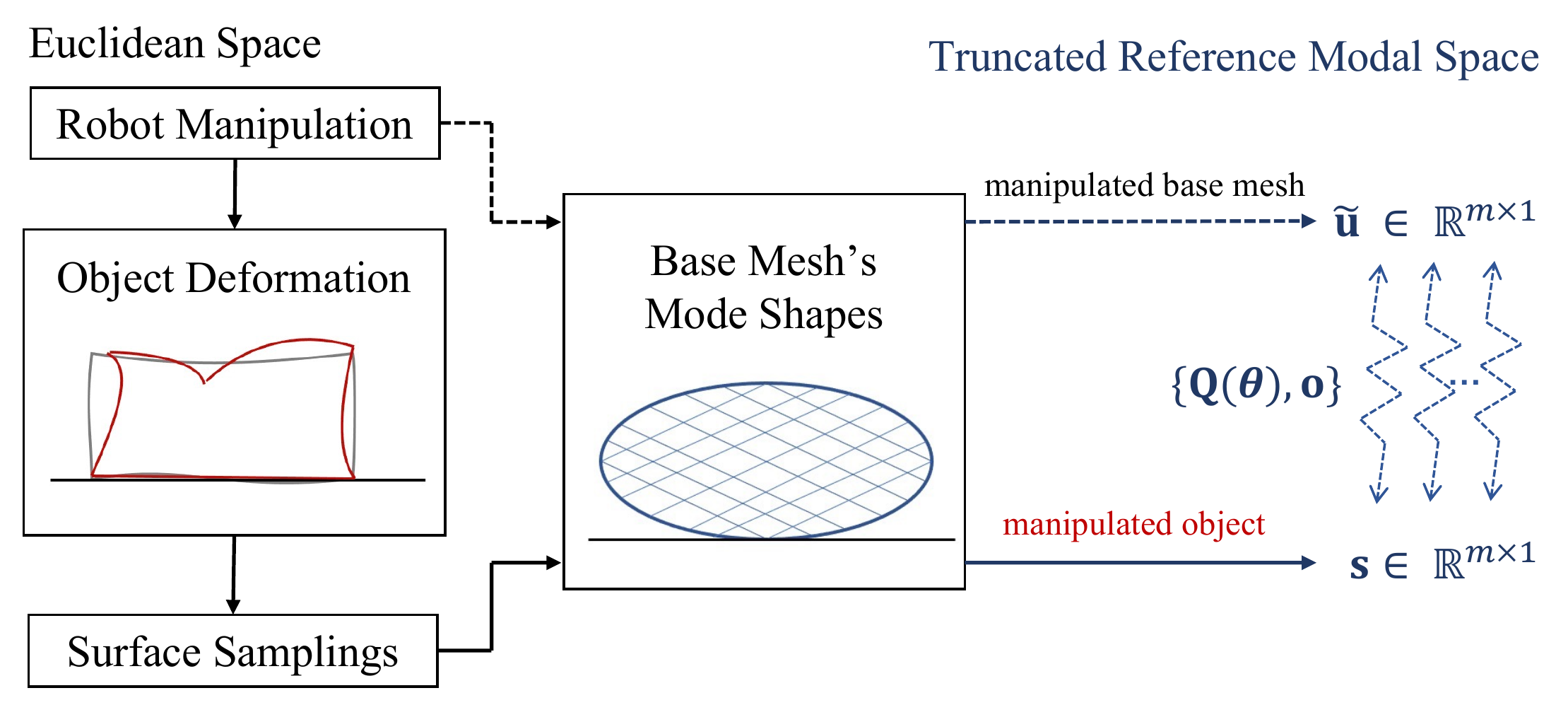}
    \vspace{-0.7cm}
    \caption{Illustration of the deformation Jacobian derivation.
    The solid lines indicate the direct path from robot manipulation to our deformation features while the dotted lines indicate the path we derive the Jacobian matrix.
    The blue springs illustrate the local elastic relations between the modal displacements of the base mesh and the object in equation~\eqref{eq:jacobian_parameter}.
    } 
    \label{fig:modal_spring}
    \vspace{-0.2cm}
\end{figure}

\subsection{Unknown Modal Parameters}
Our method investigates all the 
modeling uncertainties
in the truncated reference modal space.
Given equation~\eqref{eq:shape_basis}, the deformation features $\mathbf{s}(t)$ can be regarded as the generalized displacements of the object in the truncated reference modal space.
As the modal space is generated using normalized modes, 
according to~\cite{bathe2006finite}~\cite{agudo2014good}, 
the influence of material and mass properties can be absorbed in the amplitudes of modal displacements.
To further deal with the unknown geometric model of the object,
we assume some unknown but constant virtual loads (Fig.~\ref{fig:mesh_virtual}(b)) that deform the base mesh to fit the undeformed object geometry.
In this way, according to~\cite{bathe2006finite}, all the modeling uncertainties can be absorbed in the amplitudes of modal displacements.
Then, 
inspired by~\cite{navarro2014visual}, 
we approximate the local elastic relations between $\mathbf{s}(t)$ and $\widetilde{\mathbf{u}}(t)$ with the following affine deformation model~\cite{ogden1997non}:
\begin{equation}
    \mathbf{s}(t) = \mathbf{Q}\widetilde{\mathbf{u}}(t)+\mathbf{o}
    \label{eq:jacobian_parameter}
\end{equation}
where the unknown constant matrix $\mathbf{Q} \in \mathbb{R}^{m \times m}$ and the unknown constant vector $\mathbf{o} \in \mathbb{R}^{m}$ represent deformation properties of the object and the base mesh.
Unlike the affine model used in~\cite{navarro2014visual} that approximates deformation in Euclidean space, this affine model defined in the modal space is decoupled and low-dimensional.
Thus $\mathbf{Q}$ is a diagonal matrix:
\begin{equation}
    \mathbf{Q}(\boldsymbol{\theta}) = diag{
    \begin{bmatrix}
  \theta_1, & \theta_2,  & ... & \theta_m
   \end{bmatrix}
    }
\end{equation}
consisting of a low-dimensional unknown modal parameter vector $\boldsymbol{\theta} = {
    \begin{bmatrix}
  \theta_1, & \theta_2,  & ... & \theta_m
   \end{bmatrix}
    }^T \in \mathbb{R}^{m}$.
\begin{remark}
Note that the linearization of equation~\eqref{eq:equi_eq} and the affine deformation model~\eqref{eq:jacobian_parameter}
are local quasi-static approximations of the object deformation.
Similar techniques are widely used in the existing Jacobian-based deformation controllers (such as~\cite{navarro2016automatic}~\cite{hirai2000indirect}~\cite{navarro2014visual}~\cite{yu2022shape}), 
whose aim is not developing an accurate deformation model but approximating how slow manipulation in a small time interval transforms into changes of the deformation features.
Under Assumption 3, our method uses the approximations in a closed-loop controller to determine proper combinations of the manipulation velocities that would result in the desired changes of deformation features.
We also design adaptive laws to online update the approximations such that their effective ranges can be continuously extended.
\end{remark}

\subsection{Deformation Jacobian Matrix}
Based on equation~\eqref{eq:jacobian_mesh} and~\eqref{eq:jacobian_parameter}, we derive the following Jacobian relationship:
\begin{equation}
  \begin{aligned}
      \dot{\mathbf{s}}(t) 
      & = \mathbf{Q}(\boldsymbol{\theta})(\widetilde{\mathbf{K}}+\mathbf{I}_6)^{-1}[\mathbf{\Phi}_n]_r^T\mathbf{N}_r^T(\boldsymbol{\eta}(\mathbf{p}_r) ,\mathbf{n}_r)\dot{\mathbf{x}}(\mathbf{p}_r,t) \\
      & = \mathbf{J}_s(\boldsymbol{\theta}) \dot{\mathbf{x}}(\mathbf{p}_r,t)
  \end{aligned}
  \label{eq:J}
\end{equation}
where $\mathbf{J}_s(\boldsymbol{\theta}) \in \mathbb{R}^{m \times 3k}$ is our deformation Jacobian matrix.
Unlike some existing works~\cite{navarro2016automatic}~\cite{navarro2018fourier}~\cite{lagneau2020active} that treat the whole deformation Jacobian matrix as unknown, our Jacobian matrix is analytically formulated using mode shapes and with only low-dimensional unknown parameters.

\section{Adaptive Deformation Controller}
This section designs an adaptive deformation controller with online estimation of the unknown Jacobian parameters.
The controller's block diagram is shown in Fig.~\ref{fig:algorithm_blockgram}.

\begin{figure}[b]
    \vspace{-0.2cm}
    \centering
    \includegraphics[width=\linewidth]{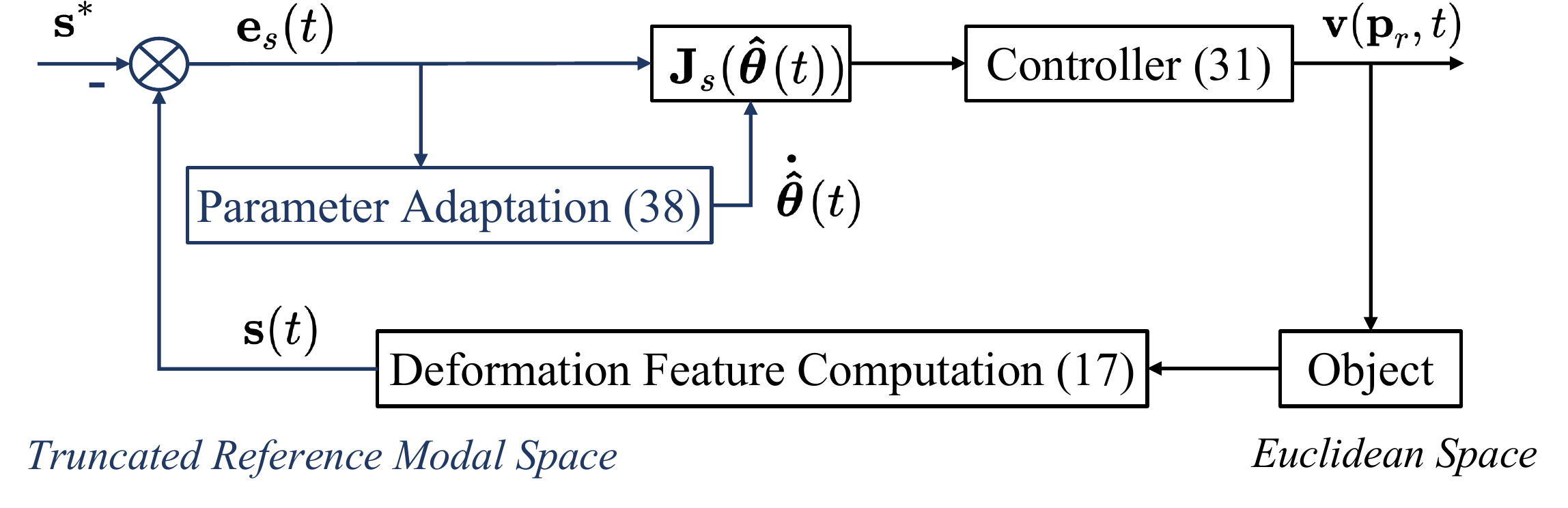}
    \vspace{-0.7cm}
    \caption{Block diagram of our deformation controller. We use the dark blue color to present the parts in the truncated reference modal space.
    }
    \label{fig:algorithm_blockgram}
\end{figure}

\subsection{Kinematic Control Law}
Given the deformation feature errors computed with:
\begin{equation}
    \mathbf{e}_s (t) =  \mathbf{s}(t) - {\mathbf{s}}^*,
    \label{eq:e_s}
\end{equation}
we propose the following kinematic control law using the transpose of the online estimated deformation Jacobian matrix:
\begin{equation}
    \mathbf{v}(\mathbf{p}_r,t) = -\mathbf{K}_s \mathbf{J}_s^T(\boldsymbol{\hat{\theta}}(t)) \mathbf{e}_s (t)
    \label{eq:v_cmd}
\end{equation}
where the linear velocities $\mathbf{v}(\mathbf{p}_r,t)$ are the manipulation commands for $\mathbf{p}_r$;
the transposed Jacobian matrix $\mathbf{J}_s^T(\boldsymbol{\hat{\theta}}(t)) \in \mathbb{R}^{3k \times m}$ is computed using the online estimated parameter $\boldsymbol{\hat{\theta}}(t) \in \mathbb{R}^m$;
$\mathbf{K}_s \in \mathbb{R}^{3k \times 3k}$ is a positive definite gain matrix called the feedback gain.
Then, we can obtain the closed-loop error dynamics by combining equation~\eqref{eq:v_cmd}, equation~\eqref{eq:e_s}, and equation~\eqref{eq:J}:
\begin{equation}
  \begin{aligned}
     \dot{\mathbf{e}}_s(t) 
     &= \mathbf{J}_s(\boldsymbol{\theta})( -\mathbf{K}_s \mathbf{J}_s^T(\boldsymbol{\hat{\theta}}(t)) \mathbf{e}_s (t)) \\
     &= - \mathbf{J}_s(\boldsymbol{\theta})\mathbf{K}_s\mathbf{J}_s^T(\boldsymbol{\hat{\theta}}(t)) \mathbf{e}_s (t).
  \end{aligned}
  \label{eq:closed-loop}
\end{equation}

We here discuss
the physical interpretation of our transposed-Jacobian-based controller.
Expanding the transpose of the deformation Jacobian matrix, we have:
\begin{equation}
\begin{aligned}
    \mathbf{J}_s^T(\boldsymbol{\hat{\theta}}(t)) 
    & = \mathbf{N}_r(\boldsymbol{\eta}(\mathbf{p}_r),\mathbf{n}_r)[\mathbf{\Phi}_n]_r(\widetilde{\mathbf{K}}+\mathbf{I}_6)^{-T}\mathbf{Q}^T(\boldsymbol{\hat{\theta}}(t)) \\
    & = \mathbf{N}_r(\boldsymbol{\eta}(\mathbf{p}_r),\mathbf{n}_r)[\mathbf{\Phi}_n]_r(\widetilde{\mathbf{K}}+\mathbf{I}_6)^{-1}\mathbf{Q}(\boldsymbol{\hat{\theta}}(t))
\end{aligned}
\end{equation}
where $\mathbf{Q}(\boldsymbol{\hat{\theta}}(t))$ and $(\widetilde{\mathbf{K}}+\mathbf{I}_6)^{-1}$ are the estimated parameter matrix and the force-displacement rectified term in the truncated reference modal space;
$\mathbf{N}_r(\boldsymbol{\eta}(\mathbf{p}_r),\mathbf{n}_r)$ represents the object-mesh mappings for the robot manipulation;
as $det(\mathbf{\Phi}_n) \ne 0$, the sub-matrix $[\mathbf{\Phi}_n]_r$ relates a small change of the displacements of the manipulated nodes to a small change of the intermediate model-free variable $\widetilde{u}(t)$.
Thus, the term:
\begin{equation}
    \mathbf{N}_r(\boldsymbol{\eta}(\mathbf{p}_r),\mathbf{n}_r)[\mathbf{\Phi}_n]_r \in \mathbb{R}^{3k \times m}
    \label{eq:reference_direction}
\end{equation}
provides a deformation Jacobian relationship for the base mesh manipulation, whose direction is determined by the object-mesh mapping and the low-frequency modes of the base mesh.
As discussed in~\cite{pentland1991recovery}, low-frequency modes are mainly determined by the low-order moments of inertia, thus compact bodies of similar sizes have similar low-frequency modes.
In this way, if the selected base mesh is of similar size to the object, $[\mathbf{\Phi}_n]_r$ provides mass-scaled similar low-frequency modes of the object's low-frequency modes.
We can thus regard $\mathbf{N}_r(\boldsymbol{\eta}(\mathbf{p}_r),\mathbf{n}_r)[\mathbf{\Phi}_n]_r$ as a reference direction of the object's deformation Jacobian matrix and simply initialize the unknown modal parameters to be ones (indicating no \textit{a priori} knowledge).
By comparison, some existing works~\cite{navarro2016automatic}~\cite{navarro2018fourier}~\cite{lagneau2020active} need an extra testing stage to initialize the whole deformation Jacobian matrix.

\subsection{Online Parameter Estimation}
Our parameter updating law is formulated by compensating for the estimation errors in the closed-loop system online.
To do that, we first add a zero term to the closed-loop error dynamics~\eqref{eq:closed-loop} and rewrite the equation as:
\begin{equation}
  \begin{aligned}
      \dot{\mathbf{e}}_s(t) 
      & = [-\mathbf{J}_s(\boldsymbol{\hat{\theta}}(t)) + (\mathbf{J}_s(\boldsymbol{\hat{\theta}}(t))- \mathbf{J}_s(\boldsymbol{\theta}))]\mathbf{K}_s\mathbf{J}_s^T(\boldsymbol{\hat{\theta}}(t)) \mathbf{e}_s (t)
      \\
      & = 
      -\mathbf{J}_s(\boldsymbol{\hat{\theta}}(t))\mathbf{K}_s\mathbf{J}_s^T(\boldsymbol{\hat{\theta}}(t)) \mathbf{e}_s (t) \\
      & + 
     (\mathbf{J}_s(\boldsymbol{\hat{\theta}}(t))- \mathbf{J}_s(\boldsymbol{\theta}))\mathbf{K}_s\mathbf{J}_s^T(\boldsymbol{\hat{\theta}}(t))\mathbf{e}_s (t).
  \end{aligned}
  \label{eq:closed_regression}
\end{equation}
Then, we define the following parameter errors $\Delta \boldsymbol{\theta}(t) \in \mathbb{R}^m$ with:
\begin{equation}
    \Delta \boldsymbol{\theta}(t) = \boldsymbol{\hat{\theta}}(t) - \boldsymbol{\theta}
\end{equation}
which can be used to linearize the estimation errors in the closed-loop error dynamics with:
\begin{equation}
    \begin{aligned}
    &(\mathbf{J}_s(\boldsymbol{\hat{\theta}}(t))- \mathbf{J}_s(\boldsymbol{\theta}))\mathbf{K}_s\mathbf{J}_s^T(\boldsymbol{\hat{\theta}}(t))\mathbf{e}_s (t) =( \mathbf{Q}(\boldsymbol{\hat{\theta}}(t)) - \mathbf{Q}(\boldsymbol{\theta})) \cdot \\
    &
    (\widetilde{\mathbf{K}}+\mathbf{I}_6)^{-1}[\mathbf{\Phi}_n]_r^T\mathbf{N}_r^T(\boldsymbol{\eta}(\mathbf{p}_r),\mathbf{n}_r)
    \mathbf{K}_s\mathbf{J}_s^T(\boldsymbol{\hat{\theta}}(t)) \mathbf{e}_s(t) \\
    &=\mathbf{Y}(\mathbf{e}_s(t) ,\boldsymbol{\hat{\theta}}(t) )\Delta\boldsymbol{\theta}(t) 
    \end{aligned}
    \label{eq:regression}
\end{equation}
where the regression matrix $\mathbf{Y}(\mathbf{e}_s(t) ,\boldsymbol{\hat{\theta}}(t)) \in \mathbb{R}^{m \times m}$ is independent of the unknown parameters $\boldsymbol{\theta}$.
Afterward, according to the Slotine-Li algorithm~\cite{slotine1987adaptive}, we propose the following online parameter updating law:
\begin{equation}
    \boldsymbol{\dot{\hat{\theta}}}(t)=-\Gamma^{-1}
    \mathbf Y^T(\mathbf{e}_s(t) ,\boldsymbol{\hat{\theta}}(t) ){\mathbf e}_s(t)
    \label{eq:theta_dot}
\end{equation}
where $\Gamma$ is a positive definite scalar called the parameter updating gain.

\begin{remark}
The linear parametrization of the closed-loop error dynamics enables us to design the adaptive control laws (the equation~\eqref{eq:v_cmd} and~\eqref{eq:theta_dot}) with guaranteed stability (the detailed proof will be given in the next subsection).
The objective of equation~\eqref{eq:theta_dot} is updating the parameters to minimize the deformation feature errors $\mathbf{e}_s(t)$ rather than identifying the true parameter values that can fit some input-output measurements of the system. 
Thus, unlike some existing shape deformation controllers (such as~\cite{navarro2018fourier}~\cite{lagneau2020active}~\cite{hu20193}~\cite{ficuciello2018fem}) that need to learn or estimate the deformation models/parameters accurately,
our controller is stable without requiring true parameter identification.
\end{remark}

\subsection{Stability Analysis}
In this subsection, we analyze the stability of our deformation controller under the proposed adaptive laws.

\newtheorem{theorem}{Theorem}
\begin{theorem}
Under the control of the controller~\eqref{eq:v_cmd} and the adaptive algorithm~\eqref{eq:theta_dot} for parameter estimation, the deformation feature errors are convergent in the following way:
\begin{equation}
    \lim_{t \to \infty} \mathbf{J}^T _s(\boldsymbol{\hat{\theta}}(t)) \mathbf{e}_s (t) = \boldsymbol{0}.
\end{equation}
\end{theorem}

\begin{proof}
We introduce the following Lyapunov candidate function:
\begin{equation}
    \begin{aligned}
    V(t)
    =
    & \frac{1}{2}\mathbf e_s^T(t){\mathbf e}_s(t) 
    +
    \frac{\Gamma}{2}\Delta\boldsymbol{\theta}^T(t)\Delta\boldsymbol{\theta}(t).
    \label{eq:V}
    \end{aligned}
\end{equation}
Differentiating equation \eqref{eq:V} with respect to time, we obtain:
\begin{equation}
  \begin{aligned}
    \dot V(t)
    &=  {\mathbf e}_s^T(t){\dot{\mathbf e}}_s(t)+
    \Gamma \Delta{\boldsymbol{\dot\theta}}^T(t)  \Delta\boldsymbol{\theta}(t) \\
    & = {\mathbf e}_s^T(t){\dot{\mathbf e}}_s(t)+
    \Gamma \boldsymbol{\dot{\hat{\theta}}}^T(t) \Delta\boldsymbol{\theta}(t). 
  \end{aligned}
    \label{eq:V_dot}
\end{equation}
Substituting the closed-loop error dynamics~\eqref{eq:closed_regression}, we get:
\begin{equation}
  \begin{aligned}
     \dot V(t) = 
    &-{\mathbf e}_s^T(t)\mathbf{J}_s(\boldsymbol{\hat{\theta}}(t))\mathbf{K}_s\mathbf{J}^T _s(\boldsymbol{\hat{\theta}}(t)) \mathbf{e}_s (t)\\
    &+{\mathbf e}_s^T(t)(\mathbf{J}_s(\boldsymbol{\hat{\theta}}(t))- \mathbf{J}_s(\boldsymbol{\theta}))\mathbf{K}_s
    \mathbf{J}^T _s(\boldsymbol{\hat{\theta}}(t)) \mathbf{e}_s (t)\\
    & +
    \Gamma \dot{\hat{\boldsymbol{\theta}}}^T(t) \Delta\boldsymbol{\theta}(t). 
  \end{aligned}
\end{equation}
Then injecting equation~\eqref{eq:theta_dot} into it, we have:
\begin{equation}
    \begin{aligned}
    \dot V(t) = 
    &-{\mathbf e}_s^T(t)\mathbf{J}_s(\boldsymbol{\hat{\theta}}(t))\mathbf{K}_s\mathbf{J}^T _s(\boldsymbol{\hat{\theta}}(t)) \mathbf{e}_s (t)\\
    &+{\mathbf e}_s^T(t)(\mathbf{J}_s(\boldsymbol{\hat{\theta}}(t))- \mathbf{J}_s(\boldsymbol{\theta}))\mathbf{K}_s
    \mathbf{J}^T _s(\boldsymbol{\hat{\theta}}(t)) \mathbf{e}_s (t)\\
    &-{\mathbf e}_s^T(t)\mathbf{Y}(\mathbf{e}_s(t) ,\boldsymbol{\hat{\theta}}(t) )\Delta\boldsymbol{\theta}(t).
    \end{aligned}
\end{equation}
Further injecting equation~\eqref{eq:regression}, we obtain:
\begin{equation}
    \begin{aligned}
    \dot V(t) = 
    &-{\mathbf e}_s^T(t)\mathbf{J}_s(\boldsymbol{\hat{\theta}}(t))\mathbf{K}_s\mathbf{J}^T _s(\boldsymbol{\hat{\theta}}(t)) \mathbf{e}_s (t)\\
    &+{\mathbf e}_s^T(t)(\mathbf{J}_s(\boldsymbol{\hat{\theta}}(t))- \mathbf{J}_s(\boldsymbol{\theta}))\mathbf{K}_s
    \mathbf{J}^T _s(\boldsymbol{\hat{\theta}}(t)) \mathbf{e}_s (t)\\
    &-{\mathbf e}_s^T(t)(\mathbf{J}_s(\boldsymbol{\hat{\theta}}(t))- \mathbf{J}_s(\boldsymbol{\theta}))\mathbf{K}_s
    \mathbf{J}^T _s(\boldsymbol{\hat{\theta}}(t)) \mathbf{e}_s (t).
    \end{aligned}
\end{equation}
Thus:
\begin{equation}
    \begin{aligned}
       \dot V(t) = 
    &-{\mathbf e}_s^T(t)\mathbf{J}_s(\boldsymbol{\hat{\theta}}(t))\mathbf{K}_s\mathbf{J}^T _s(\boldsymbol{\hat{\theta}}(t)) \mathbf{e}_s (t).
    \end{aligned}
    \label{eq:vdot1}
\end{equation}
It is obvious that $\dot{V}(t) \le 0$, and hence $V(t) \le V(0)$ (i.e. $V(t)$ is upper bounded), which means that $\mathbf{e}_s(t)$ and $ \boldsymbol{\hat{\theta}}(t)$ are bounded.
Under our quasi-static deformation assumption, $\dot{\mathbf{e}}_s(t)$ and $\boldsymbol{\dot{\hat{\theta}}}(t)$ are also bounded, which further leads to that $\dot{V}(t)$ is uniformly continuous.
Thus, according to barbalat's lemma~\cite{slotine1991applied}, we can conclude that $\lim_{t \to \infty} \mathbf{J}^T _s(\boldsymbol{\hat{\theta}}(t)) \mathbf{e}_s (t) = \boldsymbol{0}$.
\end{proof}
Therefore, with the proposed controller, the deformation feature errors $\mathbf{e}_s(t)$ will be attracted to the invariant set in the null space $Ker(\mathbf{J}^T _s(\boldsymbol{\hat{\theta}}(t))$ as $t \rightarrow \infty$.
If $m = 3k$, $\mathbf{e}_s(t)$ will converge to zeros.
If $m > 3k$ (the case for most robotic tasks of shape deformation control), the control problem is over-determined. 
Then, the non-trivial null space of $\mathbf{J}^T _s(\boldsymbol{\hat{\theta}}(t))$ will give rise to local minima configurations where the vector $\mathbf{e}_s(t)$ converges to a steady-state error.
It is well-known that the over-determined controller may lead to local minima.
More information and strategies about the over-determined deformation controller and visual servoing controller can be found in~\cite{navarro2018fourier} and~\cite{chaumette1998potential}, respectively.

\begin{remark}
Existence of the nontrivial $Ker(\mathbf{J}^T _s(\boldsymbol{\hat{\theta}}(t))$ does not mean that local minima always exist.
According to~\cite{chaumette1998potential}, potential local minima are also due to the existence of unrealizable feature motion $\dot{\mathbf{s}}^\perp(t)$ that does not belong to the range space of $\mathbf{J}_s(\boldsymbol\theta)$:
\begin{equation}
    \dot{\mathbf{s}}^\perp(t) \in [Im(\mathbf{J}_s(\boldsymbol\theta))]^\perp = Ker(\mathbf{J}_s^T(\boldsymbol\theta)).
\end{equation}
On the other hand, potential local minima configurations
\begin{equation}
    (\mathbf{s}(t) - {\mathbf{s}}^*) \in Ker(\mathbf{J}^T _s(\boldsymbol{\hat{\theta}}(t)) \cap Ker(\mathbf{J}_s^T(\boldsymbol\theta))
\end{equation}
must be physically coherent, which means that the corresponding robot manipulation poses must exist under all the object's physical constraints including the manipulation constraints and the other unknown boundary conditions.
\end{remark}

\begin{figure}[t]
	\centering
	\begin{minipage}{\linewidth}
	\centering
	\includegraphics[width=\linewidth]{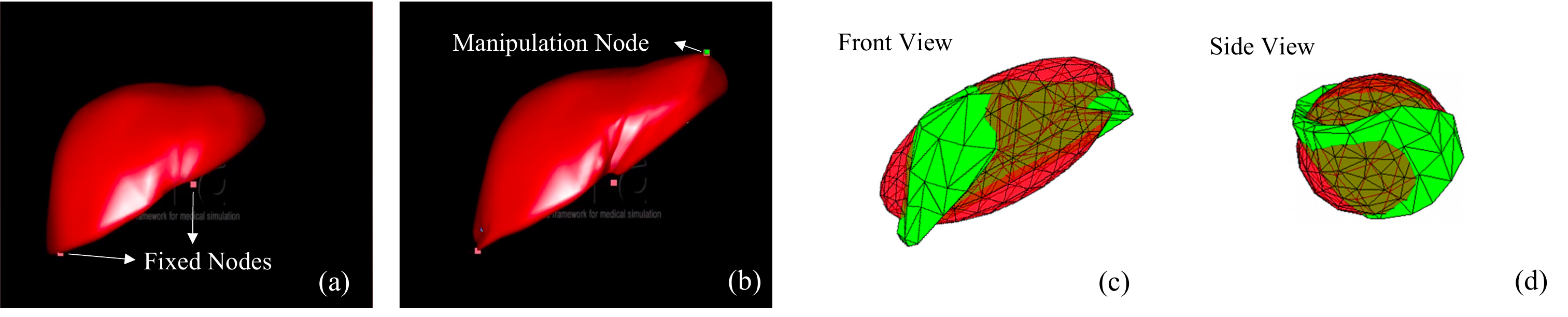}
	\vspace{-0.7cm}
    \caption{Simulation setup: (a) undeformed configuration of the liver-shaped object, where the pink nodes denote the fixed nodes, (b) manipulating the object by changing positions of the manipulation node (the green node), (c,d) comparisons of the base mesh (the red mesh) and the deformation model of the object (the green mesh).\\}
    \label{fig:simulation_setup}
	\end{minipage}
	\\
	\begin{minipage}{\linewidth}
	\centering
	\includegraphics[width=\linewidth]{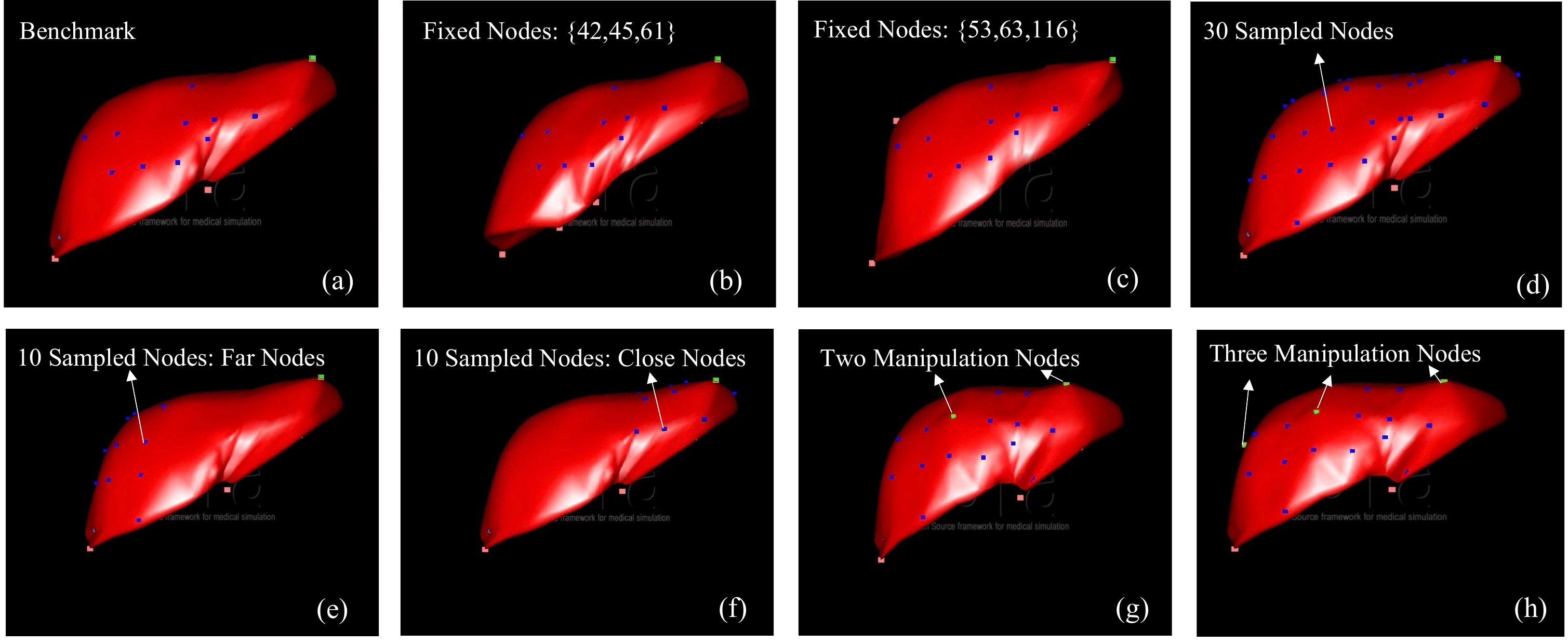}
	\vspace{-0.7cm}
    \caption{Desired configurations of different cases:
    (a) the benchmark case;
    (b) the case with fixed node indices of $\left \{ 42,45,61 \right \}$;
    (c) the case with fixed node indices of $\left \{ 53,63,116 \right \}$;
    (d) the case with $30$ sampled nodes (the blue nodes);
    (e) the far node case with the $10$ sampled nodes far from the manipulation node;
    (f) the close node case with the $10$ sampled nodes close to the manipulation node;
    (g) the case with two manipulation nodes;
    (h) the case with three manipulation nodes.}
    \label{fig:simulation_bc}
	\end{minipage}
	\vspace{-0.3cm}
\end{figure}

\section{Simulation Analysis}
\subsection{Simulation Setup}  
Simulation validations are conducted based on the Simulation Open Framework Architecture (SOFA)~\cite{allard2007sofa} platform within a liver-shaped object scene~\cite{sofascn.org}.
In the SOFA scene, the object has a visual model (the red liver surface in Fig.~\ref{fig:simulation_setup}(a,b)) and a deformation model (the green mesh in Fig.~\ref{fig:simulation_setup}(c,d)).
We can modify the physical properties of the deformation model and directly read or change the positions of its nodes with specific indices.
The object can be fixed in the simulation space by adding zero-displacement constraints to some fixed nodes (the pink nodes in Fig.~\ref{fig:simulation_setup} and Fig.~\ref{fig:simulation_bc}) on the deformation model.
We also set manipulation nodes (the green nodes in Fig.~\ref{fig:simulation_setup} and Fig.~\ref{fig:simulation_bc}) and sampled nodes (the blue nodes in Fig.~\ref{fig:simulation_bc}) on the deformation model to deform the object and to obtain surface samplings, respectively.
The base mesh (the red mesh in Fig.~\ref{fig:simulation_setup}(c,d)) is generated using the moment-based method proposed in~\cite{pentland1991closed} with the same unit system in the SOFA scene.
Fig.~\ref{fig:simulation_setup}(c,d) show the comparisons between the deformation model of the object and the base mesh used for control.
The material properties of the base mesh are set to be: $E=50$, $v=0.45$, and $M=20$.
The control gains are set to be:
$\mathbf{K}_s=80\mathbf{I}_{m \times m}$ and $\Gamma=500$.
The frequency of the simulation is 50Hz.

In our simulations, we first set a benchmark case (Fig.~\ref{fig:simulation_bc}(a)) with the following conditions:
the object's material and mass parameters are: $E=100,v=0.49,M=100$; 
the fixed node indices are: $\left \{ 3,39,64 \right \}$; 
desired deformations are generated by manipulating one node (of the index $85$) with the displacement vector of $\mathbf{u}(\mathbf{p}_r)^* = (1,1,0.8)^T$ (unit: voxel of the simulation scene); 
the sampled nodes are selected on the front side of the object; 
the dimension of the deformation features is $m=30$.
Then, for each group of simulations, we separately changed the object's modeling parameters, the boundary conditions, the sampled node distributions, the manipulation node numbers, and the deformation feature dimensions for controlled validations.
For different simulations, we developed deformation controllers using the same base mesh with fixed physical properties and pose parameters.

To analyze the convergence of our deformation features, we recorded the following feature error norm:
\begin{equation}
\left \| \mathbf{e}_{s}(t) \right \|
 = \sqrt{\mathbf{e}_s^T(t) \mathbf{e}_s(t)}
\end{equation}
Also, to analyze the convergence of the object's 3D shape, we recorded the following total mesh error sum:
\begin{equation}
    e_{x}(t) = (\mathbf{x}(\mathbf{n}_o,t) - \mathbf{x}^*(\mathbf{n}_o))^{T}(\mathbf{x}(\mathbf{n}_o,t) - \mathbf{x}^*(\mathbf{n}_o))
\end{equation}
where $\mathbf{x}(\mathbf{n}_o,t)$ is the position vector of the object's deformation model nodes $\mathbf{n}_o$ and $\mathbf{x}^*(\mathbf{n}_o)$ is the desired position vector.

\subsection{Simulation Cases}
\subsubsection{Simulations with different material and mass parameters}
We conducted a group of simulation cases with a large range of material and mass parameter changes of the object:
$E=100,v=0.49,M=100$ (the benchmark case in Fig.~\ref{fig:simulation_bc}(a)); $E=1000,v=0.48,M=3000$; $E=5000,v=0.47,M=100$;
$E=50000,v=0.4,M=3000$.
Simulation results in Fig.~\ref{fig:simulation_r_p_bc}(a) show the minimization of the magnitude errors $\left \| \mathbf{e}_{s}(t) \right \|$ and $e_{x}(t)$.
The start points of the curves for $\left \| \mathbf{e}_{s}(t) \right \|$ and $e_{x}(t)$
are different because: the same manipulation on objects with different material and mass parameters produced different deformation feature values.

\subsubsection{Simulations with different boundary conditions}
We conducted a group of simulation cases with different fixed node indices: $\left \{ 3,39,64 \right \}$ (the benchmark case in Fig.~\ref{fig:simulation_bc}(a)); $\left \{ 42,45,61 \right \}$ (Fig.~\ref{fig:simulation_bc}(b)); $\left \{ 53,63,116 \right \}$ (Fig.~\ref{fig:simulation_bc}(c)).
Simulation results in Fig.~\ref{fig:simulation_r_p_bc}(b) show the minimization of the magnitude errors $\left \| \mathbf{e}_{s}(t) \right \|$ and $e_{x}(t)$.
The differences of start points of the mesh error sum curves are much larger than the feature error norm curves.
This is because the local deformations of the sampled nodes were less different than the object's global deformation. 
It further implies that our method can control the 3D shape globally even with locally sampled nodes. 

\begin{figure}[t]
	\centering
	\begin{minipage}{\linewidth}
	\centering
	\includegraphics[width=\linewidth]{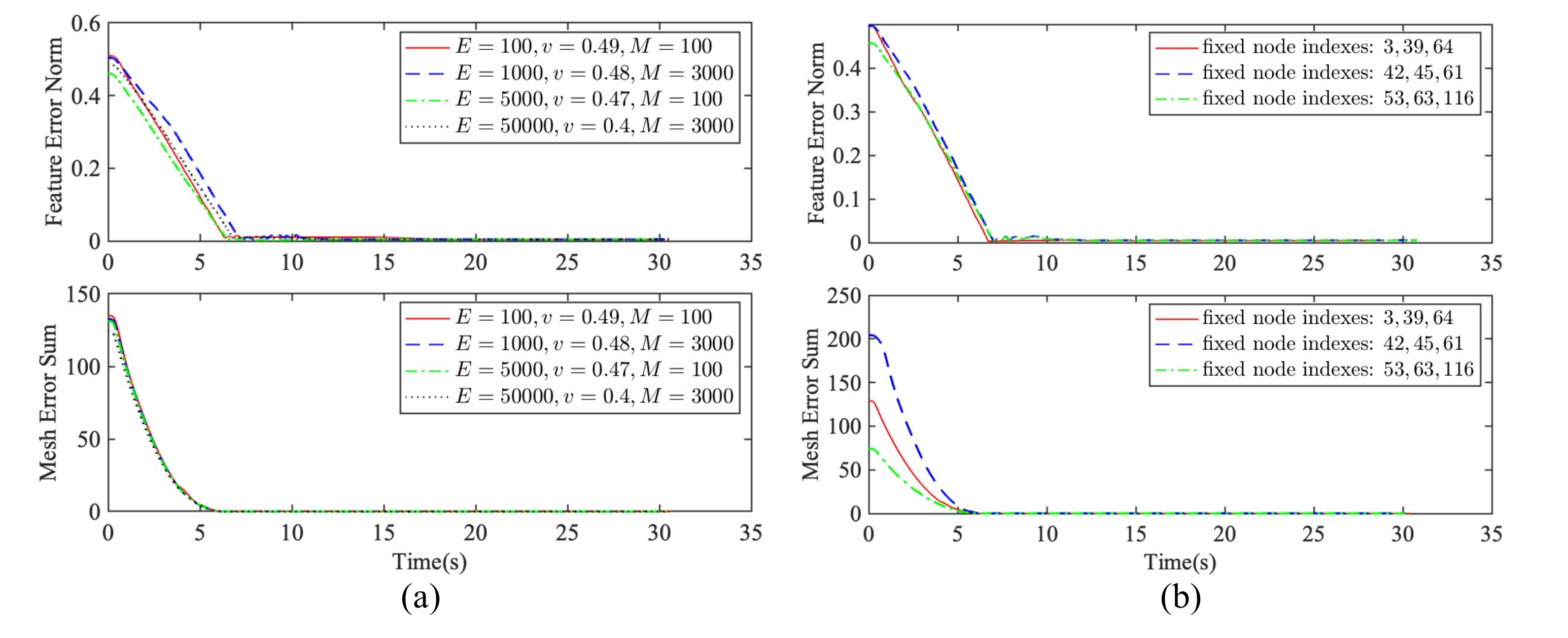}
     \vspace{-0.9cm}
    \caption{Simulation results with different: (a) material and mass parameters; (b) boundary conditions. \\}
    \label{fig:simulation_r_p_bc}
	\end{minipage}
		\begin{minipage}{\linewidth}
	\centering
    	\includegraphics[width=\linewidth]{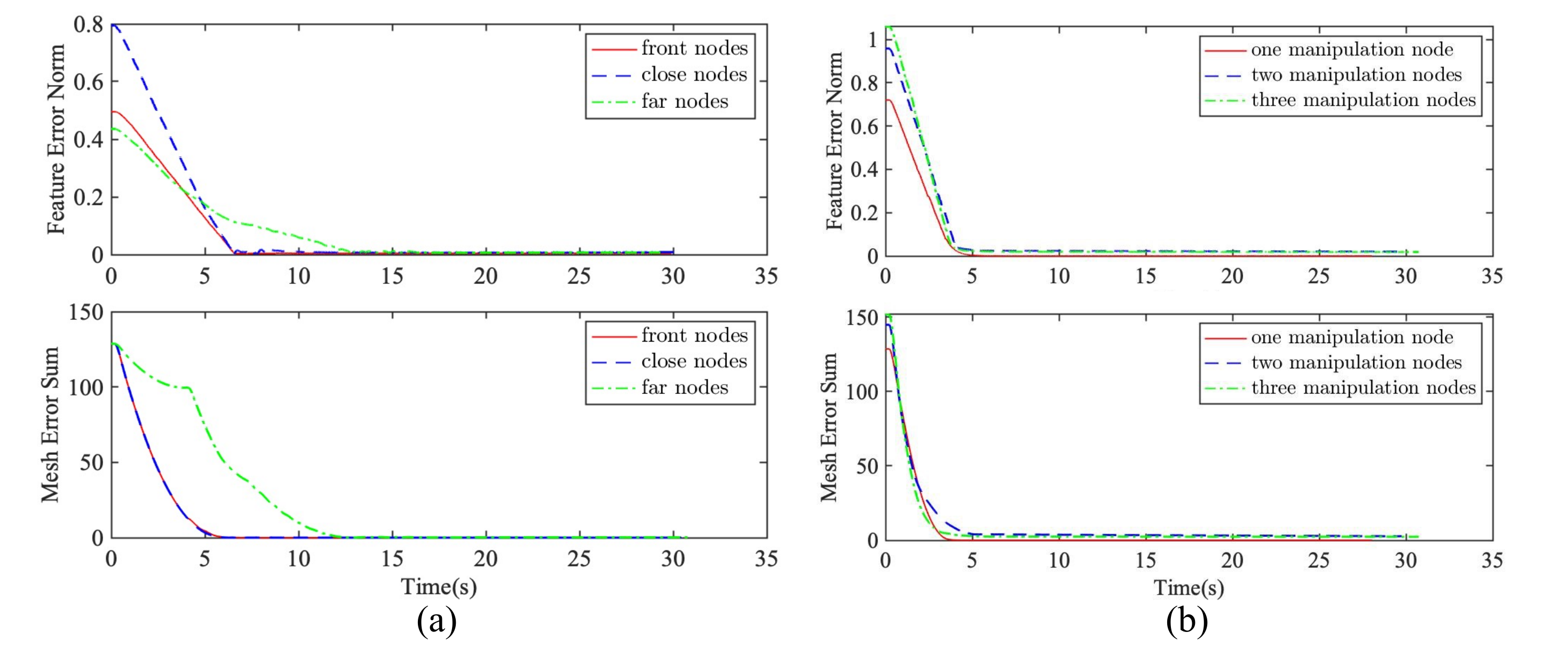}
	\vspace{-0.9cm}
    \caption{Simulation results with: (a) different sampled node distributions; (b) multiple manipulation nodes.\\ }
    \label{fig:simulation_mm_n}
	\end{minipage}
	\begin{minipage}{\linewidth}
	\centering
    \includegraphics[width=\linewidth]{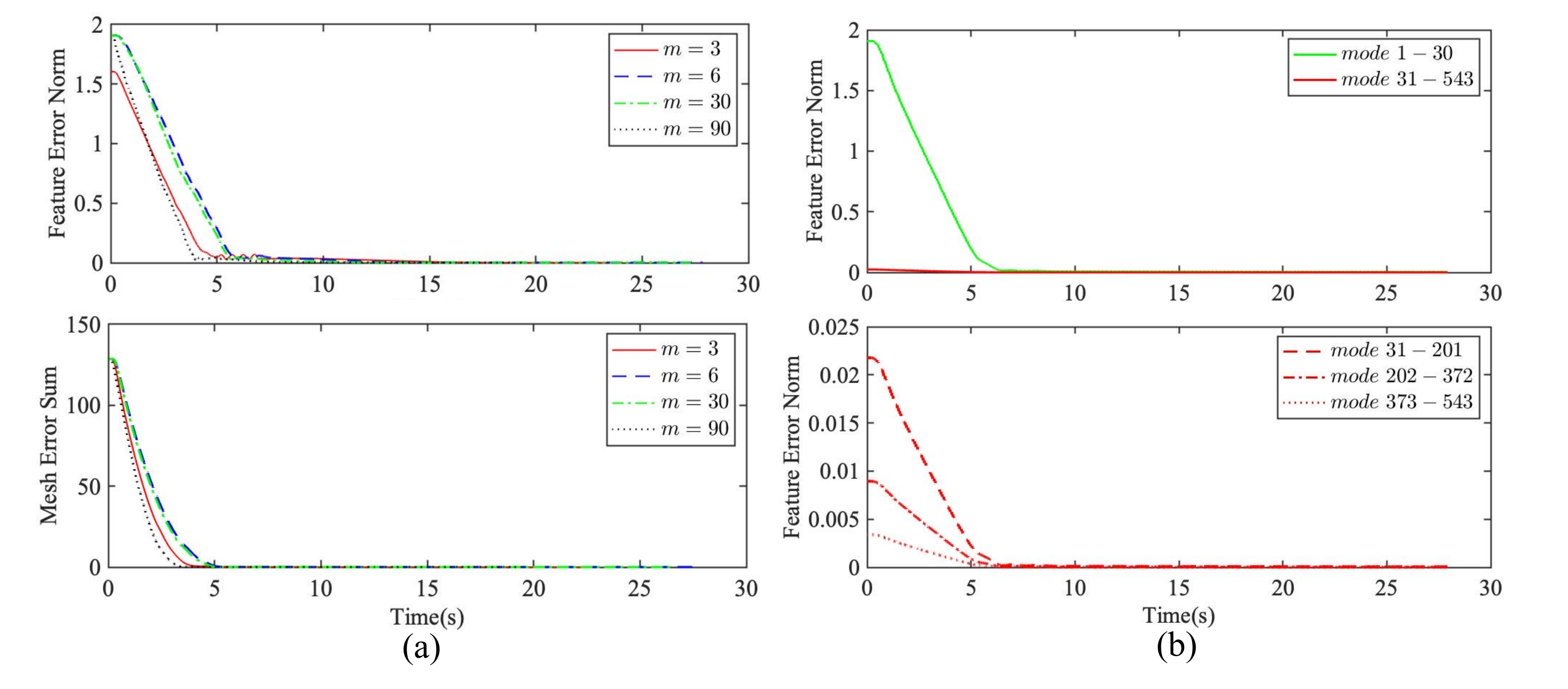}
    \vspace{-0.9cm}
    \caption{Simulation results for: (a) different deformation feature dimensions; (b) the controlled and the non-directly actuated modes.} 
    \label{fig:simulation_r_modes}
	\end{minipage}
	\vspace{-0.2cm}
\end{figure}

\subsubsection{Simulations with different sampled node distributions}
To validate our method with different distributions of the sampled nodes, we conducted three cases: the front node case (the benchmark case in Fig.~\ref{fig:simulation_bc}(a)),
the close node case (Fig.~\ref{fig:simulation_bc}(f)),
and the far node case (Fig.~\ref{fig:simulation_bc}(e)).
The results in Fig.~\ref{fig:simulation_mm_n}(a) show the minimization of the magnitude errors $\left \| \mathbf{e}_{s}(t) \right \|$ and $e_{x}(t)$.
The start points of the feature error norm curves are different because:
despite the same manipulation, different sampled nodes underwent different (local) deformations producing different deformation feature values.

In the far node case, the sampled nodes were distributed in a local region far from the manipulation nodes but close to the constrained nodes.
Their desired deformation was more complex and different than the desired manipulation, which affected the initial manipulation directions.
However, our controller kept adjusting the manipulation command to minimize $\mathbf{e}_s(t)$.
The turning points on the result curves happened when some of the dimensions of the manipulation velocity command changed between positive and negative.
We further conclude that the sampled node distributions that are more global or close to the manipulation node produce better initial manipulation directions.

\subsubsection{Simulations with multiple manipulation nodes}
To validate our method with multiple manipulation nodes, we conducted three simulation cases:
the case with one manipulation node (the benchmark case in Fig.~\ref{fig:simulation_bc}(a));
the case with two manipulation nodes (Fig.~\ref{fig:simulation_bc}(g));
the case with three manipulation nodes (Fig.~\ref{fig:simulation_bc}(h));
Results in Fig.~\ref{fig:simulation_mm_n}(b) show the minimization of the magnitude errors $\left \| \mathbf{e}_{s}(t) \right \|$ and $e_{x}(t)$.

\subsubsection{Simulations with different deformation feature dimensions}
We conducted a group of simulation cases with different deformation feature dimensions: $m=3$, $m=6$, $m=30$, and $m=90$.
As our method requires that $3n \ge m$, we increased the number of sampled nodes to 30 (Fig.~\ref{fig:simulation_bc}(d)) in these cases.
Results in Fig.~\ref{fig:simulation_r_modes}(a) show the minimization of the magnitude errors $\left \| \mathbf{e}_{s}(t) \right \|$ and $e_{x}(t)$ with different $m$.

\subsubsection{Behaviors of Non-Directly Actuated Modes}
We also analyze the behavior of the high-frequency modes that were not directly actuated in the simulation case with $m=30$.
As we can measure $181$ nodes from the object's deformation model, the amplitudes of all the $543$ modes were computed and recorded during the control process.
We compare the behaviors of the modes used in our controller (the $1$st-$30$th modes) and the non-directly actuated modes (the $31$th - $543$th modes) in Fig.~\ref{fig:simulation_r_modes}(b).
The results show that:
1) amplitudes of the non-directly actuated modes are much smaller than the controlled modes;
2) the non-directly actuated modes 
are similarly minimized as the controlled modes.
The reasons are:
modal amplitudes are inversely proportional to the square of their natural frequency~\cite{pentland1991closed};
even though not directly actuated, the high-frequency modes converge towards the desired values as the total mesh error sum $e_{x}(t)$ being minimized.

\begin{figure}[t]
    \centering
    \includegraphics[width=\linewidth]{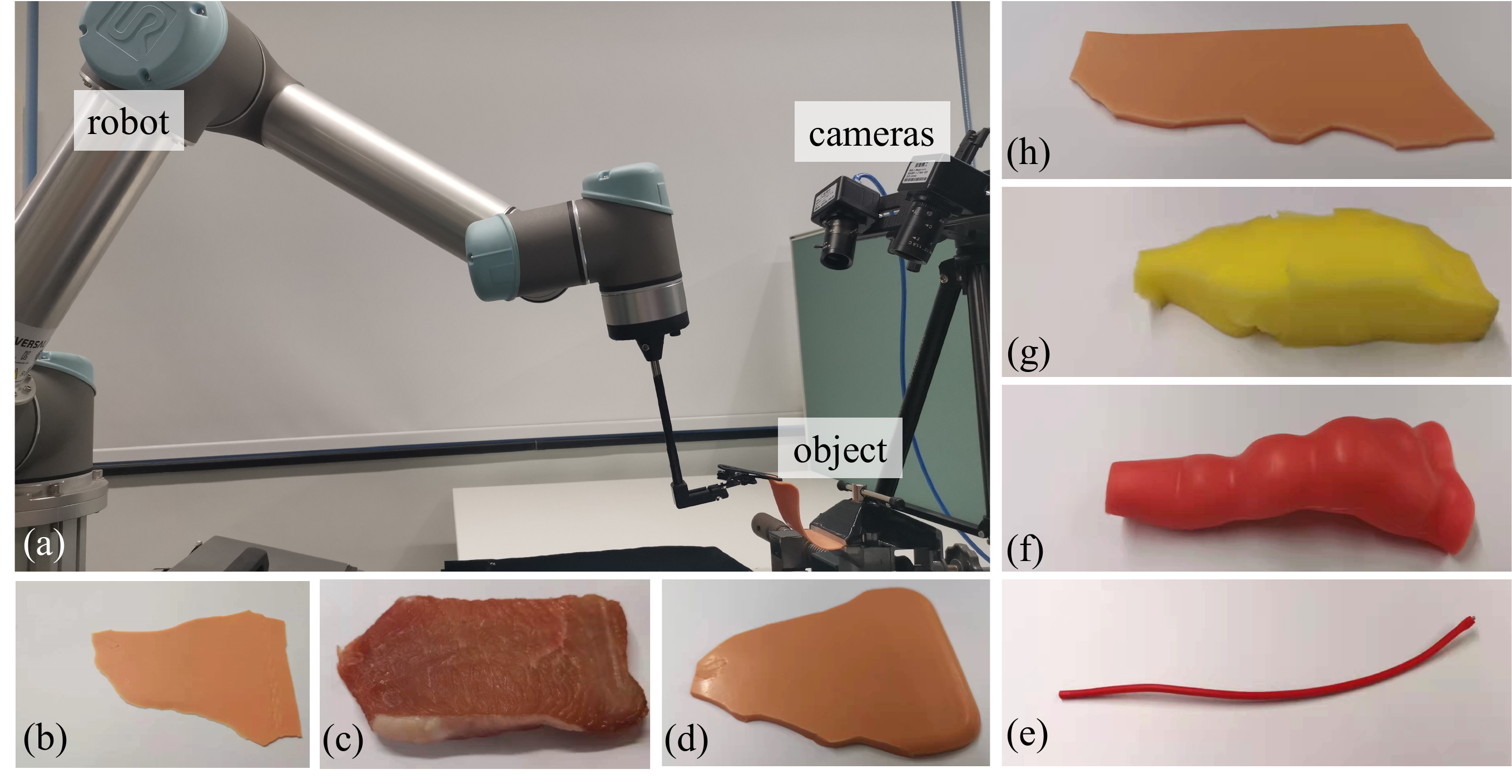}
    \caption{Experiment setup and the experimental objects.
    (a): the robot-camera platform; 
    (b): the deformable planar object (a piece of silicon skin); 
    (c): the pork tissue block;
    (d): the short silicon block; 
    (e): the deformable linear object (a piece of red electric wire);
    (f): the silicon colon model;
    (g): the sponge block;
    (h): the long silicon block.
    }
    \label{fig:experiment_setup}
    \vspace{-0.2cm}
\end{figure}

\begin{algorithm}[b]
	\caption{Controller($\mathbf{x}^*(\mathbf{p}_s)$, $a_z$, $m$, $E$, $v$, $M$, $\mathbf{K}_s$, $\Gamma$)}  
	\begin{algorithmic}[1]
	  \While {receiving vision and robot measurements} 
	  \State measure $\mathbf{x}(\mathbf{p}_s,t)$, $\mathbf{x}(\mathbf{p}_r,t)$,
	    ${}^{c}\mathbf{R}_e(t)$
	  \If{ $t = t_0$ }
	    \State generate the base mesh $\gets$ equation~(\ref{eq:base_mesh_size},\ref{eq:base_mesh_R},\ref{eq:base_mesh_t}), $a_z$
	    \State compute $\mathbf{\Phi}_n$ and $\widetilde{\mathbf{K}}$ $\gets$ $m$, $E$, $v$, $M$
	   \State establish the object-mesh mappings of $\mathbf{p}_s$ and $\mathbf{p}_r$
	    \State compute $\mathbf{s}^*$ $\gets$  equation \eqref{eq:point_displacement} and \eqref{eq:feature_computation}, $\mathbf{x}^*(\mathbf{p}_s)$
	    \State initialize $\boldsymbol{\theta}(t_0)$ with ones
	  \EndIf
	  \If{ $t > t_0$ }
	    \State compute $\mathbf{s}(t)$ $\gets$ equation \eqref{eq:point_displacement} and \eqref{eq:feature_computation}, $\mathbf{x}(\mathbf{p}_s,t)$
	    \While {robot is running}
	      \State compute $\mathbf{e}_s(t)$ $\gets$ equation \eqref{eq:e_s}
	      \State update $\boldsymbol{\hat\theta}(t)$ $\gets$ equation \eqref{eq:theta_dot}, $\Gamma$
	      \State compute $\mathbf{J}_s(\boldsymbol{\hat{\theta}}(t))$ $\gets$ equation $\eqref{eq:J}$
	      \State compute $\mathbf{v}(\mathbf{p}_r,t)$ $\gets$ equation $\eqref{eq:v_cmd}$, $\mathbf{K}_s$
	      \State send $\mathbf{v}(\mathbf{p}_r,t)$ to robot
	    \EndWhile
	  \EndIf
	  \EndWhile
	\end{algorithmic} 
\end{algorithm}

\section{Experimental Validation}
\subsection{Experiment Setup}
Our experiment platform consists of a Universal Robot 5 and two cameras.
Fig.~\ref{fig:experiment_setup} shows the experiment setup, where one end of the object is fixed in the space while the other end is rigidly grasped by a slim clip mounted on the robot end-effector. 
Fig.~\ref{fig:experiment_setup}(b-h) show the deformable objects used in experiments.
For different objects, their base meshes used for control are set with the same material and mass parameters ($E=1 GPa$, $v=0.4$, $M=0.05 kg$).
The control gains are set to be: $\mathbf{K}_s=0.1\mathbf{I}_{m \times m}$ and
$\Gamma=0.1$.
The frequency of the experiments is 30Hz.

In advance of the control phase, we calibrated the stereo cameras, the hand-eye relationship (the transformation matrix ${}^{r}\mathbf{T}_{c}$) and the hand-grasper relationship (the position vector ${}^{e}\mathbf{x}(\mathbf{p}_r)$ of the grasped point under the end-effector frame $\left \{ e \right \}$).
Desired deformations in our experiments were generated by manually manipulating the object to target configurations.
During the control phase,
we tracked some image features of the object to obtain surface samplings.
For base mesh generation, at $t_0$, and under our eye-grasper configuration, we selected four points among the surface samplings:
$\mathbf{p}_0$/$\mathbf{p}_1$ with the smallest/largest $x$ coordinate and the smallest $y$ coordinate;
$\mathbf{p}_2$/$\mathbf{p}_3$ with the smallest/largest $x$ coordinate and the largest $y$ coordinate.
We computed the base mesh's size parameters $\left \{ a_x, a_y \right \}$ with:
\begin{equation}
    \begin{aligned}
  & a_x = \frac{\left \| \mathbf{x}(\mathbf{p}_1, t_0) - \mathbf{x}(\mathbf{p}_0, t_0) \right \| + \left \| \mathbf{x}(\mathbf{p}_3, t_0) - \mathbf{x}(\mathbf{p}_2, t_0) \right \|}{4}, 
  \\
  & a_y = \frac{1}{2}\left \|  \frac{\mathbf{x}(\mathbf{p}_0, t_0)+\mathbf{x}(\mathbf{p}_1, t_0)}{2} - \frac{\mathbf{x}(\mathbf{p}_2, t_0)+\mathbf{x}(\mathbf{p}_3, t_0)}{2}
  \right \| .
  \end{aligned}
  \label{eq:base_mesh_size}
\end{equation}
We roughly estimated the object's thickness to set the size parameter $a_z$.
Then, to determine the base mesh pose ${}^c\mathbf{T}_b$,
we computed $\mathbf{c}(\mathbf{x}(\mathbf{p}_s,t_0)) \in \mathbb{R}^3$ (the mass center of the surface samplings) and measured ${}^{c}\mathbf{R}_e(t_0)$ (the orientation of the end-effector) to set:
\begin{equation}
    {}^c\mathbf{T}_b = 
    \begin{bmatrix}
    {}^{c}\mathbf{R}_e(t_0)  &  {}^c\mathbf{t}_b(t_0)\\
      \mathbf{0} & 1
\end{bmatrix},
\label{eq:base_mesh_R}
\end{equation}
\begin{equation}
    {}^c\mathbf{t}_b(t_0) = \mathbf{c}(\mathbf{x}(\mathbf{p}_s,t_0)) - a_z \begin{bmatrix}
  0 & 0 & 1
\end{bmatrix}^T.
\label{eq:base_mesh_t}
\end{equation}
When the robot was running, we sent the deformation feature errors to the controller, which computed linear velocity commands for the robot manipulation.
After transforming the velocity commands to the joint space of the robot,
we directly set the joint velocities to each of the robot joints via the low-level interface provided by URScript~\cite{URScript}.
Algorithm 1 shows the implementation of our deformation controller.

To analyze the convergence of our deformation features, we recorded the feature error norm $\left \| \mathbf{e}_{s}(t) \right \|$.
In addition, to analyze the convergence of the object's 3D shape,
we recorded the following manipulation errors:
\begin{equation}
    \mathbf{e}_d(t) = \mathbf{x}(\mathbf{p}_r,t)-\mathbf{x}^*(\mathbf{p}_r)
\end{equation}
where $\mathbf{x}(\mathbf{p}_r,t)$ is the measured position vector of the manipulation point $\mathbf{p}_r$; $\mathbf{x}^*(\mathbf{p}_r)$ is the target position vector of $\mathbf{p}_r$.
It should be mentioned that we only recorded $\mathbf{x}^*(\mathbf{p}_r)$ for result analysis.
We did not use it for robot control.
We also gave 2D image comparisons between the object's resulting and target shape by adding the desired image as a translucent mask on the resulting image.

\subsection{Validation Cases}

\subsubsection{Validations with Different Kinds of Objects}
\begin{figure}[t]
    \centering
    \includegraphics[width=\linewidth]{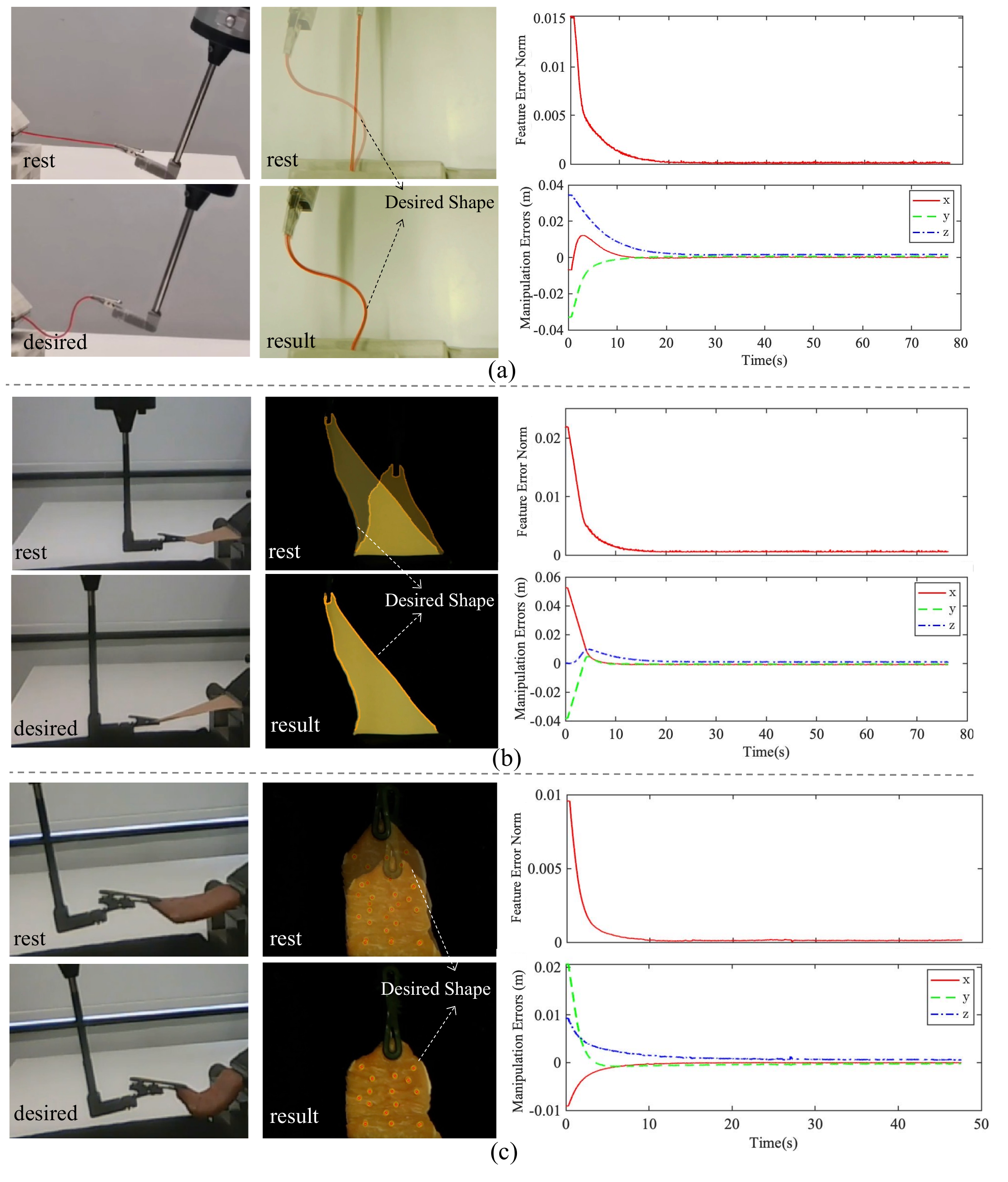}
    \vspace{-0.7cm}
    \caption{Cases with different kinds of objects:
    the first/second column shows the third/left camera views;
    the third column shows the result curves of the feature error norm ($\left \| \mathbf{e}_{s}(t) \right \|$) and the manipulation errors ($\mathbf{e}_d(t)$);
    (a): for the linear object;
    (b): for the planar object;
    (c): for the volumetric tissue.
    }
    \vspace{-0.2cm}
    \label{fig:object_kinds}
\end{figure}
To validate that our method can deal with different types of objects and the tracked image features, we conducted experimental cases using linear, planar, and volumetric objects under stereo measurements of points, curves, and contours.
According to~\cite{pentland1991recovery}~\cite{pentland1991closed}, using 30 lowest-frequency deformation modes, a wide range of nonrigid shapes can be described.
Aiming at the global 3D deformation control, we set the deformation feature dimensions to be: $m=30 \sim 50$.

The linear object case (Fig.~\ref{fig:object_kinds}(a)) used a piece of red wire.
The desired deformation was generated by pushing the wire forward while lifting it.
We obtained surface samplings from the wire's curve tracked by color segmentation.
We set the deformation feature dimension to be: $m=30$, and estimated the wire's thickness to be: ${a}_z = 0.001 (m)$.
The planar object case (Fig.~\ref{fig:object_kinds}(b)) used a piece of silicon skin.
The desired deformation was generated by pulling the silicon skin to left.
We obtained surface samplings from the silicon skin's contour tracked by color segmentation.
We selected the deformation feature dimension to be: $m=40$, and estimated the silicon skin's thickness to be: ${a}_z = 0.0005 (m)$.
The volumetric object case (Fig.~\ref{fig:object_kinds}(c)) used a pork tissue block.
The desired deformation was generated by lifting the tissue to right.
We obtained surface samplings by tracking points on the tissue surface.
We selected the deformation feature dimension to be: $m=30$, and estimated the tissue's thickness to be: ${a}_z = 0.0125 (m)$.

Result curves in Fig.~\ref{fig:object_kinds} show the minimization of both $\left \| \mathbf{e}_{s}(t) \right \|$ and $\mathbf{e}_d(t)$ for different cases.
Manipulation error curves present few steady-state errors (less than 3mm) because modal-based dimension reductions neglect some details of the object's shape deformation.
In addition, from the resulting images of the left camera views, we can see that the objects are manipulated to fit the desired shapes in a general sense.

\subsubsection{Validations with Base Mesh Size Errors}
To validate our method with base mesh size errors, we conducted experimental cases using the same object (the short silicon block in Fig.~\ref{fig:experiment_setup}(d)) with different-sized base meshes.
We set the same desired deformation from the rest configuration in Fig.~\ref{fig:base_mesh_size}(c) to the desired configuration in Fig.~\ref{fig:base_mesh_size}(d).
We obtained surface samplings (the red dots in Fig.~\ref{fig:base_mesh_size}(a,b)) from the object's contour tracked by color segmentation.
We selected four base meshes (Fig.~\ref{fig:base_mesh_size}(e)) with different size parameters (unit:m):
the red mesh with $\left \{ a_x, a_y,a_z\right \} = \left \{ 0.015, 0.05,0.0075\right \}$;
the green mesh with 
$\left \{ a_x, a_y,a_z\right \} = \left \{ 0.0175, 0.025,0.001\right \}$;
the black mesh with 
$\left \{ a_x, a_y,a_z\right \} = \left \{ 0.03, 0.045,0.0025\right \}$;
the blue mesh with
$\left \{ a_x, a_y,a_z\right \} = \left \{ 0.001, 0.055,0.001\right \}$.
We set the same base mesh pose according to equations~(\ref{eq:base_mesh_R},~\ref{eq:base_mesh_t}).
The deformation feature dimension was set to be: $m=15$.

Results in Fig.~\ref{fig:base_mesh_size}(f) show the minimization of $\left \| \mathbf{e}_{s}(t) \right \|$ for all cases.
Also, as we can see from Fig.~\ref{fig:base_mesh_size}(g), different base mesh sizes do influence the deformation processes (illustrated by the manipulation trajectories), but the target manipulation position can always be reached.
The more similar the size of the base mesh is to the object, the better the manipulation trajectory.
On the other hand, the initial directions of these manipulation trajectories are similar.
This is because: even projected to different base meshes, the primary directions of the manipulation point to fit the same desired deformation were similar.

\begin{figure}[b]
    \centering
    \vspace{-0.2cm}
    \includegraphics[width=\linewidth]{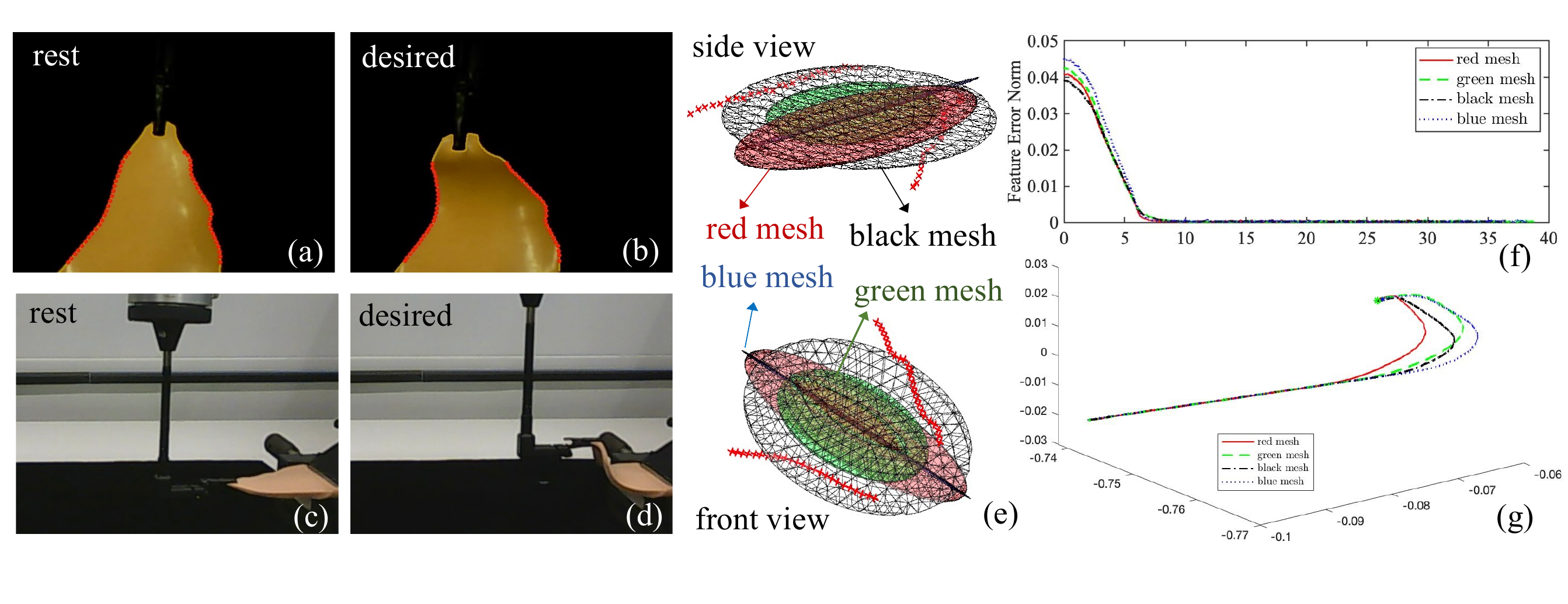}
    \vspace{-0.7cm}
    \caption{Cases with different base mesh size errors. (a)/(b): the left camera views of the rest/desired configuration, where the red dots denote the surface samplings; 
    (c)/(d): the third camera views of the rest/desired configuration;
    (e): the illustrations of the different-sized meshes with respect to the surface samplings (the red cross symbol) under the rest configuration;
    (f): the result curves of $\left \| \mathbf{e}_{s}(t) \right \|$;
    (e): the manipulation trajectories (of $\mathbf{x}(\mathbf{p}_r,t)$), where the green star symbol denotes the target manipulation position ($\mathbf{x}^*(\mathbf{p}_r)$).}
    \label{fig:base_mesh_size}
\end{figure}

\subsubsection{Validations with Different Rest Configurations}
\begin{figure}[b]
    \vspace{-0.2cm}
    \centering
    \includegraphics[width=\linewidth]{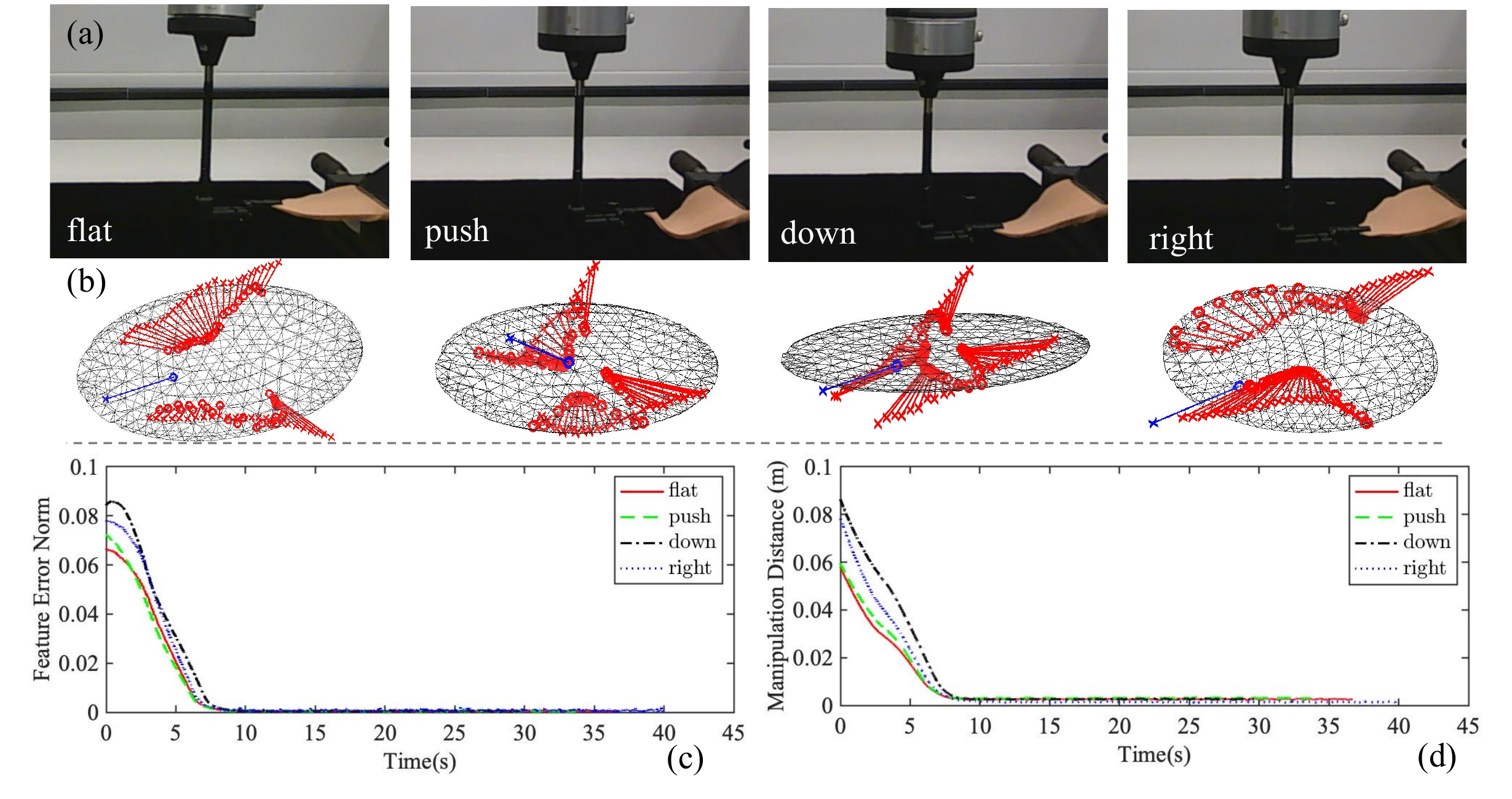}
    \vspace{-0.7cm}
    \caption{Cases with different rest configurations.
    (a): the third camera views;
    (b): surface samplings/manipulation points (the red/blue cross symbol) connecting (through red/blue lines) with their base mesh projection (the red/blue circles);
    (c): the result curves of $\left \| \mathbf{e}_{s}(t) \right \|$;
    (d): the result curves of the manipulation distance ($\left \| \mathbf{e}_{d}(t) \right \|$).}
    \label{fig:base_mesh_init}
\end{figure}

\begin{figure}[t]
    \centering
    \includegraphics[width=\linewidth]{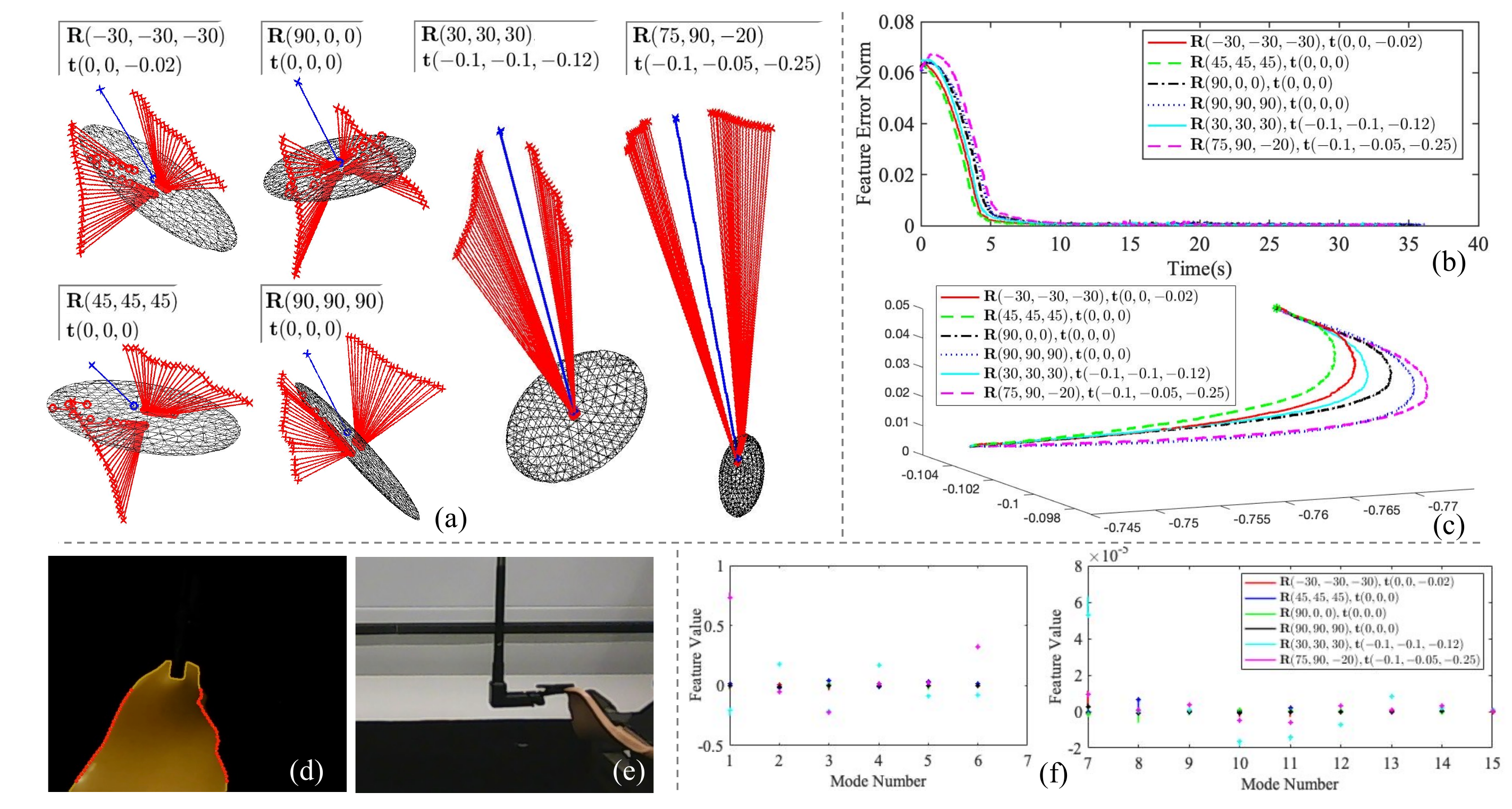}
    \vspace{-0.7cm}
    \caption{Cases with base mesh pose errors.
    (a) surface samplings and manipulation points connecting with their base mesh projections under the rest configuration;
    (b): the result curves of $\left \| \mathbf{e}_{s}(t) \right \|$;
    (c): the manipulation trajectories (of $\mathbf{x}(\mathbf{p}_r,t)$), where the green star symbol denotes the target manipulation position ($\mathbf{x}^*(\mathbf{p}_r)$);
    (d/e): the left/third camera views of the desired configuration;
    (f): deformation feature values of 1st-6th/7th-15th mode. The star symbols denote the desired feature values, which are connected by solid lines from the corresponding rest feature values.}
    \label{fig:base_mesh_large_pose}
    \vspace{-0.2cm}
\end{figure}

As we compute the base mesh pose using rest measurements of surface samplings, we further validate our method with different rest configurations. We selected the same object (the short silicon block in Fig.~\ref{fig:experiment_setup}(d)) but set four different rest configurations  (Fig.~\ref{fig:base_mesh_init}(a)):
the flat case (with the object being in a flat shape);
the push case (with the object being pushed forward);
the down case (with the object being put down);
the right case (with the object being pulled to right).
These different rest configurations led to different base mesh poses (Fig.~\ref{fig:base_mesh_init}(b)). 
Desired deformations in these cases were generated from their rest configurations to the same desired configuration (Fig.~\ref{fig:base_mesh_large_pose}(e)).
For different cases, we selected the same base mesh (the black mesh in Fig.~\ref{fig:base_mesh_size}).
The deformation feature dimension was set to be: $m=15$.
Results in Fig.~\ref{fig:base_mesh_init}(c,d) show the minimization of both $\left \| \mathbf{e}_{s}(t) \right \|$ and $\left \| \mathbf{e}_{d}(t) \right \|$.

\subsubsection{Validations with Base Mesh Pose Errors}
To validate our method with large pose errors between the base mesh and the object, we added an extra transformation to the roughly estimated base mesh pose:
\begin{equation}
    {}^c\mathbf{T}_b = 
    \begin{bmatrix}
    {}^{c}\mathbf{R}_e(t_0)  &  {}^c\mathbf{t}_b(t_0)\\
      \mathbf{0} & 1
\end{bmatrix} \cdot \mathbf{T}_{add}
\end{equation}
The extra transformation was formulated as:
\begin{equation}
    \mathbf{T}_{add} = 
    \begin{bmatrix}
    \mathbf{R}(\alpha_x,\alpha_y,\alpha_z) &  \mathbf{t}(\xi_x,\xi_y,\xi_z)  \\
  \mathbf{0} &  1
\end{bmatrix} 
\end{equation}
where $\mathbf{R}(\alpha_x,\alpha_y,\alpha_z) \in \mathbb{R}^{3 \times 3}$ is a rotation matrix with $\left \{ \alpha_i,i=x,y,z \right \}$ denoting the rotation angle (unit: degree) about the $i$-th axis; 
$\mathbf{t}(\xi_x,\xi_y,\xi_z) \in \mathbb{R}^{3}$ is a translation vector with $\left \{ \xi_i,i=x,y,z \right \}$ denoting the displacement component (unit: m) along the $i$-th axis.
We set the same object (the short silicon block in Fig.~\ref{fig:experiment_setup}(d)) and the same base mesh (the black mesh in Fig.~\ref{fig:base_mesh_size}).
As shown in Fig.~\ref{fig:base_mesh_large_pose}(a), we set six cases: 
the case p1 with 
$\mathbf{R}(-30,-30,-30)$ and $\mathbf{t}(0,0,-0.02)$;
the case p2 with $\mathbf{R}(45,45,45)$ and $\mathbf{t}(0,0,0)$; 
the case p3 with $\mathbf{R}(90,0,0)$ and $\mathbf{t}(0,0,0)$; 
the case p4 with $\mathbf{R}(90,90,90)$ and $\mathbf{t}(0,0,0)$; 
the case p5 with $\mathbf{R}(30,30,30)$ and $\mathbf{t}(-0.1,-0.1,-0.12)$;
the case p6 with $\mathbf{R}(75,90,-20)$ and $\mathbf{t}(-0.1,-0.05,-0.25)$.
For different cases, we generated the desired deformation from the flat case in Fig.~\ref{fig:base_mesh_init}(a) to the desired configuration in Fig.~\ref{fig:base_mesh_large_pose}(e).
The deformation feature dimension was set to be: $m=15$.

Results in Fig.~\ref{fig:base_mesh_large_pose}(b) show the minimization of $\left \| \mathbf{e}_{s}(t) \right \|$ for all cases.
Also, as we can see from Fig.~\ref{fig:base_mesh_large_pose}(c), different base mesh poses do influence the deformation processes (illustrated by the manipulation trajectories), but the target manipulation position can always be reached.
In Fig.~\ref{fig:base_mesh_large_pose}(b), the feature error norm curve of the case p6 presents an increasing trend at the beginning stage, which implies unsatisfactory initial manipulation commands.
The reason is: the pose errors were so large that the base mesh projection of the manipulation point was among the projections of the surface samplings.
Nevertheless, our controller still can achieve the minimization of the deformation feature errors.
It should be noted that the initial values of the deformation feature errors (the curves' start points in Fig.~\ref{fig:base_mesh_large_pose}(b)) of different cases are similar even with the large differences of their base mesh poses. This is because: the object-mesh pose errors,
which can be regarded as rigid motions absorbed in the first six modes,
mainly influence the (rest and desired) deformation feature values of the first six modes rather than the deformation feature errors.
It can be seen from Fig.~\ref{fig:base_mesh_large_pose}(f) that:
feature values of the first six modes are much larger than the other modes;
feature error values (the length of solid lines) are much smaller than the desired feature values (the star symbols' positions) in the first six modes.

\subsubsection{Validations with Complex-Shaped Objects of Similar Sizes}
\begin{figure}[t]
    \centering
    \includegraphics[width=\linewidth]{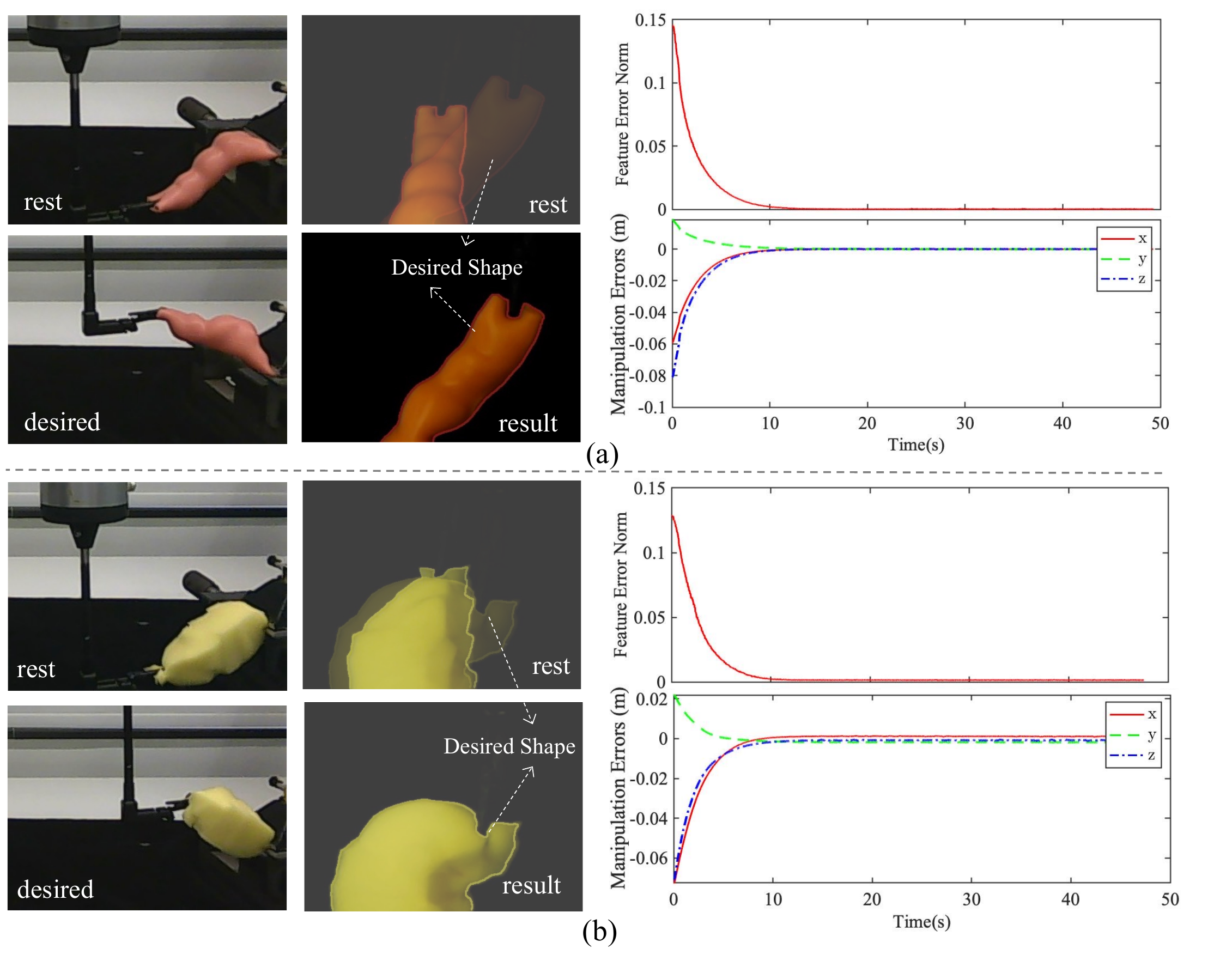}
    \vspace{-0.7cm}
    \caption{Cases with complex-shaped objects of similar sizes:
    the first/second column shows the third/left camera views;
    the third column shows the result curves of $\left \| \mathbf{e}_{s}(t) \right \|$ and $\mathbf{e}_{d}(t)$;
    (a): for the sponge block;
    (b): for the silicon colon model.
    }
    \label{fig:object_similar}
    \vspace{-0.2cm}
\end{figure}
As discussed in Section VI, different-shaped objects with similar sizes have similar low-frequency modes, which leads to different behaviors in Euclidean space but similar behaviors in the truncated modal space.
To validate our method for objects with different complex shapes but similar sizes, we conducted experiments using controllers with the same base mesh.
We selected an irregular-shaped sponge block (Fig.~\ref{fig:experiment_setup}(g)) and a silicon colon model (Fig.~\ref{fig:experiment_setup}(f)).
We set the black mesh in Fig.~\ref{fig:base_mesh_size} as the base mesh, and obtained surface samplings from the objects' contours tracked by color segmentation.
We set the same kind of desired deformation (Fig.~\ref{fig:object_similar}): lifting the object to the left while pushing it forward.
The deformation feature dimension was set to be: $m=15$.

Results in Fig.~\ref{fig:object_similar} show the minimization of both $\left \| \mathbf{e}_{s}(t) \right \|$ and $\mathbf{e}_d(t)$ for different cases.
Manipulation error curves present few steady-state errors (less than 2mm) because modal-based dimension reductions neglect some details of the object's shape deformation.
In addition, from the resulting images of the left camera views, we can see that the objects are manipulated to fit the desired shapes in a general sense.

\begin{figure}[t]
    \centering
    \includegraphics[width=\linewidth]{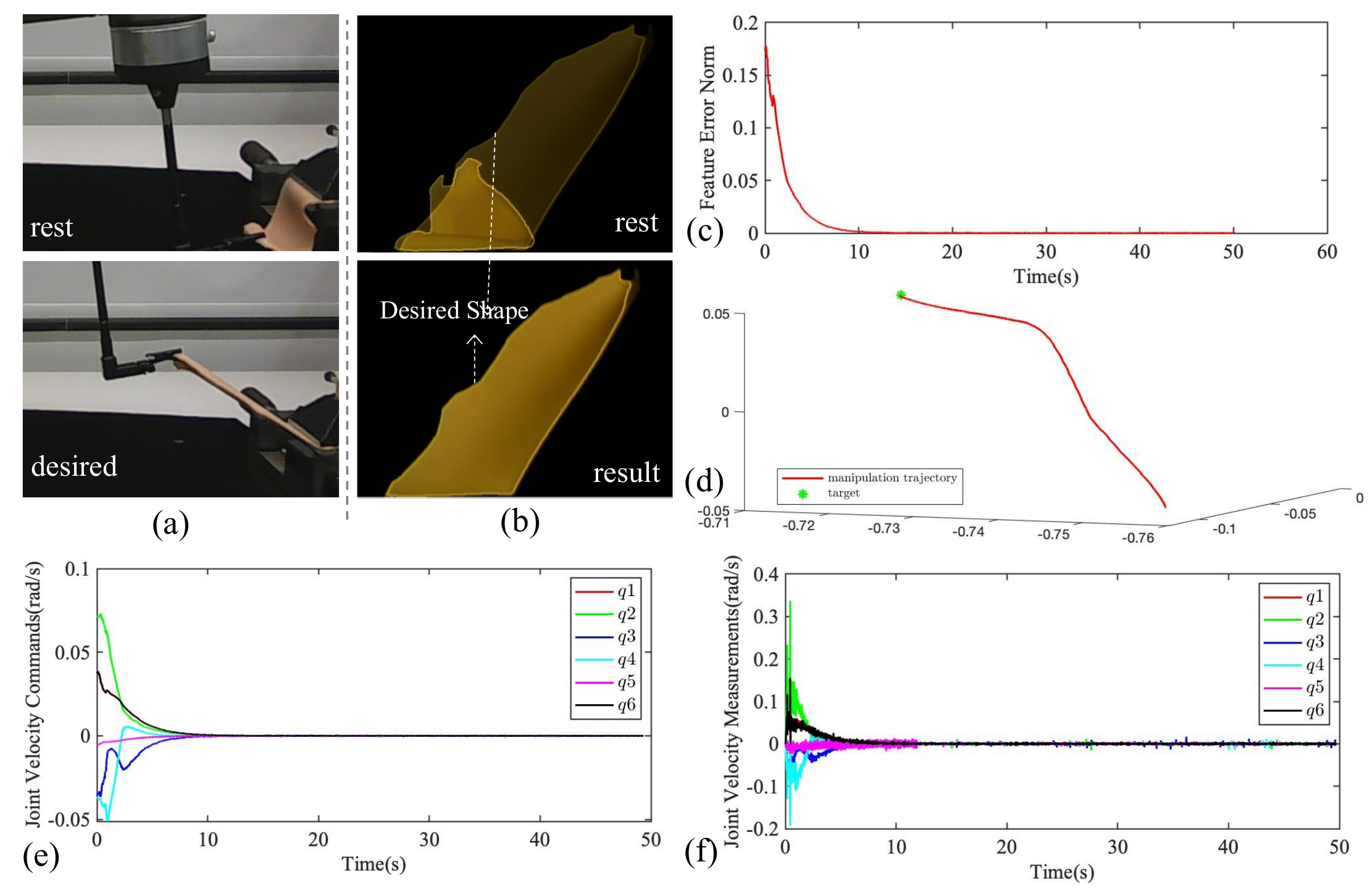}
    \vspace{-0.7cm}
    \caption{The case with large deformation.
    (a/b): the third/left camera view of the rest and desired configurations;
    (c): the result curve of $\left \| \mathbf{e}_{s}(t) \right \|$, where the small bulges are due to contour detection noises;
    (d): the manipulation trajectory (of $\mathbf{x}(\mathbf{p}_r,t)$), where the green star symbol denotes the target manipulation position ($\mathbf{x}^*(\mathbf{p}_r)$);
    (e): the curves of joint velocity commands;
    (f): the curves of computational joint velocity measurements.
    }
    \label{fig:large_deformation}
    \vspace{-0.2cm}
\end{figure}

\subsubsection{Validation with Large Deformation}
We also validated our method with large deformation.
It should be mentioned that the large desired deformation causes large initial deformation feature errors and fast initial manipulation commands, 
thus we need to slow down the parameter updating speed to avoid oscillations at the beginning of the control phase.
In the large deformation case, we adjusted the parameter updating gain to be: $\Gamma = 100$.
We selected the long silicon block (Fig.~\ref{fig:experiment_setup}(h)) as the object and the black mesh in Fig.~\ref{fig:base_mesh_size} as the base mesh.
We obtained surface samplings from the objects' contours tracked by color segmentation.
The deformation feature dimension was set to be: $m=50$.
To further analyze the control performance with large deformation, besides the deformation feature errors (Fig.~\ref{fig:large_deformation}(c)) and the manipulation trajectory (Fig.~\ref{fig:large_deformation}(d)),  we also plotted the velocity commands (Fig.~\ref{fig:large_deformation}(e)) and the computational speed measurements (Fig.~\ref{fig:large_deformation}(f)) in the joint space.
The results show that: the deformation feature errors are minimized; the target manipulation position is reached; the manipulation is smooth in the joint space.

\begin{figure}[t]
    \centering
    \vspace{0.1cm}
    \includegraphics[width=\linewidth]{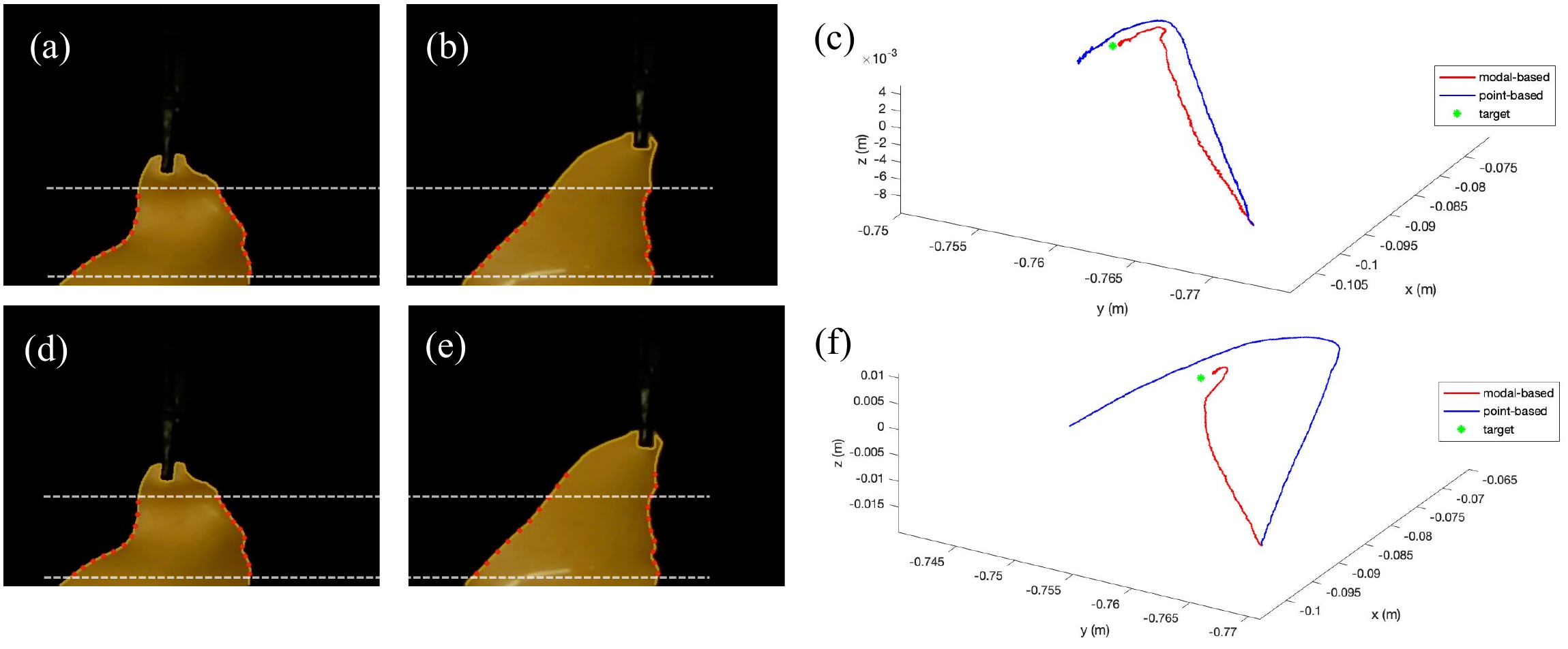}
    \vspace{-0.7cm}
    \caption{The sampling illustrations and the manipulation trajectories of the comparative study:
    the first/second column shows the left camera views of the different sampling methods at the rest/desired configurations;
    the third column shows manipulation trajectories ($\mathbf{x}(\mathbf{p}_r,t)$), where the green star symbol denotes the target manipulation position ($\mathbf{x}^*(\mathbf{p}_r)$);
    (a-c): for the "good" sampling case;
    (d-f): for the "bad" sampling case.
    }
    \label{fig:compare_tr}
    \vspace{-0.2cm}
\end{figure}

\begin{figure}[t]
    \centering
    \includegraphics[width=\linewidth]{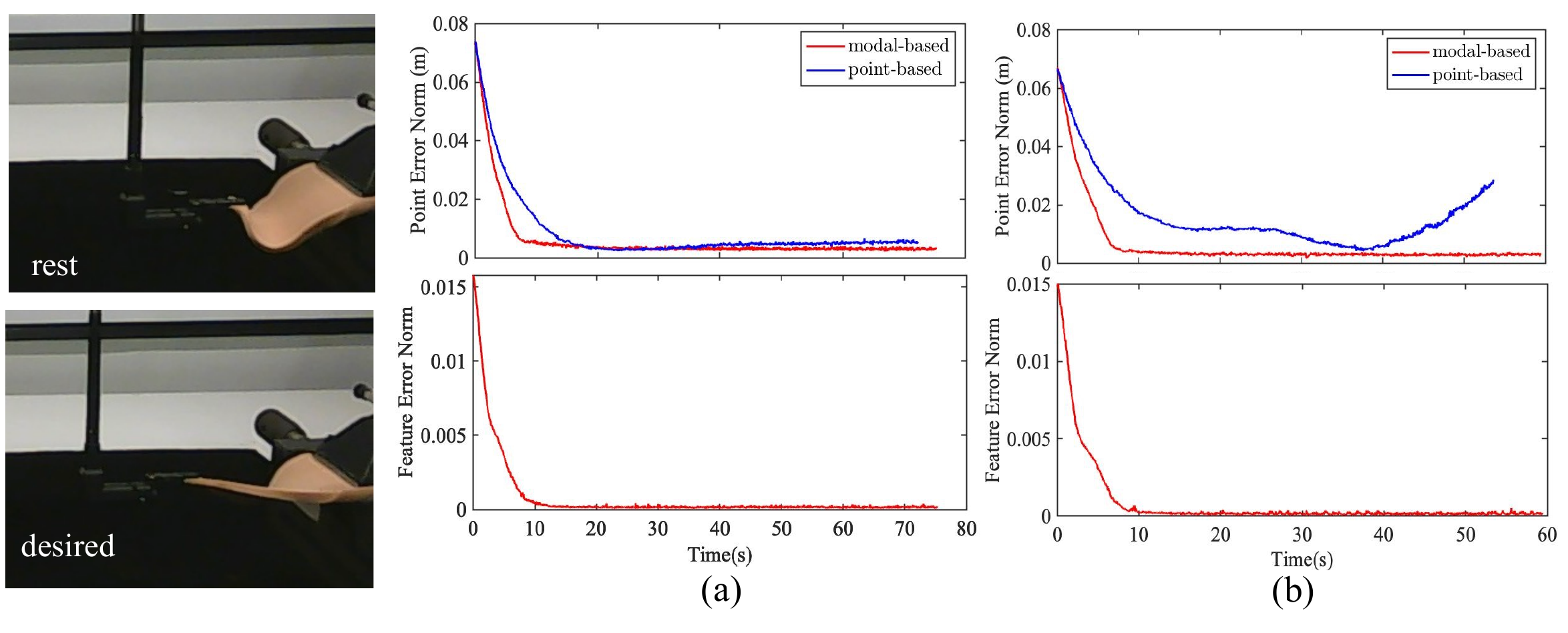}
    \vspace{-0.7cm}
    \caption{The setup and error curves of the comparative study:
    the first column shows the third camera views of the rest and desired configurations;
    the second/third column shows the curves of the point error norm ($\left \| \mathbf{x}(\mathbf{p}_s,t) - \mathbf{x}^*(\mathbf{p}_s) \right \|$) and the feature error norm ($\left \| \mathbf{e}_s(t) \right \|$) of the "good"/"bad" sampling case.
    }
    \label{fig:compare_error}
    \vspace{-0.1cm}
\end{figure}

\subsection{Comparative Study}
As our method belongs to the deformation-Jacobian-based approaches, 
we compared our method (the 3D deformation controller using modal-based global deformation features and modal-based deformation Jacobian matrix) with the baseline method (the 3D deformation controller using point-based local geometric features and point-based deformation Jacobian matrix) proposed by Navarro-Alarcon et al.~\cite{navarro2016automatic}.

We selected the same object (the short silicon block in Fig.~\ref{fig:experiment_setup}(d)) and the same desired deformation from the rest configuration to the desired configuration in Fig.~\ref{fig:compare_error}.
We sampled $20$ points from the object's contour tracked by color segmentation.
To unify the two methods to the same setup, we set the sampled points $\mathbf{p}_s$ as the control points of~\cite{navarro2016automatic} and formulated a point-based controller where the object's deformation is directly described by $\mathbf{x}(\mathbf{p}_s,t)$.
According to the baseline method, we considered the following point-based deformation Jacobian matrix between a small position change of the sampled points $\mathbf{p}_s$ and a small position change of the manipulation point $\mathbf{p}_r$:
$$
  \delta \mathbf{x}(\mathbf{p}_s) = \mathbf{J}_p \delta \mathbf{x}(\mathbf{p}_r)
$$
Based on the online estimation of the (whole) deformation Jacobian matrix $\hat{\mathbf{J}}_p(t)$, we computed the manipulation commands of the point-based controller with:
$$
\mathbf{v}(\mathbf{p}_r,t) = -\mathbf{K}_p \hat{\mathbf{J}}_p^{+}(t) (\mathbf{x}(\mathbf{p}_s,t) - \mathbf{x}^*(\mathbf{p}_s))
$$
where $\hat{\mathbf{J}}_p^{+}(t)$ is the pseudo-inverse matrix of $\hat{\mathbf{J}}_p(t)$.
We initialized $\hat{\mathbf{J}}_p$ by conducting test motions around the starting point.
As to our modal-based controller, we selected the black mesh in Fig.~\ref{fig:base_mesh_size} as the base mesh, and used $\mathbf{x}(\mathbf{p}_s,t)$ to compute the modal-based deformation features $\mathbf{s}(t)$.
The dimension of our modal-based deformation features was set to be: $m=20$.
During the robot manipulation, we recorded the manipulation trajectories of $\mathbf{x}(\mathbf{p}_r,t)$ and 
the point error norm ((i.e. $\left \| \mathbf{x}(\mathbf{p}_s,t) - \mathbf{x}^*(\mathbf{p}_s) \right \|$) for both controllers.
We also recorded $\left \| \mathbf{e}_s(t) \right \|$ for our modal-based controller.
We adjusted $\mathbf{K}_p$ and $\mathbf{K}_s$ such that the initial velocity amplitudes of the two controllers were similar (which can be seen from the similar initial slopes of the point error norm curves in Fig.~\ref{fig:compare_error}(a,b)).

We set two cases with different point sampling methods:
the "good" sampling case (Fig.~\ref{fig:compare_tr}(a,b)), where we sampled points by dividing the contour's arc length equally during the whole manipulation process;
the "bad" sampling case (Fig.~\ref{fig:compare_tr}(d,e)), where we only sampled points by dividing the arc length equally at $t_0$, and then sampled points with fixed 2D intervals (of the y-coordinate in the image plane of the left camera).
In this way, the "good" sampling case adjusted the point sampling along with the object deformation, and thus was able to sample the 3D object points from a small neighborhood of the initial sampled points.
On the other hand, the "bad" sampling case sampled the 3D object points that kept changing and getting farther away from the manipulation point.

Fig.~\ref{fig:compare_tr}(c) and Fig.~\ref{fig:compare_error}(a) show the results of the "good" sampling case.
We can conclude that compared to the point-based controller, our modal-based controller achieves:
1). smaller steady-state errors of both $\mathbf{e}_d(t)$ and $\left \| \mathbf{x}(\mathbf{p}_s,t) - \mathbf{x}^*(\mathbf{p}_s) \right \|$; 
2). faster minimization of the point errors;
3). better manipulation trajectories.
Fig.~\ref{fig:compare_tr}(f) and Fig.~\ref{fig:compare_error}(b) show the results of the "bad" sampling case.
Our modal-based controller successfully reaches the target manipulation position (with only a small steady-state error) while the point-based controller fails.
The reasons are: our modal-based features describe the global object deformation rather than the local 3D positions of the sampled points;
our modal-based features enable us to formulate the deformation Jacobian matrix with physically based reference directions.

\subsection{Cases with Unsatisfactory Situations}

\subsubsection{Cases with Unreachable Desired Deformation}
To validate our method with unreachable desired deformation,
we selected one manipulation point to generate the desired deformation (Fig.~\ref{fig:unreachable}(a)) while using a different manipulation point for control (Fig.~\ref{fig:unreachable}(b)).
We set the object to be the short silicon block (Fig.~\ref{fig:experiment_setup}(d)) and the base mesh to be the black mesh in Fig.~\ref{fig:base_mesh_size}.
The deformation feature dimension was set to be: $m=15$.
During the control phase, we recorded the point error norms ((i.e. $\left \| \mathbf{x}(\mathbf{p}_s,t) - \mathbf{x}^*(\mathbf{p}_s) \right \|$) and $\left \| \mathbf{e}_s(t) \right \|$. 
The results in Fig.~\ref{fig:unreachable}(c) show that the controller reaches a local minimum with a small steady-state error which means that the resulting deformation is close to the target deformation.

\begin{figure}[t]
    \centering
    \includegraphics[width=\linewidth]{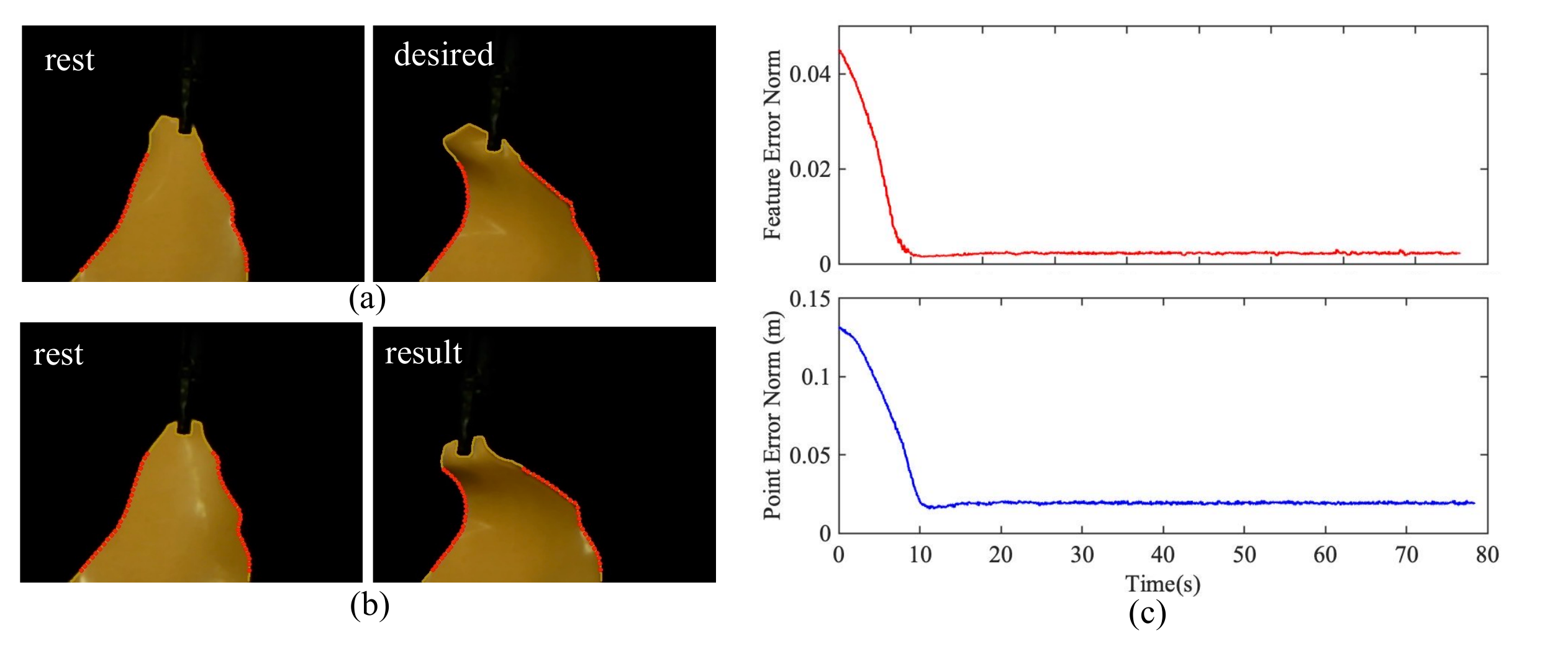}
    \vspace{-0.7cm}
    \caption{Results of the case with unreachable desired deformation:
    (a): the left camera views of the rest/desired configuration when generating the desired deformation;
    (b): the left camera views of the rest/result configuration for the control task;
    (c): the result curves of the feature error norm ($\left \| \mathbf{e}_s(t) \right \|$) and the point error norm ($\left \| \mathbf{x}(\mathbf{p}_s,t) - \mathbf{x}^*(\mathbf{p}_s) \right \|$).
    }
    \label{fig:unreachable}
    \vspace{-0.2cm}
\end{figure}
\begin{figure}[t]
    \centering
    \includegraphics[width=\linewidth]{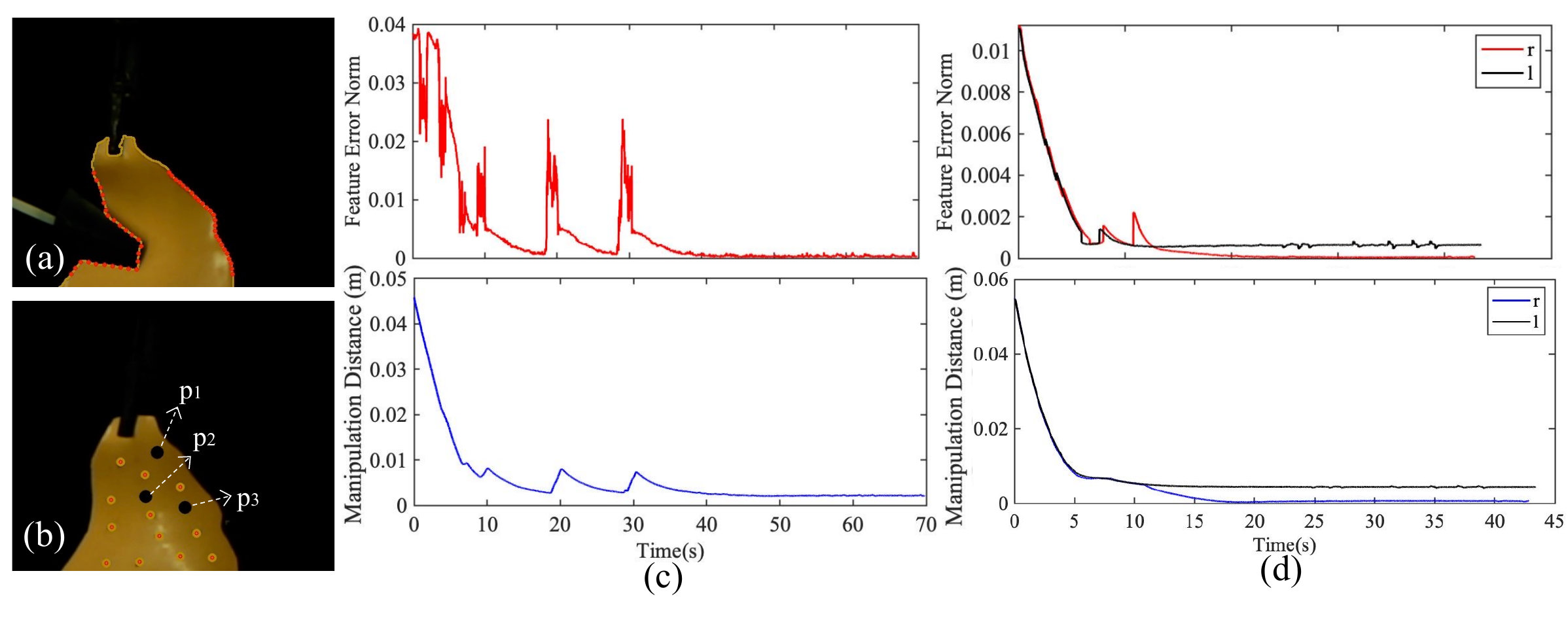}
    \vspace{-0.7cm}
    \caption{Cases with occlusions and sampling loss:
    (a/b): snapshots of the occlusion/sampling loss case;
    (c/d): the feature error norm curve and the manipulation distance (i.e. $\left \| \mathbf{e}_{d}(t) \right \|$) curve of the occlusion/sampling loss case,
    where "l" is for the permanent loss case and "r" is for the temporary loss case.
    Bumps in the feature error curves are caused by surface sampling changes.
    However, the corresponding bumps in the manipulation distance curves are much smaller since the modal truncation discards the high-frequency deformation modes.
    }
    \label{fig:occlussion}
    \vspace{-0.2cm}
\end{figure}

\subsubsection{Cases with Occlusions}
To validate our method with occlusions, we set two kinds of cases:
the occlusion case (Fig.~\ref{fig:occlussion}(a)), where we manually placed black occlusion above the object during deformation and after convergence;
the sampling loss cases (Fig.~\ref{fig:occlussion}(b)), where we simulated the situations of temporary and permanent loss of some sampled points.
We set the same desired deformation from the rest configuration in Fig.~\ref{fig:base_mesh_size}(c) to the desired configuration in Fig.~\ref{fig:base_mesh_size}(d).
We set the object to be the short silicon block (Fig.~\ref{fig:experiment_setup}(d)) and the base mesh to be the black mesh in Fig.~\ref{fig:base_mesh_size}.
The deformation feature dimension was set to be: $m=15$.

For the occlusion case, 
we obtained surface samplings from the object's contour tracked by color segmentation.
When occlusions occurred (Fig.~\ref{fig:occlussion}(a)), we sampled points on the (partially) obscured contour, which changed the values of our deformation features (causing the large bumps of the feature error curves in Fig.~\ref{fig:occlussion}(c)).
Results in Fig.~\ref{fig:occlussion}(c) show the minimization of both $\left \| \mathbf{e}_{s}(t) \right \|$ and $\left \| \mathbf{e}_{d}(t) \right \|$ even with dynamic occlusions.
For the sampling loss cases, 
we obtained surface samplings by tracking points on the object.
We set two sub-cases:
the temporary loss case, where we simulated the process of losing point $\mathbf{p}_1 \to \mathbf{p}_2 \to \mathbf{p}_3$, and then recovering point $\mathbf{p}_3 \to \mathbf{p}_2 \to \mathbf{p}_1$;
the permanent loss case, where we simulated the process of losing point $\mathbf{p}_1 \to \mathbf{p}_2 \to \mathbf{p}_3$, and then only recovering point $\mathbf{p}_3 \to \mathbf{p}_2$.
As point loss/recovery caused changes of the samplings' one-to-one correspondences, we reformulated the deformation feature computation matrices in equation~\eqref{eq:feature_computation} by re-assembling the matrices $\mathbf{N}_s(\boldsymbol{\eta}(\mathbf{p}_s),\mathbf{n}_s)$ and $[\mathbf{\Phi}_n]_s$ according to the changed configurations of $\mathbf{p}_s$ and $\mathbf{n}_s$.
Results in Fig.~\ref{fig:occlussion}(d) show the minimization of both $\left \| \mathbf{e}_{s}(t) \right \|$ and $\left \| \mathbf{e}_{d}(t) \right \|$.
Permanent loss of the sampled point only causes a small steady-state error which means that the object was manipulated to a position close to the target position.

\section{Discussion and Conclusion}
\subsection{Limitations}
There are some situations where the proposed controller cannot deform the object to (or close to) the desired shape.
To start with, our control performance is limited by the accuracy of object perception and reconstruction, which is challenging for texture-less volumetric objects with complex shapes. For example, as shown in Fig.~\ref{fig:failure_cases}(a,b), if the cross-sectional contours detected by the two cameras correspond to different parts of the object (where the inherent self-occlusion problem in stereo matching leads to inaccurate reconstruction), the object cannot be deformed to the desired shape (implied by the steady-state errors of the manipulation distance in Fig.~\ref{fig:failure_cases}(c)).
In addition, we cannot tackle the desired shape with torsional deformation.
This is because, at this stage, we design the deformation basis using linear modal analysis.
Due to its linear strain assumption, we can only control the linear velocity of manipulation points.
For example, given the desired shape (Fig.~\ref{fig:failure_cases}(d)) whose generation includes twisting the object, our controller can minimize the errors of our deformation features (Fig.~\ref{fig:failure_cases}(f)) but the resulting shape is quite different from the desired one (Fig.~\ref{fig:failure_cases}(e)).
To solve this kind of problem, further efforts should be made to investigate rotational deformation under the framework of modal analysis.

\begin{figure}[t]
    \centering
    \includegraphics[width=\linewidth]{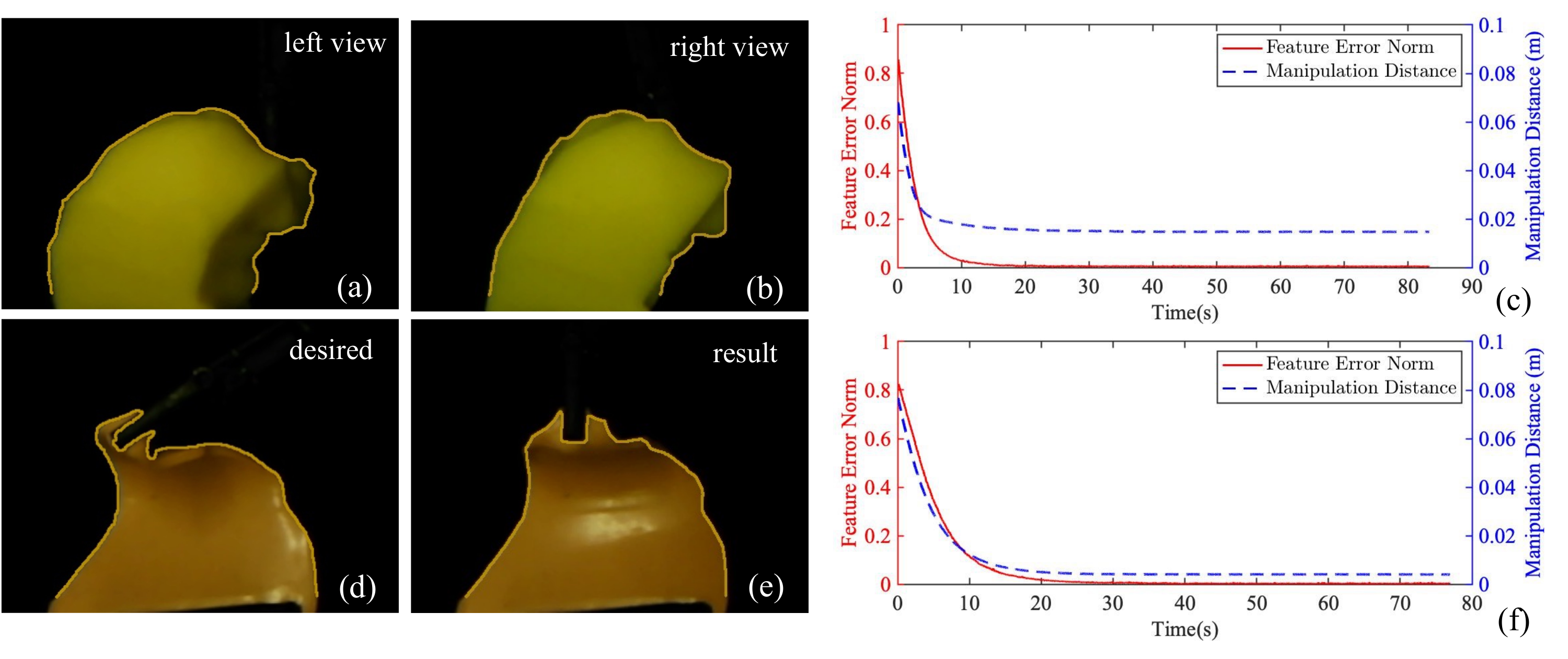}
    \vspace{-0.7cm}
    \caption{Example situations where our controller cannot deform the object to the desired shape.
    The case with inaccurate reconstruction: (a,b) are the camera views during manipulation, and (c) shows the result curves of $\left \| \mathbf{e}_{s}(t) \right \|$ and $\left \| \mathbf{e}_{d}(t) \right \|$.
    The case with torsional deformation: (d,e) are the left camera views of the desired and resulting shapes, and (f) shows the result curves of $\left \| \mathbf{e}_{s}(t) \right \|$ and $\left \| \mathbf{e}_{d}(t) \right \|$.
    }
    \label{fig:failure_cases}
    \vspace{-0.3cm}
\end{figure}

We also discuss some methodological limitations of the proposed controller.
First, 
our method is developed under the quasi-static assumption and we conducted validation cases with slow manipulation.
Object dynamics are not considered in our controller design.
To show the influence of the object dynamics on the closed-loop behavior, we simulated a case 
with accelerated manipulation and lowered Young's modulus of the object,
causing the dynamic deformation to be non-negligible.
The results in Fig.~\ref{fig:limitation_dynamics} show that the deformation feature and shape errors only converge to some bounded areas and 
oscillate with the object's vibrations.
To minimize the errors and compress the vibrations, we need to further study object dynamics in our controller design, which will be a future direction of our research.
Second,
we need to track image features on the object (such as points, curves, and contours) to obtain surface samplings.
Point detection or object segmentation is required.
We also need to find a start point of the segmented contour or curve to track the spatial order of the sampled points.
In the experiments, we selected the lowest left corner of the object's fixed end as the start point.
Another limitation is the local nature of the deformation Jacobian matrix and the linear modal analysis. 
Further efforts can be made to improve the method with global control or planning strategies.

\begin{figure}[t]
    \centering
    \includegraphics[width=1\linewidth]{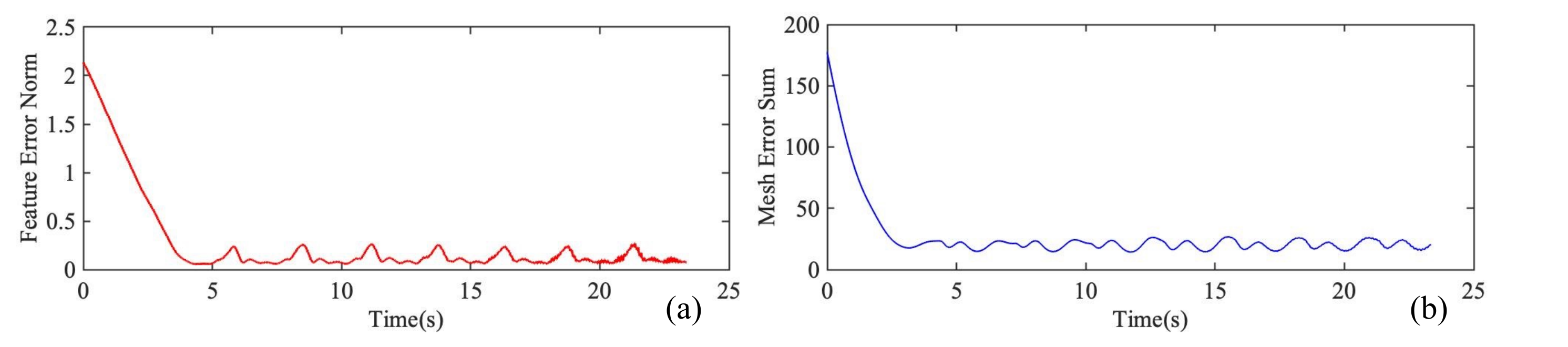}
    \vspace{-0.7cm}
    \caption{Simulation analysis of the vibrating object. (a): the result curve of $\left \| \mathbf{e}_{s}(t) \right \|$. (b) the result curve of $ \mathbf{e}_{x}(t)$.}
    \label{fig:limitation_dynamics}
\end{figure}

\subsection{Conclusion}
This paper proposed an adaptive 3D shape deformation controller using novel global deformation features based on modal analysis.
Instead of performing modal analysis on the object, we proposed a model-free framework using truncated modes of the base meth to span a low-dimensional deformation feature space with unique object representations.
All the modeling uncertainties introduced by the unknown geometry and physical properties of the object were treated as low-dimensional modal parameters in the deformation Jacobian matrix.
Based on the Jacobian matrix, we formulated an adaptive controller where the modal parameters can be linearized from the closed-loop error dynamics.
In this way, we can design the online parameter estimation laws with guaranteed control stability proved using the Lyapunov-based method.
We conducted simulations to validate our controller with different modeling, sampling, and manipulating conditions.
Extensive experiments were conducted with different linear, planar, and volumetric objects under different stereo measurements of points, curves, and contours.
The comparative study with the baseline method showed the advantages of our controller.
We also validated the effectiveness of our controller with large pose and size errors of the base mesh, with large deformation, with unreachable desired deformation, and with occlusions.

To further generalize the modal-based deformation control approach, our future works will develop methods without requiring tracking image features on the object.
Global and dynamic deformation control will also be investigated.


\bibliographystyle{./bib/IEEEtran}
\bibliography{./bib/references}

\begin{thebibliography}{10}
\providecommand{\url}[1]{#1}
\csname url@samestyle\endcsname
\providecommand{\newblock}{\relax}
\providecommand{\bibinfo}[2]{#2}
\providecommand{\BIBentrySTDinterwordspacing}{\spaceskip=0pt\relax}
\providecommand{\BIBentryALTinterwordstretchfactor}{4}
\providecommand{\BIBentryALTinterwordspacing}{\spaceskip=\fontdimen2\font plus
\BIBentryALTinterwordstretchfactor\fontdimen3\font minus
  \fontdimen4\font\relax}
\providecommand{\BIBforeignlanguage}[2]{{%
\expandafter\ifx\csname l@#1\endcsname\relax
\typeout{** WARNING: IEEEtran.bst: No hyphenation pattern has been}%
\typeout{** loaded for the language `#1'. Using the pattern for}%
\typeout{** the default language instead.}%
\else
\language=\csname l@#1\endcsname
\fi
#2}}
\providecommand{\BIBdecl}{\relax}
\BIBdecl

\bibitem{shin2019autonomous}
C.~Shin, P.~W. Ferguson, S.~A. Pedram, J.~Ma, E.~P. Dutson, and J.~Rosen,
  ``Autonomous tissue manipulation via surgical robot using learning based
  model predictive control,'' in \emph{2019 International Conference on
  Robotics and Automation (ICRA)}.\hskip 1em plus 0.5em minus 0.4em\relax IEEE,
  2019, pp. 3875--3881.

\bibitem{lee2015learning}
A.~X. Lee, A.~Gupta, H.~Lu, S.~Levine, and P.~Abbeel, ``Learning from multiple
  demonstrations using trajectory-aware non-rigid registration with
  applications to deformable object manipulation,'' in \emph{2015 IEEE/RSJ
  International Conference on Intelligent Robots and Systems (IROS)}.\hskip 1em
  plus 0.5em minus 0.4em\relax IEEE, 2015, pp. 5265--5272.

\bibitem{li2018vision}
X.~Li, X.~Su, and Y.-H. Liu, ``Vision-based robotic manipulation of flexible
  pcbs,'' \emph{IEEE/ASME Transactions on Mechatronics}, vol.~23, no.~6, pp.
  2739--2749, 2018.

\bibitem{wada2001robust}
T.~Wada, S.~Hirai, S.~Kawamura, and N.~Kamiji, ``Robust manipulation of
  deformable objects by a simple pid feedback,'' in \emph{Proceedings 2001
  ICRA. IEEE International Conference on Robotics and Automation (Cat. No.
  01CH37164)}, vol.~1.\hskip 1em plus 0.5em minus 0.4em\relax IEEE, 2001, pp.
  85--90.

\bibitem{das2011autonomous}
J.~Das and N.~Sarkar, ``Autonomous shape control of a deformable object by
  multiple manipulators,'' \emph{Journal of Intelligent \& Robotic Systems},
  vol.~62, no.~1, pp. 3--27, 2011.

\bibitem{tokumoto2002deformation}
S.~Tokumoto and S.~Hirai, ``Deformation control of rheological food dough using
  a forming process model,'' in \emph{Proceedings 2002 IEEE International
  Conference on Robotics and Automation (Cat. No. 02CH37292)}, vol.~2.\hskip
  1em plus 0.5em minus 0.4em\relax IEEE, 2002, pp. 1457--1464.

\bibitem{navarro2016automatic}
D.~Navarro-Alarcon, H.~M. Yip, Z.~Wang, Y.-H. Liu, F.~Zhong, T.~Zhang, and
  P.~Li, ``Automatic 3-d manipulation of soft objects by robotic arms with an
  adaptive deformation model,'' \emph{IEEE Transactions on Robotics}, vol.~32,
  no.~2, pp. 429--441, 2016.

\bibitem{hu20193}
Z.~Hu, T.~Han, P.~Sun, J.~Pan, and D.~Manocha, ``3-d deformable object
  manipulation using deep neural networks,'' \emph{IEEE Robotics and Automation
  Letters}, vol.~4, no.~4, pp. 4255--4261, 2019.

\bibitem{navarro2018fourier}
D.~Navarro-Alarcon and Y.-H. Liu, ``Fourier-based shape servoing: a new
  feedback method to actively deform soft objects into desired 2-d image
  contours,'' \emph{IEEE Transactions on Robotics}, vol.~34, no.~1, pp.
  272--279, 2018.

\bibitem{qi2021contour}
J.~Qi, G.~Ma, J.~Zhu, P.~Zhou, Y.~Lyu, H.~Zhang, and D.~Navarro-Alarcon,
  ``Contour moments based manipulation of composite rigid-deformable objects
  with finite time model estimation and shape/position control,''
  \emph{IEEE/ASME Transactions on Mechatronics}, 2021.

\bibitem{ficuciello2018fem}
F.~Ficuciello, A.~Migliozzi, E.~Coevoet, A.~Petit, and C.~Duriez, ``Fem-based
  deformation control for dexterous manipulation of 3d soft objects,'' in
  \emph{2018 IEEE/RSJ International Conference on Intelligent Robots and
  Systems (IROS)}.\hskip 1em plus 0.5em minus 0.4em\relax IEEE, 2018, pp.
  4007--4013.

\bibitem{abayazid2015ultrasound}
M.~Abayazid, P.~Moreira, N.~Shahriari, S.~Patil, R.~Alterovitz, and S.~Misra,
  ``Ultrasound-guided three-dimensional needle steering in biological tissue
  with curved surfaces,'' \emph{Medical engineering \& physics}, vol.~37,
  no.~1, pp. 145--150, 2015.

\bibitem{sanchez2018robotic}
J.~Sanchez, J.-A. Corrales, B.-C. Bouzgarrou, and Y.~Mezouar, ``Robotic
  manipulation and sensing of deformable objects in domestic and industrial
  applications: a survey,'' \emph{The International Journal of Robotics
  Research}, vol.~37, no.~7, pp. 688--716, 2018.

\bibitem{arriola2020modeling}
V.~E. Arriola-Rios, P.~Guler, F.~Ficuciello, D.~Kragic, B.~Siciliano, and J.~L.
  Wyatt, ``Modeling of deformable objects for robotic manipulation: A tutorial
  and review,'' \emph{Frontiers in Robotics and AI}, vol.~7, p.~82, 2020.

\bibitem{nadon2018multi}
F.~Nadon, A.~J. Valencia, and P.~Payeur, ``Multi-modal sensing and robotic
  manipulation of non-rigid objects: A survey,'' \emph{Robotics}, vol.~7,
  no.~4, p.~74, 2018.

\bibitem{berenson2013manipulation}
D.~Berenson, ``Manipulation of deformable objects without modeling and
  simulating deformation,'' in \emph{2013 IEEE/RSJ International Conference on
  Intelligent Robots and Systems}.\hskip 1em plus 0.5em minus 0.4em\relax IEEE,
  2013, pp. 4525--4532.

\bibitem{navarro2013model}
D.~Navarro-Alarcon, Y.-H. Liu, J.~G. Romero, and P.~Li, ``Model-free visually
  servoed deformation control of elastic objects by robot manipulators,''
  \emph{IEEE Transactions on Robotics}, vol.~29, no.~6, pp. 1457--1468, 2013.

\bibitem{shetab2022rigid}
M.~Shetab-Bushehri, M.~Aranda, Y.~Mezouar, and E.~{\"O}zg{\"u}r,
  ``As-rigid-as-possible shape servoing,'' \emph{IEEE Robotics and Automation
  Letters}, vol.~7, no.~2, pp. 3898--3905, 2022.

\bibitem{qi2020adaptive}
J.~Qi, W.~Ma, D.~Navarro-Alarcon, H.~Gao, and G.~Ma, ``Adaptive shape servoing
  of elastic rods using parameterized regression features and auto-tuning
  motion controls,'' \emph{arXiv preprint arXiv:2008.06896}, 2020.

\bibitem{lagneau2020active}
R.~Lagneau, A.~Krupa, and M.~Marchal, ``Active deformation through visual
  servoing of soft objects,'' in \emph{2020 IEEE International Conference on
  Robotics and Automation (ICRA)}.\hskip 1em plus 0.5em minus 0.4em\relax IEEE,
  2020, pp. 8978--8984.

\bibitem{cherubini2020model}
A.~Cherubini, V.~Ortenzi, A.~Cosgun, R.~Lee, and P.~Corke, ``Model-free
  vision-based shaping of deformable plastic materials,'' \emph{The
  International Journal of Robotics Research}, p. 0278364920907684, 2020.

\bibitem{hu2018three}
Z.~Hu, P.~Sun, and J.~Pan, ``Three-dimensional deformable object manipulation
  using fast online gaussian process regression,'' \emph{IEEE Robotics and
  Automation Letters}, vol.~3, no.~2, pp. 979--986, 2018.

\bibitem{cretu2011soft}
A.-M. Cretu, P.~Payeur, and E.~M. Petriu, ``Soft object deformation monitoring
  and learning for model-based robotic hand manipulation,'' \emph{IEEE
  Transactions on Systems, Man, and Cybernetics, Part B (Cybernetics)},
  vol.~42, no.~3, pp. 740--753, 2011.

\bibitem{yu2022global}
M.~Yu, K.~Lv, H.~Zhong, S.~Song, and X.~Li, ``Global model learning for large
  deformation control of elastic deformable linear objects: An efficient and
  adaptive approach,'' \emph{IEEE Transactions on Robotics}, 2022.

\bibitem{montagnat2001review}
J.~Montagnat, H.~Delingette, and N.~Ayache, ``A review of deformable surfaces:
  topology, geometry and deformation,'' \emph{Image and vision computing},
  vol.~19, no.~14, pp. 1023--1040, 2001.

\bibitem{madi2019new}
K.~Madi, E.~Paquet, and H.~Kheddouci, ``New graph distance for deformable 3d
  objects recognition based on triangle-stars decomposition,'' \emph{Pattern
  Recognition}, vol.~90, pp. 297--307, 2019.

\bibitem{fayad2010piecewise}
J.~Fayad, L.~Agapito, and A.~Del~Bue, ``Piecewise quadratic reconstruction of
  non-rigid surfaces from monocular sequences,'' in \emph{European conference
  on computer vision}.\hskip 1em plus 0.5em minus 0.4em\relax Springer, 2010,
  pp. 297--310.

\bibitem{newcombe2010live}
R.~A. Newcombe and A.~J. Davison, ``Live dense reconstruction with a single
  moving camera,'' in \emph{2010 IEEE Computer Society Conference on Computer
  Vision and Pattern Recognition}.\hskip 1em plus 0.5em minus 0.4em\relax IEEE,
  2010, pp. 1498--1505.

\bibitem{gascuel1993implicit}
M.-P. Gascuel, ``An implicit formulation for precise contact modeling between
  flexible solids,'' in \emph{Proceedings of the 20th annual conference on
  Computer graphics and interactive techniques}, 1993, pp. 313--320.

\bibitem{prasad2010finding}
M.~Prasad, A.~Fitzgibbon, A.~Zisserman, and L.~Van~Gool, ``Finding nemo:
  Deformable object class modelling using curve matching,'' in \emph{2010 IEEE
  Computer Society Conference on Computer Vision and Pattern
  Recognition}.\hskip 1em plus 0.5em minus 0.4em\relax IEEE, 2010, pp.
  1720--1727.

\bibitem{kelemen1996segmentation}
C.~B. Kelemen and G.~Gerig, ``Segmentation of 2-d and 3-d objects from mri
  volume data using constrained elastic deformations of flexible fourier
  contour and surface models,'' \emph{Medical image analysis}, vol.~1, no.~1,
  pp. 19--34, 1996.

\bibitem{sederberg1986free}
T.~W. Sederberg and S.~R. Parry, ``Free-form deformation of solid geometric
  models,'' in \emph{Proceedings of the 13th annual conference on Computer
  graphics and interactive techniques}, 1986, pp. 151--160.

\bibitem{fugl2012simultaneous}
A.~R. Fugl, A.~Jordt, H.~G. Petersen, M.~Willatzen, and R.~Koch, ``Simultaneous
  estimation of material properties and pose for deformable objects from depth
  and color images,'' in \emph{Joint DAGM (German Association for Pattern
  Recognition) and OAGM Symposium}.\hskip 1em plus 0.5em minus 0.4em\relax
  Springer, 2012, pp. 165--174.

\bibitem{zaidi2017model}
L.~Zaidi, J.~A. Corrales, B.~C. Bouzgarrou, Y.~Mezouar, and L.~Sabourin,
  ``Model-based strategy for grasping 3d deformable objects using a
  multi-fingered robotic hand,'' \emph{Robotics and Autonomous Systems},
  vol.~95, pp. 196--206, 2017.

\bibitem{barr1987global}
A.~H. Barr, ``Global and local deformations of solid primitives,'' in
  \emph{Readings in Computer Vision}.\hskip 1em plus 0.5em minus 0.4em\relax
  Elsevier, 1987, pp. 661--670.

\bibitem{pentland1987perceptual}
A.~P. Pentland, ``Perceptual organization and the representation of natural
  form,'' in \emph{Readings in Computer Vision}.\hskip 1em plus 0.5em minus
  0.4em\relax Elsevier, 1987, pp. 680--699.

\bibitem{barbivc2005real}
J.~Barbi{\v{c}} and D.~L. James, ``Real-time subspace integration for st.
  venant-kirchhoff deformable models,'' \emph{ACM transactions on graphics
  (TOG)}, vol.~24, no.~3, pp. 982--990, 2005.

\bibitem{an2008optimizing}
S.~S. An, T.~Kim, and D.~L. James, ``Optimizing cubature for efficient
  integration of subspace deformations,'' \emph{ACM transactions on graphics
  (TOG)}, vol.~27, no.~5, pp. 1--10, 2008.

\bibitem{metaxas1992dynamic}
D.~Metaxas and D.~Terzopoulos, ``Dynamic deformation of solid primitives with
  constraints,'' in \emph{Proceedings of the 19th annual conference on Computer
  graphics and interactive techniques}, 1992, pp. 309--312.

\bibitem{krysl2001dimensional}
P.~Krysl, S.~Lall, and J.~E. Marsden, ``Dimensional model reduction in
  non-linear finite element dynamics of solids and structures,''
  \emph{International Journal for numerical methods in engineering}, vol.~51,
  no.~4, pp. 479--504, 2001.

\bibitem{pentland1989good}
A.~Pentland and J.~Williams, ``Good vibrations: Modal dynamics for graphics and
  animation,'' in \emph{Proceedings of the 16th annual conference on Computer
  graphics and interactive techniques}, 1989, pp. 215--222.

\bibitem{baraff1992dynamic}
D.~Baraff and A.~Witkin, ``Dynamic simulation of non-penetrating flexible
  bodies,'' \emph{ACM SIGGRAPH Computer Graphics}, vol.~26, no.~2, pp.
  303--308, 1992.

\bibitem{choi2005modal}
M.~G. Choi and H.-S. Ko, ``Modal warping: Real-time simulation of large
  rotational deformation and manipulation,'' \emph{IEEE Transactions on
  Visualization and Computer Graphics}, vol.~11, no.~1, pp. 91--101, 2005.

\bibitem{pentland1991closed}
A.~Pentland and S.~Sclaroff, ``Closed-form solutions for physically based shape
  modeling and recognition,'' \emph{IEEE Transactions on Pattern Analysis \&
  Machine Intelligence}, no.~7, pp. 715--729, 1991.

\bibitem{agudo2014good}
A.~Agudo, L.~Agapito, B.~Calvo, and J.~M. Montiel, ``Good vibrations: A modal
  analysis approach for sequential non-rigid structure from motion,'' in
  \emph{Proceedings of the IEEE Conference on Computer Vision and Pattern
  Recognition}, 2014, pp. 1558--1565.

\bibitem{ljung2010perspectives}
L.~Ljung, ``Perspectives on system identification,'' \emph{Annual Reviews in
  Control}, vol.~34, no.~1, pp. 1--12, 2010.

\bibitem{pentland1991recovery}
A.~Pentland and B.~Horowitz, ``Recovery of nonrigid motion and structure,''
  \emph{IEEE Transactions on Pattern Analysis \& Machine Intelligence}, no.~7,
  pp. 730--742, 1991.

\bibitem{bathe2006finite}
K.-J. Bathe, \emph{Finite element procedures}.\hskip 1em plus 0.5em minus
  0.4em\relax Klaus-Jurgen Bathe, 2006.

\bibitem{witkin1990fast}
A.~Witkin and W.~Welch, ``Fast animation and control of nonrigid structures,''
  in \emph{Proceedings of the 17th annual conference on Computer graphics and
  interactive techniques}, 1990, pp. 243--252.

\bibitem{pentland1989perception}
A.~Pentland and J.~Williams, ``Perception of non-rigid motion: Inference of
  shape, material and force,'' in \emph{Proceedings of the 11th international
  joint conference on Artificial intelligence-Volume 2}, 1989, pp. 1565--1570.

\bibitem{solina1990recovery}
F.~Solina and R.~Bajcsy, ``Recovery of parametric models from range images: The
  case for superquadrics with global deformations,'' \emph{IEEE transactions on
  pattern analysis and machine intelligence}, vol.~12, no.~2, pp. 131--147,
  1990.

\bibitem{terzopoulos1987symmetry}
D.~Terzopoulos, A.~Witkin, and M.~Kass, ``Symmetry seeking models for 3d object
  reconstruction: Active contour models,'' in \emph{Proceedings of the rst
  International Conference on Computer Vision (ICCV 87), London}, 1987.

\bibitem{kattan2010matlab}
P.~I. Kattan, \emph{MATLAB guide to finite elements: an interactive
  approach}.\hskip 1em plus 0.5em minus 0.4em\relax Springer Science \&
  Business Media, 2010.

\bibitem{junkins1993introduction}
J.~L. Junkins, \emph{Introduction to dynamics and control of flexible
  structures}.\hskip 1em plus 0.5em minus 0.4em\relax Aiaa, 1993.

\bibitem{barbivc2012fem}
J.~Barbi{\v{c}}, ``Fem simulation of 3d deformable solids: A practitioner’s
  guide to theory, discretization and model reduction. part 2: Model
  reduction,'' in \emph{SIGGRAPH 2012 Course Notes}.\hskip 1em plus 0.5em minus
  0.4em\relax Citeseer, 2012.

\bibitem{von2013efficient}
C.~Von~Tycowicz, C.~Schulz, H.-P. Seidel, and K.~Hildebrandt, ``An efficient
  construction of reduced deformable objects,'' \emph{ACM Transactions on
  Graphics (TOG)}, vol.~32, no.~6, pp. 1--10, 2013.

\bibitem{hirai2000indirect}
S.~Hirai and T.~Wada, ``Indirect simultaneous positioning of deformable objects
  with multi-pinching fingers based on an uncertain model,'' \emph{Robotica},
  vol.~18, no.~1, pp. 3--11, 2000.

\bibitem{zhong2019dual}
F.~Zhong, Y.~Wang, Z.~Wang, and Y.-H. Liu, ``Dual-arm robotic needle insertion
  with active tissue deformation for autonomous suturing,'' \emph{IEEE Robotics
  and Automation Letters}, vol.~4, no.~3, pp. 2669--2676, 2019.

\bibitem{navarro2014visual}
D.~Navarro-Alarcon, Y.-h. Liu, J.~G. Romero, and P.~Li, ``On the visual
  deformation servoing of compliant objects: Uncalibrated control methods and
  experiments,'' \emph{The International Journal of Robotics Research},
  vol.~33, no.~11, pp. 1462--1480, 2014.

\bibitem{ogden1997non}
R.~W. Ogden, \emph{Non-linear elastic deformations}.\hskip 1em plus 0.5em minus
  0.4em\relax Courier Corporation, 1997.

\bibitem{yu2022shape}
M.~Yu, H.~Zhong, and X.~Li, ``Shape control of deformable linear objects with
  offline and online learning of local linear deformation models,'' in
  \emph{2022 International Conference on Robotics and Automation (ICRA)}.\hskip
  1em plus 0.5em minus 0.4em\relax IEEE, 2022, pp. 1337--1343.

\bibitem{slotine1987adaptive}
J.-J.~E. Slotine and W.~Li, ``On the adaptive control of robot manipulators,''
  \emph{The international journal of robotics research}, vol.~6, no.~3, pp.
  49--59, 1987.

\bibitem{slotine1991applied}
J.-J.~E. Slotine, W.~Li \emph{et~al.}, \emph{Applied nonlinear control}.\hskip
  1em plus 0.5em minus 0.4em\relax Prentice hall Englewood Cliffs, NJ, 1991,
  vol. 199, no.~1.

\bibitem{chaumette1998potential}
F.~Chaumette, ``Potential problems of stability and convergence in image-based
  and position-based visual servoing,'' in \emph{The confluence of vision and
  control}.\hskip 1em plus 0.5em minus 0.4em\relax Springer, 1998, pp. 66--78.

\bibitem{allard2007sofa}
J.~Allard, S.~Cotin, F.~Faure, P.-J. Bensoussan, F.~Poyer, C.~Duriez,
  H.~Delingette, and L.~Grisoni, ``Sofa-an open source framework for medical
  simulation,'' in \emph{MMVR 15-Medicine Meets Virtual Reality}, vol.
  125.\hskip 1em plus 0.5em minus 0.4em\relax IOP Press, 2007, pp. 13--18.

\bibitem{sofascn.org}
\url{https://github.com/sofa-framework/sofa/blob/master/examples/Tutorials
  /ForceFields/TutorialForceFieldLiverFEM.scn}.

\bibitem{URScript}
\url{https://s3-eu-west-1.amazonaws.com/ur-support-site/32554/scriptManual-3.5.4.pdf}.

\end{thebibliography}

\vskip 0pt plus -1fil

\begin{IEEEbiography}[{\includegraphics
[width=1in,height=1.25in,clip,
keepaspectratio]{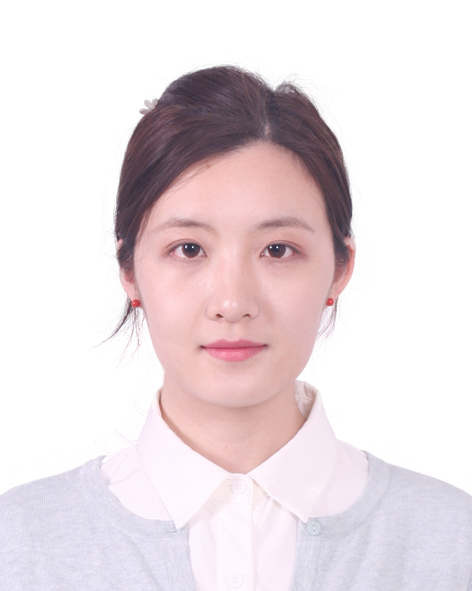}}]
{Bohan Yang}
received the B.Eng. degree in automation from Hunan University, Changsha, China, in 2014, the M.S. degree in control science and engineering from Shanghai Jiao Tong University, Shanghai, China, in 2017, and the Ph.D. degree in mechanical and automation engineering from the Chinese University of Hong Kong, HKSAR, China, in 2023. She is currently a Postdoctoral Fellow with the T Stone Robotics Institute, The Chinese University of Hong Kong. Her research interests include visual serving and medical robotics.
\end{IEEEbiography}

\vskip 0pt plus -1fil

\begin{IEEEbiography}[{\includegraphics
[width=1in,height=1.25in,clip,
keepaspectratio]{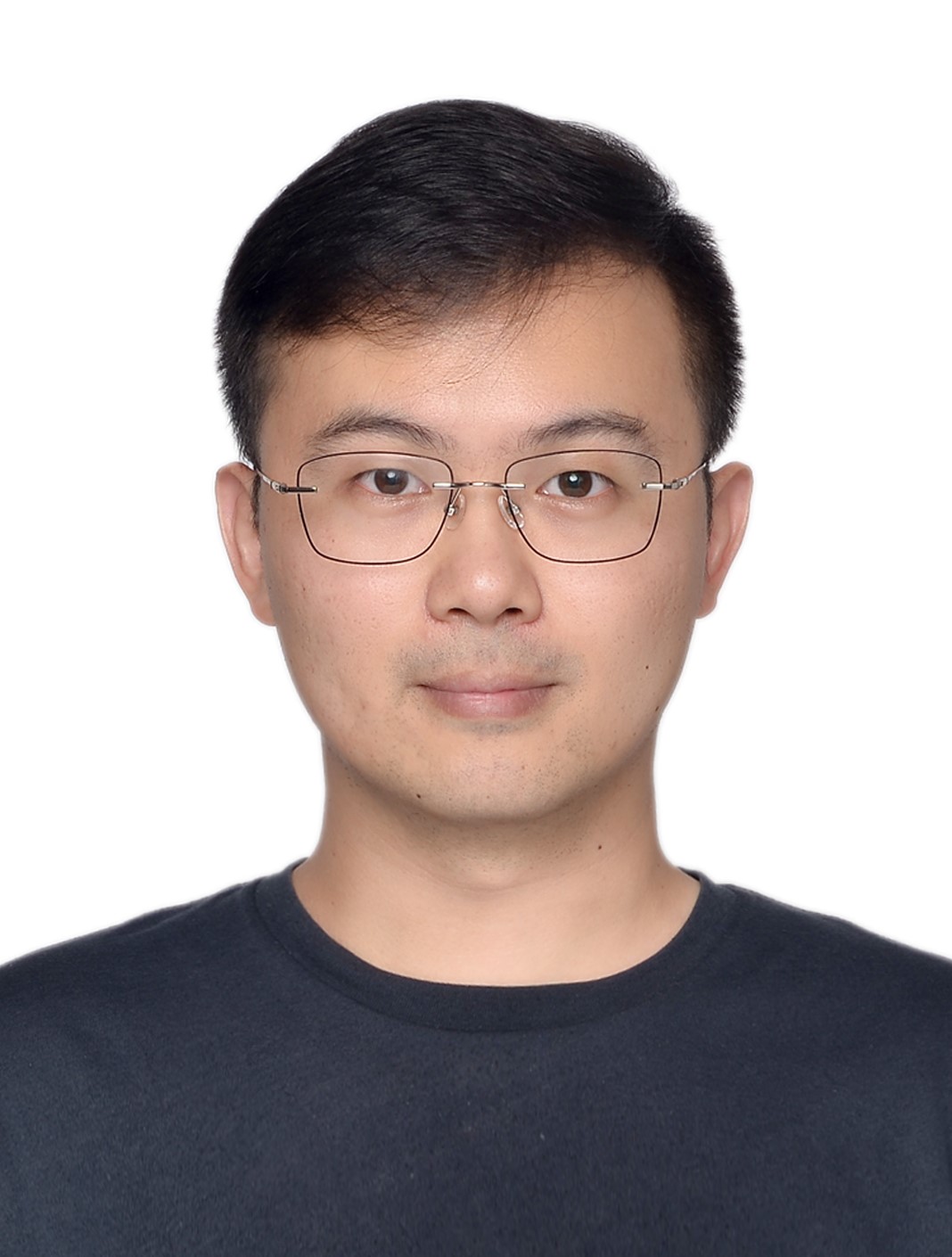}}]
{Bo Lu}
received the B.Eng. degree from the Department of Ship and Offshore Engineering, Dalian University of Technology, Liaoning Province, China, in 2013. He obtained his M.S. Degree (first honor) and Ph.D. Degree from the Department of Mechanical Engineering, The Hong Kong Polytechnic University, HKSAR, China, in 2015 and 2019, respectively. Afterward, he worked as a Postdoctoral Research Fellow at the T-stone Robotics Institute, The Chinese University of Hong Kong, HKSAR, China. He is now an Associate Professor at the Robotics and Microsystems Center, School of Mechanical and Electric Engineering, Soochow University. His current research interests include medical robotics, computer vision, and vision-based and learning-driven automation and intervention.
\end{IEEEbiography}

\vskip 0pt plus -1fil

\begin{IEEEbiography}[{\includegraphics
[width=1in,height=1.25in,clip,
keepaspectratio]{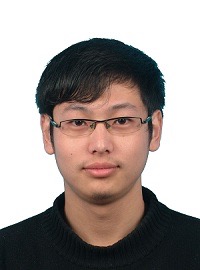}}]
{Wei Chen}
received the B.E. degree in computer science and technology from Zhengzhou University, Zhengzhou, China, in 2012, and the M.S. degree in mechanical and automation engineering from The Chinese University of Hong Kong, Hong Kong in 2021. He is currently working toward the Ph.D. degree with the Department of Mechanical and Automation Engineering, The Chinese University of Hong Kong, HKSAR, China.
\end{IEEEbiography}

\vskip 0pt plus -1fil

\begin{IEEEbiography}[{\includegraphics
[width=1in,height=1.25in,clip,
keepaspectratio]{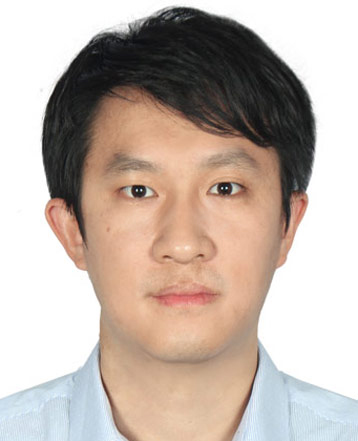}}]
{Fangxun Zhong}
received the B.Eng. degree in automation from Beijing Institute of Technology, Beijing, China, in 2014 and the Ph.D. degree in mechanical and automation engineering from The Chinese University of Hong Kong, HKSAR, China, in 2021. He is currently a postdoctoral fellow with the T Stone Robotics Institute, The Chinese University of Hong Kong. His research interests include surgery autonomy, medical robotics, dexterous robot planning and control.
\end{IEEEbiography}

\vskip 0pt plus -1fil

\begin{IEEEbiography}[{\includegraphics
[width=1in,height=1.25in,clip,
keepaspectratio]{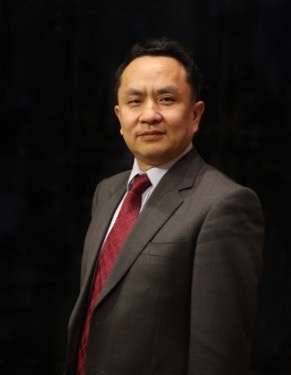}}]
{Yun-Hui Liu} received his Ph.D. degree in Applied Mathematics and Information Physics from the University of Tokyo. After working at the Electrotechnical Laboratory of Japan as a Research Scientist, he joined The Chinese University of Hong Kong (CUHK) in 1995 and is currently Choh-Ming Li Professor of Mechanical and Automation Engineering and the Director of the T Stone Robotics Institute. He also serves as the Director/CEO of Hong Kong Centre for Logistics Robotics sponsored by the InnoHK programme of the HKSAR government. He is an adjunct professor at the State Key Lab of Robotics Technology and System, Harbin Institute of Technology, China. He has published more than 500 papers in refereed journals and refereed conference proceedings and was listed in the Highly Cited Authors (Engineering) by Thomson Reuters in 2013. His research interests include visual servoing, logistics robotics, medical robotics, multi-fingered grasping, mobile robots, and machine intelligence. Dr. Liu has received numerous research awards from international journals and international conferences in robotics and automation and government agencies. He was the Editor-in-Chief of Robotics and Biomimetics and served as an Associate Editor of the IEEE TRANSACTION ON ROBOTICS AND AUTOMATION and General Chair of the 2006 IEEE/RSJ International Conference on Intelligent Robots and Systems. He is an IEEE Fellow.
\end{IEEEbiography}

\end{document}